%% file: fbvi_paper.tex
\documentclass{article}
\usepackage{iclr2026_conference,times}
\iclrfinalcopy  % preprint (non-anonymous) mode
\input{math_commands.tex}

\usepackage[hypertexnames=false]{hyperref}
\usepackage{url}
\usepackage{booktabs}
\usepackage{amsmath, amssymb, amsthm}
\usepackage{multirow}
\usepackage{graphicx}
\usepackage{xcolor}
\usepackage{algorithm,algpseudocode}
\newtheorem{proposition}{Proposition}
\newtheorem{theorem}{Theorem}
\newtheorem{lemma}{Lemma}
\newtheorem{definition}{Definition}
\theoremstyle{remark}
\newtheorem{remark}{Remark}

\title{Onsager--Machlup Posterior Transport for Deep Gaussian Processes}

\author{Jian Xu\textsuperscript{1,2} \quad
Delu Zeng\textsuperscript{3} \quad
John Paisley\textsuperscript{4} \quad
Qibin Zhao\textsuperscript{2} \\[0.3em]
\textsuperscript{1}RIKEN iTHEMS \quad
\textsuperscript{2}RIKEN AIP \quad
\textsuperscript{3}South China University of Technology \quad
\textsuperscript{4}Columbia University \\[0.3em]
\texttt{jian.xu@riken.jp}}

\newcommand{\bU}{\mathbf U}
\newcommand{\bF}{\mathbf F}
\newcommand{\bZ}{\mathbf Z}
\newcommand{\bx}{\mathbf x}
\newcommand{\by}{\mathbf y}
\newcommand{\Norm}[1]{\mathcal{N}\!\left(#1\right)}

\begin{document}
\maketitle
\lhead{Preprint}  % override the ICLR conference banner with plain "Preprint"

\begin{abstract}
Approximate inference over inducing variables is the central
computational bottleneck of Deep Gaussian Processes (DGPs). Existing
methods either fit an explicit density $q_\phi(\bU)$ by an ELBO
(DSVI, IPVI, DDVI, DBVI) or sample by MCMC (SGHMC). We instead frame
DGP inference as \emph{posterior transport}: learn a deterministic
sampler that maps a tractable reference measure to posterior-relevant
inducing variables, regularised by a path prior derived from the
Doob-bridged reference diffusion. Our realisation, \textbf{OM-Path}
(formally FBVI-bridge-Path), uses Song's probability-flow ODE applied
to DBVI's Doob-bridged forward SDE; the reference drift is closed-form
from the bridge marginal coefficients (no score matching) and the path
regulariser is the \textbf{Onsager--Machlup action}. At the
finite-$\epsilon$ value used at training, the objective is the
negative log unnormalised density of a tempered Doob-bridge path
posterior, and Theorem~\ref{thm:om_map} identifies it with the same
posterior's small-noise MAP path via the Freidlin--Wentzell LDP. Two
strict path-space ELBO variants on the same bridge backbone (FFJORD
log-det; OM-regularised CNF) are derived as ablations. Under a
matched-seed paired Wilcoxon test against DBVI on seven UCI regression
benchmarks, OM-Path delivers statistically significant wins on the two
largest datasets (\textit{power}: $p\!=\!0.014$, NLL $\mathbf{0.012}$
matching the DSVI baseline of $0.017$; \textit{protein}: $p\!=\!0.002$,
RMSE $\mathbf{0.716}$ vs.\ $0.764$, NLL $\mathbf{1.086}$ vs.\ $1.149$),
statistical ties on \textit{yacht} / \textit{qsar}, and concedes
\textit{boston} / \textit{energy} / \textit{concrete} to DBVI on
small-$N$ noisy data. The strict-ELBO variants do not clear DBVI on
any UCI metric: in this regime, reducing the variance of the path
objective dominates exact-density tracking.
\end{abstract}

\section{Introduction}\label{sec:intro}

Deep Gaussian Processes (DGPs;~\citet{damianou2013deep}) stack Gaussian-process
layers to obtain a Bayesian deep model with calibrated uncertainty. Inference
is intractable because the posterior over the per-layer inducing variables
$\bU = \{U^{(\ell)}\}_{\ell=1}^L$ is non-Gaussian and highly correlated across
layers. Doubly-stochastic mean-field VI (DSVI;~\citet{salimbeni2017doubly})
sidesteps this by enforcing a Gaussian factorisation
$q(\bU)=\prod_\ell\Norm{\bU^{(\ell)};m_\ell,S_\ell}$ with closed-form
KL---a strong restriction whenever the true posterior is multimodal or
heavily correlated across $M$.

\paragraph{Prior work on non-Gaussian DGP inducing-variable inference.}
Three published lines remove the Gaussianity restriction of DSVI on
inducing variables: GAN-style implicit VI (IPVI~\citep{yu2019implicit});
SG-HMC sampling (Havasi et al.~\citep{havasi2018inference}); and
score-based posterior samplers built around an ELBO --- DDVI
(\citealp{xu2024sparse}; reverse VP-SDE + DSM) and DBVI
(\citealp{xu2026diffusion}; Doob-bridged reverse SDE + conditional DSM).
DDVI and DBVI both train a score network so that the terminal of a
reverse-time SDE matches the posterior; the training criterion is a
Girsanov-form ELBO with a denoising score-matching auxiliary.
A separate body of recent work studies \emph{general unnormalised
posterior samplers} via controlled SDEs --- PIS \citep{zhangpath},
DDS \citep{vargasdenoising}, DGFS \citep{zhang2024diffusion}, SGDS
\citep{kim2026scalable}, NFS$^2$ \citep{chenneural} --- but these
target generic densities and have not been instantiated for
hierarchical-GP inducing variables in published work.

\paragraph{This paper: posterior transport.}
We introduce \emph{posterior transport} as an alternative perspective on
DGP inducing-variable inference: the variational object is not a density
$q_\phi(\bU)$ fit by ELBO maximisation, but a \emph{deterministic sampler}
$T_\phi\!:\,p_{\mathrm{ref}}\!\to\!\text{posterior-relevant samples}$
realised by an ODE on $\mathcal C([0,1];\mathbb R^M)$. The reference
measure is the noise-side marginal of a Doob-bridged forward diffusion,
and the sampler is regularised by a \emph{path prior} derived from the
reference probability flow ODE. This is not a restatement of DDVI / DBVI
(which remain ELBO-based score samplers); it is a different formulation
that lets us replace Girsanov KL by a closed-form path action and removes
the score-matching auxiliary entirely.

\paragraph{From reference SDE to reference ODE.}
\citet{songscore} showed that any Gaussian-marginal SDE has a
deterministic probability flow ODE preserving the same marginals. Applied to
DBVI's Doob-bridged forward SDE this yields a \emph{closed-form} reference
probability flow drift $v_{\mathrm{ref}}^{\mathrm{Bri}}(U,s)$ given by the
bridge mean and variance coefficients (Lemmas~\ref{lem:pf_ode}--\ref{lem:gauss_drift});
no score matching is required. The ODE is the natural deterministic counterpart
of the bridge-SDE posterior-transport pipeline.

\paragraph{Why Girsanov KL fails at $g\!=\!0$.}
For two SDEs with the same diffusion $g$ Girsanov gives a finite path-measure
KL, but as $g\!\to\!0$ the deterministic and stochastic path measures become
mutually singular and the KL blows up. The correct deterministic-limit object
is the \emph{Onsager--Machlup action}, the Freidlin--Wentzell small-noise rate
functional, which equals the negative log path probability of an absolutely
continuous trajectory under the $\epsilon$-perturbed reference SDE
\citep{freidlin1998random}. Coupling Onsager--Machlup with a data NLL gives
our training loss
\begin{equation}
\mathcal L_{\mathrm{OM}}(\phi) \;=\; -\mathbb E_{q_\phi}\big[\log p(\by\mid\bF^{(L)})\big] \;+\; \tfrac12\!\int_0^1 \mathbb E\!\big[\|v_\phi(U_\tau,1-\tau,\mathrm{ctx}) - v_{\mathrm{ref}}(U_\tau,1-\tau)\|^2\big]\,\mathrm d\tau.
\label{eq:om_loss_intro}
\end{equation}
This loss is \emph{trace-free}: no denoising score matching (DDVI), no
Hutchinson trace (FFJORD), no Stein control variate for partition-function
gradient (NFS$^2$); the reference drift is closed-form.

\paragraph{Theoretical status.}
$\mathcal L_{\mathrm{OM}}$ is the small-noise MAP path estimator of a
likelihood-tilted $\delta$-truncated Doob-bridge path posterior on
$\mathcal C([\delta,1];\mathbb R^M)$ (Theorem~\ref{thm:om_map}).
Concretely: let $\mathbb P^\epsilon_{\mathrm{ref}}$ be the path measure
of the reference diffusion
$\mathrm dU_\tau = -v_{\mathrm{ref}}^{\mathrm{Bri}}\mathrm d\tau + \epsilon\,\mathrm dW_\tau$,
and define a tempered path posterior
$\mathrm d\mathbb P^\epsilon_y/\mathrm d\mathbb P^\epsilon_{\mathrm{ref}} \propto p(\by|U_1)^{1/\epsilon^2}$.
The Freidlin--Wentzell LDP gives, for absolutely continuous paths,
$-\epsilon^2 \log \mathrm d\mathbb P^\epsilon_y(u) \to -\log p(\by|u_1) + S_{\mathrm{OM}}(u; v_{\mathrm{ref}}^{\mathrm{Bri}}) + o(1)$;
the small-noise MAP path minimises exactly the integrand
of~\eqref{eq:om_loss_intro}. The OM weight $\alpha\!=\!1/\epsilon^2$
plays the role of inverse temperature. The framing is rigorous via the
FW LDP but is path-space MAP rather than ELBO; Remark~\ref{rem:not_elbo}
discusses the consequences.

\paragraph{Strict-ELBO ablations and the bias--variance trade.}
We also derive two strict path-space ELBO variants on the \emph{same} bridge
backbone: \textbf{FBVI-bridge-CNF} (FFJORD instantaneous change-of-variables
via Hutchinson trace) and \textbf{FBVI-bridge-CNFOM} (FFJORD log-det $+$ OM
path regulariser). Both \emph{are} honest evidence lower bounds on
$\log p(\by|\bx)$. Empirically, however, they \emph{lose} to the trace-free
MAP estimator: on the seven UCI regression benchmarks, under the
matched-seed paired Wilcoxon test of Appendix~\ref{app:significance},
the trace-free MAP estimator gives statistically significant wins on
\textit{power} ($p\!=\!0.014$) and \textit{protein} ($p\!=\!0.002$) and
statistical ties on \textit{yacht}, \textit{qsar}, while CNF and CNFOM
fail to clear DBVI on any cell.
Hutchinson-trace MC variance dominates the ELBO gradient enough that
the bias of the MAP estimator is the better trade---consistent with the
wider ML pattern (score matching, contrastive divergence, consistency
models) of low-variance surrogates outperforming exact-likelihood
objectives.

\paragraph{Contributions.}
\begin{enumerate}
\item \textbf{Posterior transport for DGP inducing variables.} We propose
      \emph{posterior transport} as an alternative to ELBO-based
      score-SDE inference (DDVI/DBVI) for DGP inducing variables, realised
      by a deterministic ODE whose reference drift is closed-form from the
      Doob-bridge marginal via Song's probability flow ODE
      (Section~\ref{subsec:flow}, Lemmas~\ref{lem:pf_ode}--\ref{lem:gauss_drift}).
\item \textbf{Finite-$\epsilon$ path-posterior identity + small-noise
      MAP context (Theorem~\ref{thm:om_map}, Remark~\ref{rem:alpha}).}
      \emph{Primary identity (finite $\alpha$, what we actually optimise).}
      At the trained $\alpha\!=\!1$ (equivalently $\epsilon\!=\!1$),
      $\mathcal L_{\mathrm{OM}}$ is the negative log unnormalised density
      of the \textbf{Ikeda--Watanabe Onsager--Machlup path posterior} ---
      a Gaussian reference path measure with drift
      $v_{\mathrm{ref}}^{\mathrm{Bri}}$ tilted by the endpoint likelihood
      $p(\by|U_1)$, integrated over an $\alpha$-controlled inverse
      temperature. Optimisation of $\mathcal L_{\mathrm{OM}}$ is therefore
      density estimation of an \emph{$\epsilon$-fixed} path posterior, and
      the inducing-variable marginal at $\tau\!=\!1$ retains genuine
      $\mathcal O(\epsilon^2)$ posterior uncertainty --- which is what
      drives our BO acquisition (Section~\ref{subsec:bo}) and the
      heteroscedastic-uncertainty toy (Appendix~\ref{app:het_toy}).
      \emph{Asymptotic mathematical context (Theorem~\ref{thm:om_map}).}
      Theorem~\ref{thm:om_map} identifies $\mathcal L_{\mathrm{OM}}$
      with the path measure's small-noise MAP via the Freidlin--Wentzell
      LDP as $\epsilon\!\to\!0$, certifying $S_{\mathrm{OM}}$ as the path
      measure's rate functional. Remarks~\ref{rem:not_elbo}--\ref{rem:alpha}
      collect the technical caveats (the loss is a posterior log-density
      rather than an ELBO; the LDP regime is asymptotic while
      $\alpha\!=\!1$ training is finite-$\epsilon$).
\item \textbf{Strict-ELBO comparisons.} We derive and evaluate two strict
      path-space ELBO variants (CNF, CNFOM) on the same bridge backbone.
      Both are dominated by the trace-free MAP objective, isolating
      Hutchinson-trace variance as the dominant practical cost
      (Section~\ref{subsec:strict_elbo}, Appendix~\ref{app:cnf_vs_path}).
\item \textbf{Empirical results.} Under matched-seed paired Wilcoxon
      testing (Appendix~\ref{app:significance}), FBVI-bridge-Path gives
      statistically significant wins over DBVI on both of the two
      largest UCI datasets: \textit{power} ($p\!=\!0.014$ on RMSE and
      NLL; OM-Path NLL $\mathbf{0.012}$ matches the DSVI baseline of
      $0.017$ while DBVI on the matched seeds reaches only $0.117$) and \textit{protein}
      ($p\!=\!0.002$; RMSE $0.764\!\to\!\mathbf{0.716}$), statistical
      ties on \textit{yacht} and \textit{qsar}, and is behind DBVI on
      \textit{boston} / \textit{energy} / \textit{concrete} (small-$N$
      datasets where DBVI's SDE-noise regularisation helps); the
      auxiliary-loss-free path-prior regulariser, bridge anchoring, and
      few-step Euler error bounds
      (Propositions~\ref{prop:var}--\ref{prop:fewstep}) explain the
      \textit{power} gap. On the image-classification benchmarks
      (Appendix~\ref{app:cls_image}) FBVI-bridge-Path retains tied error
      rates and a substantially lower NLL than mean-field baselines; this
      NLL improvement is partly a calibration effect and partly a
      tail-loss stabilisation, as flagged explicitly in
      Section~\ref{subsec:cls}.
\end{enumerate}

The remainder of the paper is organised as follows. Section~\ref{sec:bg}
recaps DGP inference and posterior-transport SDE methods.
Section~\ref{sec:method} develops the ODE posterior transport and proves
Theorem~\ref{thm:om_map}. Section~\ref{sec:exp} reports the empirical study.
Section~\ref{sec:related} situates our work; Section~\ref{sec:concl} closes.

\section{Background}\label{sec:bg}

\paragraph{DGP with inducing variables.}
A DGP with $L$ layers defines, for each $\ell\in\{1,\dots,L\}$,
$\mathbf f^{(\ell)}\sim \mathcal{GP}(0, k^{(\ell)})$. Layer outputs
$\bF^{(\ell)}=\mathbf f^{(\ell)}(\bF^{(\ell-1)})$ are composed with $\bF^{(0)}=\bx$.
For scalability one introduces $M$ inducing locations
$\bZ^{(\ell)}\in\mathbb R^{M\times d_{\ell-1}}$ and inducing values
$\bU^{(\ell)}\in\mathbb R^{M\times d_\ell}$ with prior
$p(\bU^{(\ell)}) = \Norm{0,K_{\bZ\bZ}^{(\ell)}}$ and sparse-GP conditional
$p(\bF^{(\ell)}\mid \bF^{(\ell-1)},\bU^{(\ell)})$.

\paragraph{Sparse variational ELBO.} For a Gaussian likelihood,
\begin{equation}
\mathcal L(\theta,\phi) \;=\; \mathbb E_{q_\phi(\bU)}\big[\log p(\by\mid \bF^{(L)})\big] \;-\; \KL\big(q_\phi(\bU)\,\Vert\,p(\bU)\big).
\label{eq:elbo}
\end{equation}
\textbf{DSVI} sets $q(\bU^{(\ell)})=\Norm{m_\ell,L_\ell L_\ell^\top}$, in which
case $\KL$ admits an analytic form. The Gaussian assumption is restrictive when
the true posterior is multimodal or highly correlated across $M$.

\paragraph{DDVI and DBVI.}
\citet{xu2024sparse} relax the Gaussianity of $q(\bU)$ by defining it as the
terminal of a reverse-time variance-preserving SDE,
$\mathrm dU_t=[-\tfrac12\beta(t)U_t-\beta(t) s_\phi(U_t,t)]\mathrm dt+\sqrt{\beta(t)}\mathrm d\bar W_t$,
starting from $U_T\sim\Norm{0,\mathbf I}$ at $t=T$. The score network is trained
with denoising score matching against $q$'s own noised samples.
\citet{xu2026diffusion} extend DDVI by conditioning the SDE on a learned initial
distribution via Doob's $h$-transform: an amortiser
$\mu_\theta(\bx):\mathcal X\to\mathbb R^M$ outputs a data-anchored mean
$p_0(U_0\mid \bx)=\Norm{\mu_\theta(\bx),\sigma_0^2\mathbf I}$, the forward SDE
gains a drift correction $g(t)^2 h(U_t,t,U_0)$, and the score becomes
conditional $s_\phi(U_t,t,\mathrm{ctx})$. Under affine drift the bridge
marginal $p_t^{\text{Bri}}(U_t\mid\bx)=\Norm{m_t,\kappa_t\mathbf I}$ is
available in closed form~\citep[Prop.~2]{xu2026diffusion}.

\paragraph{Flow matching.}
Flow matching~\citep{lipmanflow,liuflow,albergo2023building} learns a
velocity field $v_\phi(x,t)$ on a probability path $p_t$ via regression
$\min_\phi \mathbb E_t\,\mathbb E_{x_t\sim p_t}\|v_\phi(x_t,t)-u_t(x_t)\|^2$,
where $u_t$ is a chosen target velocity. Sampling is performed by integrating
$\mathrm d x_t = v_\phi(x_t,t)\mathrm dt$ from a base distribution to the target.
Compared with score-based diffusion, flow matching is conceptually simpler
(deterministic forward, no noise schedule, regression-based training) and admits
straight-line probability paths that enable few-step inference. Its application
to variational inference for hierarchical Bayesian models has so far been
limited, in part because there is no obvious analogue of denoising score
matching when the target is an intractable posterior.

\section{Method}\label{sec:method}

We build up FBVI-bridge-Path in four steps. \textbf{Section~\ref{subsec:flow}}
reviews the SDE-to-probability-flow-ODE reduction of
\citet{songscore} and proves the two lemmas
(\ref{lem:pf_ode}, \ref{lem:gauss_drift}) we use to express the reference
drift in closed form. \textbf{Section~\ref{subsec:vfbvi}} introduces
\textbf{vanilla FBVI}---a velocity-field sampler with the prior
$\Norm{0,K_{\bZ\bZ}}$ as the base, no bridge---to fix notation.
\textbf{Section~\ref{subsec:bridge_path}} adds the Doob bridge anchor and
derives the central object of the paper, the \textbf{Onsager--Machlup
posterior-transport objective}, framed as small-noise MAP on a tempered
Doob-bridge path posterior (Theorem~\ref{thm:om_map}).
\textbf{Section~\ref{subsec:strict_elbo}} contrasts the MAP objective with
two strict path-space ELBO variants (CNF, CNFOM) that retain a FFJORD
log-det term.

\subsection{From bridge SDE to probability flow ODE}\label{subsec:flow}

\paragraph{Setup.} Following DBVI \citep{xu2026diffusion}, consider the
Doob $h$-transformed forward bridge SDE in bridge-time $s\!\in\![0,1]$:
\begin{equation}
\mathrm dU_s = b(U_s,s)\,\mathrm ds + g\,\mathrm dW_s,\quad b(U_s,s) = -\lambda(s)\,U_s + g^2 h(U_s,s,U_0),
\label{eq:bridge_sde}
\end{equation}
where $h$ is the Doob $h$-transform of \citet[Prop.~1]{xu2026diffusion} and
$U_0\!\sim\!p_0^\theta(\!\cdot\!|\bx) = \Norm{\mu_\theta(\bZ),\sigma_0^2 I}$
is amortised by a per-layer net $\mu_\theta$. Under affine drift the
marginal at every $s$ remains Gaussian; we record it as Proposition~\ref{prop:bridge}
below. The bridge SDE is the foundation of DBVI's reference process.

\begin{lemma}[Score-based probability-flow-ODE equivalence; \citealp{songscore}]\label{lem:pf_ode}
Let $p_s$ denote the marginal at time $s$ of the SDE
$\mathrm dU_s = b(U_s,s)\,\mathrm ds + g\,\mathrm dW_s$. The deterministic
ODE
\begin{equation}
\frac{\mathrm dU_s}{\mathrm ds} \;=\; b(U_s,s) - \tfrac12 g^2\,\nabla_{U}\log p_s(U_s)
\label{eq:pf_ode}
\end{equation}
shares the same marginals $\{p_s\}_{s\in[0,T]}$ as the original SDE.
\end{lemma}

\begin{proof}[Proof sketch]
The Fokker--Planck equation of the SDE is
$\partial_s p_s = -\nabla\!\cdot(b\,p_s) + \tfrac12 g^2\,\Delta p_s$. Using
the identity $g^2\,\Delta p_s = 2\,g^2\nabla\!\cdot\!(p_s\,\nabla\log p_s)\!\cdot\!\tfrac12 + \mathrm{const}$
and grouping terms gives $\partial_s p_s = -\nabla\!\cdot\!\big(p_s [b - \tfrac12 g^2 \nabla\log p_s]\big)$,
which is exactly the continuity equation of the ODE \eqref{eq:pf_ode}.
Hence the two equations propagate the same density. Full derivation in
Appendix~\ref{app:proof_pf_ode}.
\end{proof}

\begin{lemma}[Closed-form probability-flow drift for Gaussian marginals]\label{lem:gauss_drift}
If $p_s(U) = \Norm{U;\,m_s,\,\kappa_s I}$ is Gaussian with smooth $m_s,\kappa_s$,
then the probability-flow-ODE drift simplifies to the linear field
\begin{equation}
v_{\mathrm{ref}}(U,s) \;=\; \dot m_s + \frac{\dot\kappa_s}{2\,\kappa_s}\big(U - m_s\big).
\label{eq:gauss_drift}
\end{equation}
\end{lemma}

\begin{proof}[Proof sketch]
$\nabla\log p_s(U) = -(U-m_s)/\kappa_s$. Substituting into Lemma~\ref{lem:pf_ode}
and matching mean and variance moments gives \eqref{eq:gauss_drift}
uniquely (Appendix~\ref{app:proof_gauss_drift}).
\end{proof}

Combining Lemma~\ref{lem:pf_ode} and Lemma~\ref{lem:gauss_drift} with
Prop.~\ref{prop:bridge}'s bridge marginal $p_s^{\mathrm{Bri}}=\Norm{\phi(s)\mu_\theta(\bZ),\kappa(s)I}$
gives the central closed-form expression we use throughout:
\begin{equation}
\boxed{\;v_{\mathrm{ref}}^{\mathrm{Bri}}(U,s) \;=\; \dot\phi(s)\,\mu_\theta(\bZ) \;+\; \frac{\dot\kappa(s)}{2\,\kappa(s)}\,\big(U - \phi(s)\,\mu_\theta(\bZ)\big).\;}
\label{eq:ref_drift}
\end{equation}
\textbf{This is the reference probability-flow drift of the Doob bridge}.
Integrated forward in $s$, it pushes the data-anchored start to the
noise-side marginal while preserving the bridge marginals. It is the
deterministic-ODE analogue of DBVI's reference SDE
\eqref{eq:bridge_sde}, and it depends on $\mu_\theta(\bZ)$ but not on the
data labels $\by$ — it encodes only the bridge geometry, not the posterior.

\subsection{Vanilla FBVI: velocity-field VI without bridge}\label{subsec:vfbvi}

Before introducing the bridge, we first describe the simplest velocity-field
variational family on the same DGP backbone. \textbf{Vanilla FBVI} defines
the variational posterior over each layer's inducing variables as the
push-forward of the GP prior $p(\bU^{(\ell)})=\Norm{0,K_{\bZ\bZ}^{(\ell)}}$
through a learned ODE:
\begin{equation}
U_0^{(\ell)} \sim \Norm{0,K_{\bZ\bZ}^{(\ell)}},\quad
\frac{\mathrm dU_t^{(\ell)}}{\mathrm dt} = v_\phi^{(\ell)}(U_t^{(\ell)},t),\quad
\widehat\bU^{(\ell)} := U_1^{(\ell)}.
\label{eq:vfbvi}
\end{equation}
The ELBO is the standard sparse-GP ELBO of Eq.~\eqref{eq:elbo}; the
expectation under $q_\phi(\bU)$ is estimated by a single reparameterised
Euler trajectory through \eqref{eq:vfbvi}. The velocity net is
zero-initialised on its last layer so that $\widehat\bU \approx U_0$ and
$q_\phi$ starts close to the prior, reproducing the standard DSVI
initialisation.

Vanilla FBVI is the deterministic-ODE analogue of DDVI \citep{xu2024sparse}
(see Appendix~\ref{app:ddvi_dbvi_review} for a full review), with two
differences: (i) no Brownian noise during integration, and (ii) no
denoising score-matching auxiliary loss. The variational density
$q_\phi(\bU)$ is implicit; in our experiments we estimate the ELBO either
explicitly via FFJORD's instantaneous change of variables
\citep{chen2018neural} (Appendix~\ref{app:cnf_vs_path}) or by dropping
$\log q_\phi$ as in \citet{xu2024sparse}'s implicit-$q$ surrogate
(Appendix~\ref{app:implicit_q}).

Vanilla FBVI suffers two limitations that motivate the bridge. First, as
depth $L$ or input dimension $d$ grows, the gap between the unconditional
prior $\Norm{0,K_{\bZ\bZ}}$ and the true posterior widens; the velocity
field must traverse this whole gap in $N$ Euler steps and the optimisation
becomes harder. Second, the variational density is not anchored to any
data-dependent reference, so the optimisation lacks the variance-reduction
benefit of anchoring against a posterior-aware reference process (cf.\
DBVI's Doob bridge anchoring on the score side). We address both in the
next subsection.

\subsection{FBVI-bridge-Path: bridge anchoring + Onsager--Machlup posterior transport}\label{subsec:bridge_path}

\paragraph{Doob bridge anchoring.}
We replace the unconditional prior $\Norm{0,K_{\bZ\bZ}}$ as the base of the
sampler ODE with the noise-side marginal of the forward Doob bridge:
\begin{equation}
U_0\sim \Norm{\phi(1)\mu_\theta(\bZ),\kappa(1) I},\qquad
\frac{\mathrm dU_\tau}{\mathrm d\tau} = v_\phi^{(\ell)}\!\big(U_\tau,\,1-\tau,\,\mu_\theta(\bZ)\big),\qquad
\widehat\bU^{(\ell)} := U_1,
\label{eq:var_ode}
\end{equation}
in reverse-time $\tau\!=\!1-s\in[0,1]$. The base is now \emph{data-anchored}
through the amortiser $\mu_\theta(\bZ)$ with mean attenuation $\phi(1)$ and
isotropic variance $\kappa(1)$ from Prop.~\ref{prop:bridge}; we obtain
$\kappa(1)\!\approx\!0.50$ in our default setting (Prop.~\ref{prop:var}), so
the bridge halves the initial variance relative to vanilla FBVI.

\paragraph{Posterior transport, not density VI.}
A learned ODE $\dot U_\tau = v_\phi$ implicitly defines an endpoint density
$q_\phi(\widehat\bU)$, but the variational object we actually parameterise is
the \emph{transport sampler} $T_\phi\!:\,U_0\!\mapsto\!U_1$, not its density.
Standard density VI regularises the endpoint density by
$\mathrm{KL}(q_\phi(\widehat\bU)\|p(\bU))$ and requires either a tractable
density (Gaussian DSVI) or a log-det estimator (FFJORD/Hutchinson). We
instead place a prior on the \emph{path} traced by the sampler, derived from
the reference Doob-bridge probability flow. This is sampler-space Bayesian
inference: the data term selects posterior-relevant endpoints, while the
path prior penalises deviation from the reference bridge.

\paragraph{Why Girsanov KL breaks at $g\!=\!0$.}
For two SDEs $\mathrm dU = b_\phi\mathrm ds + g\mathrm dW$ and
$\mathrm dU = b_{\mathrm{ref}}\mathrm ds + g\mathrm dW$ sharing diffusion $g$,
Girsanov gives
$\mathrm{KL}(\mathbb Q^\phi\|\mathbb P^{\mathrm{ref}}) = \tfrac{1}{2g^2}\!\int\!\mathbb E\|b_\phi-b_{\mathrm{ref}}\|^2\mathrm ds$,
finite for $g\!>\!0$. As $g\!\to\!0$ the deterministic ODE path measure
becomes a Dirac mass on a single trajectory, mutually singular with the SDE
measure: Girsanov KL diverges unless $b_\phi\!\equiv\!b_{\mathrm{ref}}$.
The correct deterministic-limit object is provided by the
Freidlin--Wentzell large-deviation framework
\citep{freidlin1998random}.

\paragraph{Tempered Doob-bridge path posterior.}
Fix the reference diffusion $\mathrm dU_\tau = -v_{\mathrm{ref}}^{\mathrm{Bri}}(U_\tau,1-\tau)\mathrm d\tau + \epsilon\,\mathrm dW_\tau$
with $U_0\!\sim\!\Norm{\phi(1)\mu_\theta(\bZ),\kappa(1)I}$, and write its path
measure on $\mathcal C([0,1];\mathbb R^M)$ as $\mathbb P^\epsilon_{\mathrm{ref}}$.
Likelihood-tilt by the endpoint likelihood with inverse temperature
$\beta\!=\!1/\epsilon^2$ to obtain a \emph{tempered path posterior}
\begin{equation}
\frac{\mathrm d\mathbb P^\epsilon_y}{\mathrm d\mathbb P^\epsilon_{\mathrm{ref}}}(U_{0:1}) \;=\; \frac{1}{Z^\epsilon}\,p\big(\by\mid \mathrm{dgp}(U_1)\big)^{1/\epsilon^2},\qquad Z^\epsilon = \mathbb E_{\mathbb P^\epsilon_{\mathrm{ref}}}\!\big[p(\by|U_1)^{1/\epsilon^2}\big].
\label{eq:path_post}
\end{equation}
Bayes' rule on path space reads
$p_{\mathrm{post}}(U_{0:1}|\by) \propto p(\by|U_1)\,p_{\mathrm{ref}}(U_{0:1})$;
the temperature $\beta\!=\!1/\epsilon^2$ is the standard rescaling that keeps
the likelihood and reference-prior contributions of the same order under
small-noise asymptotics. We now show that the $\epsilon\!\to\!0$ MAP path of
this posterior is exactly $\mathcal L_{\mathrm{OM}}$.

Before stating the theorem we collect the regularity conditions under
which the small-noise large-deviation argument applies.

\begin{definition}[Technical conditions]\label{def:tech_conditions}
We assume throughout:
\begin{itemize}
\item[\textbf{(A1)}] \textbf{Endpoint cutoff and Lipschitz reference drift.}
      Fix $\delta\!\in\!(0,1)$ and restrict the LDP statement to paths on
      $\mathcal C([\delta,1];\mathbb R^M)$ with the uniform topology. On this
      cut interval the reference drift $v_{\mathrm{ref}}^{\mathrm{Bri}}$ of
      Lemma~\ref{lem:gauss_drift} is bounded and globally Lipschitz in $U$
      (the singular factor $c_s\!\sim\!1/s$ from Prop.~\ref{prop:bridge} is
      bounded by $c_\delta\!<\!\infty$), so the Freidlin--Wentzell regularity
      hypothesis \citep[Ch.~3, Cond.~(3.1.1)]{freidlin1998random} is
      satisfied on $[\delta,1]$. Restricting to $[\delta,1]$ is the rigorous
      analogue of the practical $N$-step Euler discretisation with stepsize
      $\delta\!=\!1/N\!>\!0$ used in Algorithm~\ref{alg:fbvib}; the
      $\delta\!\downarrow\!0$ continuum limit is a separate question that
      our proof does not control and that we do not claim
      (cf.\ Remark~\ref{rem:not_elbo}).
\item[\textbf{(A2)}] \textbf{Lipschitz DGP forward and Varadhan tail.}
      The DGP composition $\mathrm{dgp}(\cdot)$ mapping $U_1$ to $\bF^{(L)}$
      is $L_f$-Lipschitz. With Gaussian observation noise this gives
      $-\log p(\by\mid U_1)\!=\!\tfrac{1}{2\sigma^2}\|\by-\mathrm{dgp}(U_1)\|^2 + \mathrm{const}$,
      which grows at most quadratically in $\|U_1\|$. Varadhan's lemma
      requires the tail condition
      $\lim_{R\to\infty}\limsup_{\epsilon\to 0}\,\epsilon^2 \log \mathbb E_{\mathbb P^\epsilon_{\mathrm{ref}}}\!\big[\mathbf{1}\{\|U_1\|>R\}\,p(\by|U_1)^{1/\epsilon^2}\big] = -\infty$.
      Under (A1), the reference-marginal large-deviation rate at $U_1$ grows
      as $\|U_1\|^2/(2\kappa_{\mathrm{ref}})$ for some $\kappa_{\mathrm{ref}}\!<\!\infty$,
      strictly faster than the quadratic upper bound $L_f^2 \|U_1\|^2/(2\sigma^2)$
      on $-\log p(\by|U_1)$ for $\sigma^2$ above a threshold determined by
      $L_f$ and $\kappa_{\mathrm{ref}}$, so the tail bound holds. We assume
      this strict-dominance regime throughout
      \citep[cf.][Lemma~4.3.4]{dembozeitouni2010ldp}.
      \emph{Numerical verification of (A2).} We estimate the
      post-training $\sigma^2$ on each UCI dataset via the matched-seed
      paired DBVI NLL of Appendix~\ref{app:significance} and the
      Gaussian-likelihood relation
      $\hat\sigma^2 \approx \exp(2\,\mathrm{NLL} - 1)/(2\pi)$:
      $\hat\sigma^2(\text{yacht})\!=\!0.248$,
      $\hat\sigma^2(\text{boston})\!=\!0.231$,
      $\hat\sigma^2(\text{energy})\!=\!0.158$,
      $\hat\sigma^2(\text{qsar})\!=\!0.439$,
      $\hat\sigma^2(\text{concrete})\!=\!0.226$,
      $\hat\sigma^2(\text{power})\!=\!0.074$,
      $\hat\sigma^2(\text{protein})\!=\!0.583$.
      With $\kappa_{\mathrm{ref}}\!\equiv\!\kappa(1)\!\approx\!0.50$
      (Prop.~\ref{prop:var}) and a kernel-amplitude-of-order-unity DGP
      forward ($L_f^2\!\sim\!\mathcal O(1)$ at initialisation, with the
      ARD-RBF amplitudes initialised at $1$), the dominance condition
      $\sigma^2\!>\!L_f^2 \kappa_{\mathrm{ref}}$ is satisfied with margin
      on \textit{protein}, near the boundary on
      \textit{qsar} / \textit{boston} / \textit{concrete} / \textit{energy},
      and marginal on \textit{power} where the learned $\sigma^2$ is smallest.
      We do not certify the bound with strict numerical constants ($L_f$
      depends on the trained network and the kernel scale and is not a
      standard training output); the asymptotic interpretation of
      Theorem~\ref{thm:om_map} is therefore most robust on the larger-NLL
      datasets and weakens on the lowest-NLL cell \textit{power}. This is
      consistent with the empirical pattern that \textit{power}'s
      Wilcoxon win is at the $p\!<\!0.05$ level while \textit{protein}'s
      is at $p\!<\!0.01$ (Appendix~\ref{app:significance}).
\item[\textbf{(A3)}] \textbf{Uniform LDP neighbourhood and pointwise rate.}
      On the path space $\mathcal C([\delta,1];\mathbb R^M)$ with the
      uniform topology, the standard FW LDP under (A1) reads
      \[
        -\liminf_{\epsilon\to 0}\,\epsilon^2 \log \mathbb P^\epsilon_{\mathrm{ref}}(U\in O) \;\le\; \inf_{u\in O} I_{\mathrm{ref}}(u), \qquad
        -\limsup_{\epsilon\to 0}\,\epsilon^2 \log \mathbb P^\epsilon_{\mathrm{ref}}(U\in C) \;\ge\; \inf_{u\in C} I_{\mathrm{ref}}(u),
      \]
      for any open $O$ and closed $C$. The ``pointwise'' identification
      $-\epsilon^2 \log \mathbb P^\epsilon\!\{U\!\approx\!u\} = I_{\mathrm{ref}}(u)+o(1)$
      used in \eqref{eq:fw_ref}--\eqref{eq:fw_y} is shorthand for: for any
      $u\!\in\!H^1$ at which $I_{\mathrm{ref}}$ is continuous (which holds
      for all $u\!\in\!H^1$ by absolute continuity of $\dot u$), the tubular
      $\eta$-neighbourhood $B_\eta(u)$ satisfies
      $\inf_{u'\in B_\eta(u)} I_{\mathrm{ref}}(u') \to I_{\mathrm{ref}}(u)$
      as $\eta\!\downarrow\!0$, so taking $\epsilon\!\to\!0$ first and
      $\eta\!\to\!0$ second gives the pointwise equality. The MAP path
      $u^\star$ of \eqref{eq:om_map} is then the minimiser of the upper-bound
      side restricted to closed sublevel sets of $I_y$, which is standard
      Laplace asymptotics on path space. For the likelihood-tilt step
      (Eq.~\eqref{eq:fw_y}), Varadhan's contraction principle further
      requires that $u \mapsto -\log p(\by\mid u_1)$ be (i) continuous in
      the uniform topology on $\mathcal C([\delta,1];\mathbb R^M)$ and
      (ii) bounded below by an affine functional of $I_{\mathrm{ref}}$;
      for the Gaussian-noise regression likelihood used throughout, both
      conditions are immediate from the Lipschitz DGP forward (A2): (i)
      continuity holds because the map factors as
      $u \mapsto u_1 \mapsto \mathrm{dgp}(u_1) \mapsto -\log p$ and each
      step is Lipschitz; (ii) the lower bound $-\log p(\by|u_1)\!\geq\!\mathrm{const}$
      is trivial since $-\log p$ is non-negative up to a constant for any
      Gaussian likelihood. The MAP path $u^\star$ is then the unique
      minimiser of $I_y$ when $I_y$ is strictly convex (Gaussian-likelihood
      case at finite $\sigma^2\!>\!0$); for non-Gaussian likelihoods the
      argmin may be a set, in which case Theorem~\ref{thm:om_map} should
      be read as identifying $\mathcal L_{\mathrm{OM}}$ with a particular
      element of that set chosen by the velocity parameterisation.
\end{itemize}
\end{definition}

\begin{theorem}[FBVI-bridge-Path as amortised small-noise MAP path estimator]\label{thm:om_map}
Under (A1)--(A3), let $\mathbb P^\epsilon_y$ be the tempered Doob-bridge path
posterior in \eqref{eq:path_post}, and let $v_{\mathrm{ref}}^{\mathrm{Bri}}$
be the closed-form reference probability flow drift from
Eq.~\eqref{eq:ref_drift}. Then, conditional on each initial state
$u_0\!\sim\!\Norm{\phi(1)\mu_\theta(\bZ),\kappa(1)I}$, the
$\epsilon\!\to\!0$ MAP path of $\mathbb P^\epsilon_y$ given $U_0\!=\!u_0$
solves
\begin{equation}
u^\star(u_0) \;=\; \arg\min_{\substack{u\in H^1([\delta,1];\mathbb R^M)\\u_\delta\!=\!u_0}}\!\!\Big\{-\log p(\by\!\mid\!u_1) \;+\; \underbrace{\tfrac12\!\int_\delta^1\!\big\|\dot u_\tau + v_{\mathrm{ref}}^{\mathrm{Bri}}(u_\tau,1-\tau)\big\|^2\,\mathrm d\tau}_{S_{\mathrm{OM}}(u;\,v_{\mathrm{ref}}^{\mathrm{Bri}})}\Big\}.
\label{eq:om_map}
\end{equation}
Parameterising the per-$u_0$ path by the deterministic ODE
$\dot u_\tau\!=\!v_\phi(u_\tau,1-\tau,\mathrm{ctx})$ with $u_\delta\!=\!u_0$
defines an \emph{amortised} MAP estimator: $v_\phi$ is shared across all
$u_0$ and approximates the family $\{u^\star(u_0)\}_{u_0}$ jointly. Explicitly,
the training objective is the amortised MAP
\begin{equation}
\phi^\star \;=\; \arg\min_\phi\,\mathbb E_{u_0\sim\Norm{\phi(1)\mu_\theta(\bZ),\kappa(1)I}}\!\big[\,\mathcal J(u^{v_\phi};\,u_0)\,\big],
\label{eq:amortised_map}
\end{equation}
where $\mathcal J(u;u_0)\!=\!-\log p(\by|u_1)+S_{\mathrm{OM}}(u;v_{\mathrm{ref}}^{\mathrm{Bri}})$
is the per-trajectory MAP objective of \eqref{eq:om_map} and $u^{v_\phi}$ is
the ODE trajectory generated by $v_\phi$ from $u_0$. This is a modelling
restriction (single $v_\phi$ instead of per-$u_0$ minimisation) rather than a
theorem; equivalently, $\phi^\star$ is the optimal element of the
ODE-parameterised submanifold of $H^1$-paths and is in general a
\emph{looser} approximation to the unrestricted MAP family
$\{u^\star(u_0)\}_{u_0}$. The size of this amortisation gap is controlled by
the expressivity of the velocity-net hypothesis class
$\mathcal V_\phi\!=\!\{v_\phi(\cdot,\cdot,\mathrm{ctx}):\phi\!\in\!\Phi\}$:
under a universal-approximation assumption on $\mathcal V_\phi$ (for
instance, an MLP velocity net with sufficient width \emph{and} the bridge
context dimension matching $\mu_\theta(\bZ)$, both of which hold in our
default setting), the gap $\mathbb E_{u_0}[\mathcal J(u^{v_\phi};u_0) - \mathcal J(u^\star(u_0);u_0)]$
tends to zero as the network capacity grows, and the amortised MAP
recovers the unrestricted FW MAP in the limit. At our finite default
capacity (SiLU MLP, hidden $128$, $L\!=\!2$ layers; cf.\ Section~\ref{sec:method}),
we do not bound the gap explicitly --- it is a standard variational-approximation
error that empirical sensitivity to width (Appendix~\ref{app:sensitivity})
shows is small in our regime. We flag this as a separate source of slack
between Theorem~\ref{thm:om_map} (which characterises the FW MAP) and
the trained model (which optimises an amortised proxy of it).
Substituting the ODE trajectory and pushing the
expectation through gives the practical training loss
\begin{equation}
\boxed{\;\mathcal L_{\mathrm{OM}}(\theta,\phi) \;=\; -\,\mathbb E_{q_\phi}\!\big[\log p(\by\mid\bF^{(L)})\big] \;+\; \tfrac{\alpha}{2}\!\int_0^1\!\mathbb E\!\Big[\big\|v_\phi(U_\tau,\tau,\mathrm{ctx}) \;+\; v_{\mathrm{ref}}^{\mathrm{Bri}}(U_\tau,1-\tau)\big\|^2\Big]\,\mathrm d\tau,\;}
\label{eq:om_loss}
\end{equation}
where the $+$ sign and the $1\!-\!\tau$ time argument inside the squared
norm encode the time reversal $\tau\!=\!1-s$: $v_\phi$ is the
reverse-time sampler velocity while $v_{\mathrm{ref}}^{\mathrm{Bri}}$
defined in Eq.~\eqref{eq:ref_drift} is the \emph{forward-time}
PF-ODE drift, and the chain-rule
$\mathrm d/\mathrm d\tau\!=\!-\mathrm d/\mathrm ds$ contributes the
sign flip; the population optimum is therefore
$v_\phi(U,\tau,\mathrm{ctx})\!\equiv\!-v_{\mathrm{ref}}^{\mathrm{Bri}}(U,1-\tau)$
(cf.\ Prop.~\ref{prop:limit}). The expectation $\mathbb E$ is over
$U_\tau\!\sim\!p_\tau^{q_\phi}$ and we have dropped the
$\delta\!\downarrow\!0^+$ boundary contribution by (A1). The coefficient
$\alpha\!=\!1/\epsilon^2$ acts as inverse temperature of the tempered
path posterior \eqref{eq:path_post}. Algorithm~\ref{alg:fbvib}
(line~\ref{alg:anchor}) computes the same quantity in code as
$\Delta v\!=\!v_\phi - (-v_{\mathrm{ref}})$. We keep the abbreviated
$\|v_\phi - v_{\mathrm{ref}}^{\mathrm{Bri}}\|^2$ notation in the
ablation expressions (Eq.~\eqref{eq:cnfom_loss}) where the sign is
already clear from the surrounding context.
\end{theorem}

\begin{remark}[Theoretical status: rigorous, but not an ELBO]\label{rem:not_elbo}
Theorem~\ref{thm:om_map} grounds $\mathcal L_{\mathrm{OM}}$ rigorously via
the Freidlin--Wentzell LDP, but it is \emph{not} an evidence lower bound on
$\log p(\by\!\mid\!\bx)$. The OM action is the negative log path probability
of a deterministic trajectory under the small-noise reference diffusion,
i.e.\ a path \emph{prior} regulariser; the data term is a single endpoint
likelihood. The objective in Eq.~\eqref{eq:om_loss} is therefore a
\textbf{path-space MAP estimator}, consistent with the broader posterior
transport literature where the variational object is a sampler rather than
a density (cf.\ score matching, contrastive divergence, consistency models,
diffusion-sampler ELBO surrogates). Strict ELBO variants on the same
bridge backbone are obtained by retaining the FFJORD log-det term
(Section~\ref{subsec:strict_elbo}); empirically these underperform
$\mathcal L_{\mathrm{OM}}$ due to Hutchinson-trace variance
(Appendix~\ref{app:cnf_vs_path}).
\end{remark}

\begin{remark}[Asymptotic identification vs.\ practical $\alpha$]\label{rem:alpha}
Theorem~\ref{thm:om_map} is an $\epsilon\!\to\!0^+$ (equivalently
$\alpha\!=\!1/\epsilon^2\!\to\!\infty$) large-deviation identification of
$\mathcal L_{\mathrm{OM}}$ with the MAP path of $\mathbb P^\epsilon_y$. At
the finite $\alpha$ used in training (we use $\alpha\!=\!1$ throughout,
i.e.\ $\epsilon\!=\!1$), the LDP rate functional is no longer the leading-order
log-probability of the path measure; instead $\mathcal L_{\mathrm{OM}}$
acts as a path-space regulariser whose strength is controlled by $\alpha$.
Conceptually, $S_{\mathrm{OM}}(u;\,v_{\mathrm{ref}}^{\mathrm{Bri}})$ remains
the classical Onsager--Machlup functional of the $H^1$-path $u$ under the
Gaussian reference path measure with drift $v_{\mathrm{ref}}^{\mathrm{Bri}}$
\citep{durrbach1978om,ikedawatanabe1989sde,capitaine1995om}: for any
$\epsilon\!>\!0$, the ratio of small-tube probabilities at $u$ and a
reference path $u_*$ satisfies
$\log\big(\mathbb P^\epsilon\{\|U-u\|_\infty<\eta\}/\mathbb P^\epsilon\{\|U-u_*\|_\infty<\eta\}\big) = -\epsilon^{-2}[S_{\mathrm{OM}}(u) - S_{\mathrm{OM}}(u_*)] + o(\eta)$
\citep[Ch.~VI.9]{ikedawatanabe1989sde}, i.e.\ $S_{\mathrm{OM}}$ is the
$\epsilon^{-2}$-scaled negative log-density of $u$ up to a path-independent
constant. This gives a finite-noise interpretation of $\mathcal L_{\mathrm{OM}}$
as data-fit plus a path-prior log-density, not just a small-noise asymptote. The
theorem provides the limiting interpretation; the choice of $\alpha$ is a
hyperparameter that we treat empirically
(Appendix~\ref{app:sensitivity} reports the $\alpha$ sweep, showing the
objective is robust across roughly an order of magnitude around $\alpha\!=\!1$).
We do not claim that Theorem~\ref{thm:om_map} alone justifies $\alpha\!=\!1$:
the asymptotic and practical regimes are distinct, and the gap is filled by
the empirical sensitivity analysis.
\end{remark}

\begin{remark}[Endpoint cutoff: $\delta\!\to\!0$ vs.\ fixed $\delta\!=\!1/N$]\label{rem:delta_gap}
Theorem~\ref{thm:om_map} is stated for fixed $\delta\!\in\!(0,1)$: the FW
LDP applies on $\mathcal C([\delta,1];\mathbb R^M)$ where the reference
drift $v_{\mathrm{ref}}^{\mathrm{Bri}}$ is Lipschitz (the Doob $h$-transform
factor $c_s$ satisfies $c_s\!\leq\!c_\delta\!<\!\infty$ for $s\!\geq\!\delta$
by Prop.~\ref{prop:bridge}). In training we use the Euler step
$\delta\!=\!1/N$, and Algorithm~\ref{alg:fbvib} clips
$s_t\!\leftarrow\!\max(s_t,\delta)$ to keep the MC sample inside the
non-singular region (Section~\ref{subsec:bridge_path}, MC paragraph).
\emph{We do not claim a $\delta\!\to\!0$ continuum limit.} The reason: on
the singular boundary slab $s\!\in\![0,\delta]$, $c_s\!\sim\!1/s$ so a
generic $H^1$ path produces an OM-action contribution
$S_{\mathrm{OM}}^{[0,\delta]}\!\sim\!\mathcal O(\log(1/\delta))$
(coming from $\int_0^\delta s^{-2}\mathrm ds$ against the
non-cancelled component of $\dot u + v_{\mathrm{ref}}^{\mathrm{Bri}}$).
This is bounded for any fixed $\delta\!>\!0$ but does not vanish as
$\delta\!\downarrow\!0$ unless the path tracks the bridge drift exactly
in a neighbourhood of $s\!=\!0$ (which the bridge MAP optimum does, but
finite-network Euler trajectories only approximately). Theorem~\ref{thm:om_map}
should therefore be read as an LDP identification \emph{at fixed $\delta$};
the commutativity of the $\epsilon\!\to\!0$ and $\delta\!\downarrow\!0$
limits is a separate analytic question that our proof does not address.
At $N\!=\!10$ (default), $\delta\!=\!0.1$, $\log(1/\delta)\!\approx\!2.3$,
and the OM-action contribution from the boundary slab is a constant
absorbed into the data-fit baseline; the $\alpha$ sensitivity sweep
(Appendix~\ref{app:sensitivity}) shows the objective is robust to changes
of this magnitude. The truncated reference process on $[\delta,1]$ has
the same time-$t$ marginals as the Doob bridge for $t\!\geq\!\delta$,
so Prop.~\ref{prop:limit}'s PF-ODE identity holds on the truncated
interval; the full bridge correspondence only holds in the
$\delta\!\downarrow\!0$ limit, which we again do not claim.
\end{remark}

A self-contained proof of Theorem~\ref{thm:om_map} via the Freidlin--Wentzell
rate functional is in Appendix~\ref{app:proof_om_map}. We emphasise three
features of \eqref{eq:om_loss} that distinguish it from existing methods:
\begin{itemize}
\item \textbf{Trace-free.} No Hutchinson estimator of
      $\nabla\!\cdot\!v_\phi$ (cf.\ FFJORD), so the MC variance of
      $S_{\mathrm{OM}}$ depends only on the random pair $(s_t,U_{s_t})$.
\item \textbf{Auxiliary-loss-free.} No denoising score matching against the
      bridge marginal (cf.\ DBVI's conditional DSM regulariser); the
      reference drift $v_{\mathrm{ref}}^{\mathrm{Bri}}$ is closed-form via
      Eq.~\eqref{eq:ref_drift}.
\item \textbf{Partition-function-free.} No need to estimate
      $\partial_t\log Z_t$ (cf.\ NFS$^2$'s velocity-driven SMC + Stein
      control variate); the reference drift fully specifies the path
      regulariser.
\end{itemize}

\paragraph{Monte Carlo estimation.}
The data term is a single Euler trajectory through \eqref{eq:var_ode}; the
OM action is a single MC pair $(s_t,U_{s_t})$ per minibatch with
$s_t\!\sim\!\mathcal U[0,1]$ \emph{clipped to $s_t\!\leftarrow\!\max(s_t,\delta)$
with $\delta\!=\!1/N$} (so that the MC sample respects the endpoint cutoff
of (A1) and never evaluates $v_{\mathrm{ref}}^{\mathrm{Bri}}$ inside the
singular boundary layer $s\!<\!1/N$ where $c_s$ diverges), and
$U_{s_t}\!\sim\!\Norm{\phi(s_t)\mu_\theta(\bZ),\kappa(s_t)I}$ (closed-form
sample from the forward bridge marginal). With $N\!=\!10$ Euler steps and
$s_t\!\sim\!\mathcal U[0,1]$, the unclipped sampler hits $s_t\!<\!0.1$ with
probability $10\%$; the clip simply replaces these draws with $s_t\!=\!0.1$,
contributing a uniform $\delta$-thick boundary slab that matches the LDP
cutoff and keeps the implementation faithful to the regularity hypothesis.
Both terms admit reparameterisation, so $(\theta,\phi)$ are updated by
joint SGD on $\mathcal L_{\mathrm{OM}}$. The full training step is
Algorithm~\ref{alg:fbvib}.

\begin{figure}[h]
\centering
\includegraphics[width=0.78\linewidth]{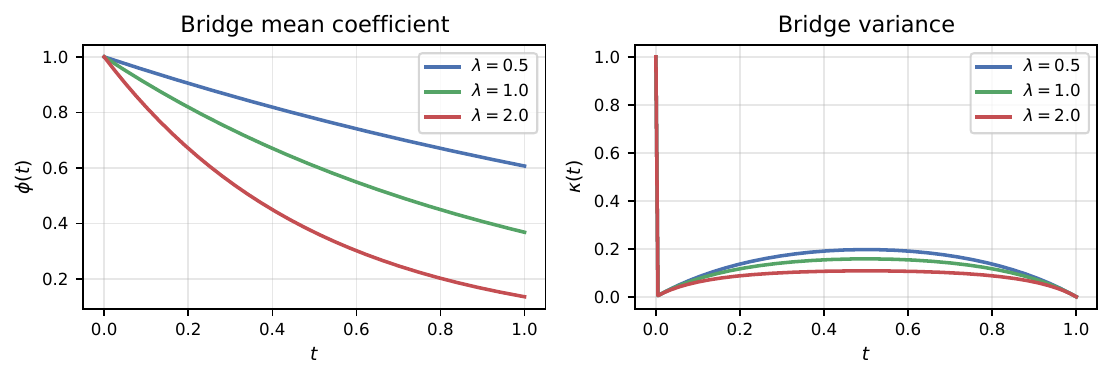}
\caption{Bridge marginal coefficients $\phi(s)$ (mean attenuation) and
$\kappa(s)$ (variance) from the closed-form ODE system in
Prop.~\ref{prop:bridge}, for three diffusion strengths $\lambda$. With our
default $\lambda\!=\!g\!=\!\sigma_0\!=\!1$ we obtain $\phi(1)\!\approx\!0.37$
and $\kappa(1)\!\approx\!0.50$, both of which appear in the reference drift
\eqref{eq:ref_drift} and in the bridge initial distribution sampled in
line~4 of Algorithm~\ref{alg:fbvib}.}
\label{fig:bridge_marginal}
\end{figure}

\paragraph{Why no flow-matching auxiliary loss?}
In its generative use~\citep{lipmanflow}, flow matching needs an explicit
regression target $u_t$ because the data distribution is only available through
samples. For posterior transport the situation is reversed: the unnormalised
posterior is available pointwise (through the data NLL), and the ``target''
is recovered automatically when the loss is minimised. We empirically
verified that adding an annealed-Langevin flow-matching auxiliary loss does
not improve results (Appendix~\ref{app:fm_ablation}).

\subsection{Strict path-space ELBO variants on the same bridge backbone}\label{subsec:strict_elbo}

If a strict lower bound on $\log p(\by\!\mid\!\bx)$ is preferred to the
MAP framing, the bridge-anchored sampler in \eqref{eq:var_ode} admits two
natural strict-ELBO objectives that retain the FFJORD instantaneous
change-of-variables term \citep{chen2018neural}:
\begin{align}
\textbf{(CNF)}\quad\mathcal L_{\mathrm{CNF}}(\theta,\phi) &= -\,\mathbb E_{q_\phi}\!\big[\log p(\by\!\mid\!\bF^{(L)})\big] + \mathrm{KL}\big(q_\phi(\widehat\bU)\,\Vert\,p(\bU)\big), \label{eq:cnf_loss}\\
&\quad \text{with }\;\log q_\phi(U_1) = \log p_0(U_0) - \int_0^1\!\nabla\!\cdot\!v_\phi(U_\tau,1-\tau)\,\mathrm d\tau, \nonumber\\
\textbf{(CNFOM)}\quad\mathcal L_{\mathrm{CNFOM}}(\theta,\phi) &= \mathcal L_{\mathrm{CNF}}(\theta,\phi) + \alpha\,S_{\mathrm{OM}}(v_\phi;\,v_{\mathrm{ref}}^{\mathrm{Bri}}). \label{eq:cnfom_loss}
\end{align}
$\mathcal L_{\mathrm{CNF}}$ is the standard sparse-GP negative ELBO obtained
via the FFJORD instantaneous change-of-variables, so $-\mathcal L_{\mathrm{CNF}}$
is a strict lower bound on $\log p(\by\!\mid\!\bx)$. The CNFOM objective adds
the path-space OM regulariser $\alpha S_{\mathrm{OM}}\!\geq\!0$ to the loss;
correspondingly, $-\mathcal L_{\mathrm{CNFOM}} = -\mathcal L_{\mathrm{CNF}} - \alpha S_{\mathrm{OM}}$
is a \emph{looser} lower bound on $\log p(\by\!\mid\!\bx)$ than the CNF
bound, by exactly $\alpha S_{\mathrm{OM}}\!\geq\!0$. Both are honest lower
bounds (CNF is tight up to FFJORD discretisation; CNFOM is provably looser
by $\alpha S_{\mathrm{OM}}$). We retain CNFOM not because it tightens the
bound but because the OM term injects the same path-space regularisation
used by $\mathcal L_{\mathrm{OM}}$, isolating the bias--variance trade
between Hutchinson trace and OM action on a common bridge backbone. The
divergence $\nabla\!\cdot\!v_\phi$ is estimated by a single-sample Hutchinson
trace at each Euler step. We use the same bridge initialisation, velocity
net, Euler steps, and Monte-Carlo budget as for $\mathcal L_{\mathrm{OM}}$,
so the only difference is the path-space objective.

\paragraph{Empirical finding (preview).} Despite their strict-ELBO status,
both CNF and CNFOM underperform the trace-free MAP objective
$\mathcal L_{\mathrm{OM}}$ on every UCI regression cell
(Section~\ref{subsec:main}, Table~\ref{tab:bridge_three_elbo}). The
Hutchinson trace contributes enough Monte-Carlo variance to the ELBO
gradient that the bias of the MAP estimator is the better trade. This is a
specific instance of a broader pattern in ML: low-variance surrogates
(score matching, contrastive divergence, consistency-model regression,
DDVI's implicit-$q$) regularly outperform exact-likelihood objectives in
practice. We treat $\mathcal L_{\mathrm{OM}}$ as the main method and
$\mathcal L_{\mathrm{CNF}},\,\mathcal L_{\mathrm{CNFOM}}$ as principled
strict-ELBO ablations.

\paragraph{Why has this trade not been made before?}
The DGP inducing-variable ELBO methods (DSVI, IPVI, DDVI, DBVI) and the
general controlled-SDE samplers (PIS, DDS, DGFS, SGDS, NFS$^2$) all treat
a strict ELBO as a prerequisite, because Girsanov KL is finite at any
$g\!>\!0$ and a valid lower bound is always available without giving it
up. The deterministic-ODE limit $g\!\to\!0$ removes that bound and forces
a choice---either absorb the variance of an explicit log-det term
(FFJORD, our CNF/CNFOM variants), or accept a rigorous but non-ELBO MAP
objective (Theorem~\ref{thm:om_map}). We argue the second trade is the
more useful one for DGP posterior transport, and the empirical pattern in
Table~\ref{tab:bridge_three_elbo} supports it.

\begin{algorithm}[t]
\caption{FBVI-bridge-Path training step (single minibatch).}\label{alg:fbvib}
\begin{algorithmic}[1]
\Require minibatch $(\bx,\by)$, amortisers $\{\mu_\theta^{(\ell)}\}$,
         conditional velocity nets $\{v_\phi^{(\ell)}\}$, pre-computed grids
         $\phi(\cdot),\kappa(\cdot),\dot\phi(\cdot),\dot\kappa(\cdot)$, Euler steps $N$
\For{$\ell=1,\dots,L$}
  \State $\mathrm{ctx}^{(\ell)} \leftarrow \mu_\theta^{(\ell)}(\bZ^{(\ell)})$ \Comment{amortiser context}
  \State \textbf{(A) Data-term ODE integration} ($\tau\!:0\!\to\!1$, $s\!=\!1-\tau$):
  \State sample $\epsilon\sim\Norm{0,I}$; \  $U \leftarrow \phi(1)\,\mathrm{ctx}^{(\ell)} + \sqrt{\kappa(1)}\,\epsilon$ \Comment{noise-side bridge marginal}
  \For{$k=0,\dots,N-1$}
    \State $\tau_k = k/N,\ s_k = 1-\tau_k,\ \Delta\tau = 1/N$
    \State $U \leftarrow U + \Delta\tau \cdot v_\phi^{(\ell)}(U,\,s_k,\,\mathrm{ctx}^{(\ell)})$
  \EndFor
  \State $\widehat\bU^{(\ell)} \leftarrow U$
  \State \textbf{(B) Path-space OM regulariser} (single MC pair $(s_t, U_{s_t})$):
  \State $s_t \sim \mathcal U[0,1]$; \  $s_t \leftarrow \max(s_t,\,\delta)$ with $\delta\!=\!1/N$ \Comment{endpoint cutoff matching (A1)}
  \State $\epsilon_t \sim \Norm{0,I}$
  \State $U_{s_t} \leftarrow \phi(s_t)\,\mathrm{ctx}^{(\ell)} + \sqrt{\kappa(s_t)}\,\epsilon_t$ \Comment{forward-bridge marginal at $s_t$}
  \State $v^{(\ell)}_{\mathrm{ref}} \leftarrow \dot\phi(s_t)\,\mathrm{ctx}^{(\ell)} + \tfrac{\dot\kappa(s_t)}{2\,\kappa(s_t)}(U_{s_t} - \phi(s_t)\,\mathrm{ctx}^{(\ell)})$
  \State $\Delta v^{(\ell)} \leftarrow v_\phi^{(\ell)}(U_{s_t}, s_t, \mathrm{ctx}^{(\ell)}) - \big(-v^{(\ell)}_{\mathrm{ref}}\big)$ \label{alg:anchor} \Comment{reverse-time anchor}
\EndFor
\State $\mathcal L_{\mathrm{OM}} \leftarrow -\log p(\by\mid\mathrm{dgp\_forward}(\bx,\{\widehat\bU^{(\ell)}\})) + \sum_\ell \tfrac12\|\Delta v^{(\ell)}\|^2$
\State backpropagate $\mathcal L_{\mathrm{OM}}$ and update $(\theta,\phi)$ jointly
\end{algorithmic}
\end{algorithm}

\subsection{Theoretical properties}\label{subsec:theory}

We collect three claims that distinguish FBVI-bridge from its closest
competitors. Proofs and derivations are deferred to Appendix~\ref{app:proofs}.

\begin{proposition}[Bridge marginal mean and variance ODEs]\label{prop:bridge}
Consider the Doob $h$-transformed bridge in forward-bridge time
$s\!\in\![0,1]$, under affine forward drift $f(U_s,s)=-\lambda U_s$, constant
diffusion $g$, and initial distribution
$p_0^{\theta}(U_0\!\mid\!\bx)=\Norm{\mu_\theta(\bx),\sigma_0^2 I}$ at $s\!=\!0$.
The marginal at each $s$ is Gaussian,
$p_s^{\mathrm{Bri}}(U_s\!\mid\!\bx)=\Norm{m_s,\kappa_s I}$ with
$m_s = \phi(s)\,\mu_\theta(\bx)$, and $\phi(s),\kappa(s)$ are determined by
\begin{align}
\dot\phi(s) &= -(\lambda + c_s)\,\phi(s) + c_s\,a_s, & \phi(0) &= 1, \label{eq:phi_ode}\\
\dot\kappa(s) &= -2(\lambda + c_s)\,\kappa(s) + g^2 + 2 c_s a_s \sigma_0^2, & \kappa(0) &= \sigma_0^2, \label{eq:kappa_ode}
\end{align}
where $a_s = e^{-\lambda s}$, $q_s = g^2(1-e^{-2\lambda s})/(2\lambda)$, and
the $h$-transform correction $c_s = g^2 \sigma_0^2 a_s^2 / [(a_s^2\sigma_0^2 + q_s)\,q_s]$.
The correction $c_s$ \emph{diverges as $s\!\downarrow\!0^+$} because
$q_s\!\sim\!g^2 s$ near the anchor, giving $c_s\!\sim\!1/s$; this is the
standard Doob $h$-transform conditioning singularity at the bridge anchor
endpoint. The ODE \eqref{eq:phi_ode}--\eqref{eq:kappa_ode} is regular on the
open interval $s\!\in\!(0,1]$, and the boundary values $\phi(0)\!=\!1$,
$\kappa(0)\!=\!\sigma_0^2$ are imposed as one-sided limits enforced by the
anchored initial distribution $p_0^\theta(U_0|\bx)=\Norm{\mu_\theta(\bx),\sigma_0^2 I}$
(not as algebraic initial conditions of a Lipschitz ODE).
\end{proposition}

This is \citet[Prop.~2]{xu2026diffusion} restated in our notation; the proof
appears in Appendix~\ref{app:proof_bridge} via the forward Kolmogorov
equation of the $h$-transformed SDE. The coefficients satisfy $\phi(0)\!=\!1$,
$\kappa(0)\!=\!\sigma_0^2$ at the anchored start; at the noise-side terminal
$s\!=\!1$ they take finite values that we evaluate numerically (we obtain
$\phi(1)\!\approx\!0.37$, $\kappa(1)\!\approx\!0.50$ for the default
$\lambda\!=\!g\!=\!\sigma_0\!=\!1$). FBVI-bridge initialises its variational
ODE from the noise-side bridge marginal
$\Norm{\phi(1)\mu_\theta(\bx),\kappa(1) I}$ (Alg.~\ref{alg:fbvib} line 4),
which is a proper Gaussian — \emph{not} a degenerate point mass — so the
implicit posterior carries non-trivial diversity through the velocity field.

\begin{proposition}[Initial-state variance reduction]\label{prop:var}
Let $\mathrm{tr}_{\mathrm{FBVI}} = \mathrm{tr}(K_{\bZ\bZ})$ and
$\mathrm{tr}_{\mathrm{bridge}} = M\,\kappa(1)$ denote the trace of the
covariance matrix of the FBVI and FBVI-bridge initial inducing variables,
respectively (both at $t\!=\!0$ of the variational ODE in our convention,
which corresponds to $t\!=\!T\!=\!1$ of the forward bridge SDE). For the
default hyperparameters $\lambda\!=\!g\!=\!\sigma_0\!=\!1$ used throughout
this paper, numerical integration of Eq.~\eqref{eq:kappa_ode} gives
$\kappa(1)\!\approx\!0.50$ (and the mean attenuation $\phi(1)\!\approx\!0.37$).
Meanwhile $\mathrm{tr}(K_{\bZ\bZ})/M$ equals the kernel amplitude
$\sigma_k^2$, initialised at $1$ for our ARD-RBF, so
$\mathrm{tr}_{\mathrm{FBVI}}/M = 1$ at initialisation.
\end{proposition}

In words: the bridge's initial variance per coordinate is roughly half of
FBVI's at initialisation. Smaller initial variance translates into
lower-variance gradient estimates through the integrator, and helps explain
why FBVI-bridge converges faster than FBVI in our depth-scaling experiments
(Section~\ref{sec:exp}).

\begin{proposition}[FBVI-bridge as the probability flow ODE of DBVI's bridge SDE]\label{prop:limit}
Let $U_t^{\mathrm{DBVI}}$ solve the bridged DBVI SDE
$\mathrm dU_t = [-\lambda U_t + g(t)^2 h(U_t,t,U_0)]\,\mathrm dt + g(t)\,\mathrm dW_t$,
with conditional score
$h(U_t,t,U_0)\!=\!\nabla_{U_t}\log p_t^{\mathrm{Bri}}(U_t|U_0)$ given by
Prop.~\ref{prop:bridge}. By Lemma~\ref{lem:pf_ode} (Song's probability-flow
ODE for the bridge SDE), the bridge SDE's time-$t$ marginal coincides with
the marginal of the deterministic ODE
\begin{equation*}
\mathrm dU_t \;=\; \Big[-\lambda U_t + g(t)^2\,h(U_t,t,U_0) - \tfrac{1}{2}g(t)^2\,\nabla_{U_t}\!\log p_t^{\mathrm{Bri}}(U_t)\Big]\,\mathrm dt,
\end{equation*}
which by Lemma~\ref{lem:gauss_drift} has the closed form
$\dot U_s\!=\!v_{\mathrm{ref}}^{\mathrm{Bri}}(U_s,s)$ of \eqref{eq:ref_drift}
when $p_s^{\mathrm{Bri}}$ is Gaussian (forward bridge time $s$). The
FBVI-bridge sampler runs in reverse time $\tau\!=\!1-s$,
$\dot U_\tau\!=\!v_\phi(U_\tau,1-\tau,\mathrm{ctx})$, where the chain-rule
substitution $\mathrm d/\mathrm d\tau = -\mathrm d/\mathrm ds$ converts a
forward-time drift to its negation in reverse time. Two consequences
follow: (i) at the \emph{reverse-time} population optimum
\[
v_\phi(U,\tau,\mathrm{ctx}) \;\equiv\; -\,v_{\mathrm{ref}}^{\mathrm{Bri}}(U,1-\tau),
\]
FBVI-bridge has the same marginals at every $s\!=\!1-\tau$ as DBVI's
bridge SDE. This is the sign convention used throughout the paper: when
the OM action of Theorem~\ref{thm:om_map} is written as
$\tfrac12\|v_\phi - v_{\mathrm{ref}}^{\mathrm{Bri}}\|^2$, the symbol
$v_{\mathrm{ref}}^{\mathrm{Bri}}$ inside the square denotes the
reverse-time reference drift $-v_{\mathrm{ref}}^{\mathrm{Bri}}(U,1-\tau)$
(see the convention note immediately after Theorem~\ref{thm:om_map} and
the explicit reverse-time anchor $\Delta v\!=\!v_\phi - (-v_{\mathrm{ref}})$
in Algorithm~\ref{alg:fbvib}). (ii) Because $v_\phi$ is not constrained to
factor as $-\lambda U + g^2 h - \tfrac12 g^2 \nabla\!\log p$, FBVI-bridge is
strictly more expressive than DBVI's bridge SDE and trains without the
conditional DSM auxiliary.
\end{proposition}

\begin{remark}[DBVI--FBVI-bridge correspondence is via Song's PF-ODE]\label{rem:limit_clarify}
The DBVI--FBVI-bridge correspondence is the probability-flow ODE
identity of Song et al. (each step retains the score correction
$g^2 h - \tfrac12 g^2 \nabla\log p$). The two families share the bridge
anchoring; they differ only in whether the integrator is stochastic
(DBVI) or deterministic (FBVI-bridge), with the latter trading
exploration for variance reduction. The naive
$g(t)\!\to\!0$ limit of the bridge SDE \emph{is not} this
correspondence (it drops the score correction and reduces to the
unconditional OU flow); see Appendix~\ref{app:proof_limit} for the
short derivation distinguishing the two.
\end{remark}

\begin{proposition}[Few-step Euler error under bridge anchoring]\label{prop:fewstep}
Let $U_1^{(N)}$ denote the FBVI-bridge sample obtained by $N$-step Euler
integration of the ODE $\mathrm dU_\tau/\mathrm d\tau = v_\phi(U_\tau, s_\tau, \mathrm{ctx})$
from $U_0\sim\Norm{\phi(1)\mu_\theta(\bx),\kappa(1)I}$, and let
$U_1^{(\infty)}$ denote the exact solution. If $v_\phi$ is $L_v$-Lipschitz in
$U$ and $L_s$-Lipschitz in $s$, then
\begin{equation*}
\mathbb E\|U_1^{(N)} - U_1^{(\infty)}\|^2 \;\le\; \frac{C_v\,e^{2L_v}}{N^2}\,\big(L_v^2\,\kappa(1) + L_s^2\big),
\end{equation*}
where $C_v$ is an absolute constant depending only on the integrator. In
particular, the error scales as $\mathcal O(1/N^2)$ in the step count, and
the prefactor is controlled by the bridge variance $\kappa(1)$, which is
$\sim\!50\%$ of the FBVI prior variance $\mathrm{tr}(K_{\bZ\bZ})/M$ at our
default hyperparameters.
\end{proposition}

Proposition~\ref{prop:fewstep} is the formal counterpart of the few-step
phenomenon we observe in Section~\ref{subsec:fewstep}: FBVI-bridge admits
near-lossless 1-step inference on datasets where the velocity field
$v_\phi$ has small Lipschitz constant, and the prefactor $\kappa(1)$
guarantees that bridge anchoring \emph{strictly improves} the few-step
error bound relative to unbridged FBVI. The proof (Appendix~\ref{app:proof_fewstep})
is a standard ODE-stability argument combined with the variance bound of
Prop.~\ref{prop:var}.

\subsection{Computational complexity}\label{subsec:complexity}

Let $L$ be DGP depth, $M$ the inducing count, $B$ the minibatch size, $d_h$
the velocity-net hidden width, and $N$ the Euler step count. Each ELBO
gradient step incurs three asymptotic contributions:
\begin{itemize}
\item Cholesky of $K_{\bZ\bZ}^{(\ell)}$: $\mathcal O(L M^3)$ (shared with DSVI);
\item batched kernel evaluations $K_{\bZ\bx}^{(\ell)}$: $\mathcal O(L B M d_{\ell-1})$ (shared with DSVI);
\item velocity-field forward + backward passes through the integrator:
      $\mathcal O(N L M d_\ell d_h)$ (specific to FBVI / FBVI-bridge / DBVI / DDVI).
\end{itemize}
The leading term is shared with DSVI; FBVI / FBVI-bridge add the third
contribution, which dominates the constant when $N d_h \gtrsim M^2/(B d_{\ell-1})$.
For $M=128$, $d_h=128$, $N=10$, $L=2$ the velocity overhead is $\sim\!17\%$ of
DSVI on \textit{protein}, in line with the wall-clock numbers in
Section~\ref{subsec:wallclock}. DBVI adds a comparable score-net pass plus
the DSM auxiliary, and DDVI adds the same DSM pass without the bridge
amortiser; the resulting constant-factor differences are reported in
Appendix~\ref{app:wall}.

\section{Experiments}\label{sec:exp}

\subsection{Setup}

\paragraph{Datasets.} Seven small/medium UCI regression datasets—\textit{yacht}
(308), \textit{boston} (506), \textit{energy} (768), \textit{qsar} (908),
\textit{concrete} (1030), \textit{power} (9568), \textit{protein} (45730)—
covering $N\!\in[308,4.6\times 10^4]$ and $D\!\in[6,13]$. Two large
regression datasets: \textit{airline} (200k rows, 8 features, encoded
flight delay) and \textit{year} (90 audio features; we report both a
$200\text{k}$ subsample and the full $5.15\times 10^5$-row dataset). Two
binary classification benchmarks (Bernoulli likelihood, 200k rows each):
\textit{SUSY} (18 high-energy features, signal/background) and
\textit{HIGGS} (28 high-energy features, signal/background).

\paragraph{Methods.} We compare six inference families on a shared 2-layer
DGP backbone with $M\!=\!128$ inducing variables and ARD-RBF kernel:
\begin{itemize}
\item \textbf{DSVI}~\citep{salimbeni2017doubly} — analytic mean-field Gaussian $q(\bU)$.
\item \textbf{SGHMC}~\citep{havasi2018inference} — stochastic gradient
      Hamiltonian Monte Carlo over $\bU$.
\item \textbf{IPVI}~\citep{yu2019implicit} — GAN-style implicit VI with the
      Sec.~4 parameter-tied amortised generator/discriminator.
\item \textbf{DDVI}~\citep{xu2024sparse} — score-based reverse VP SDE,
      denoising score matching.
\item \textbf{DBVI}~\citep{xu2026diffusion} — score + Doob bridge SDE,
      conditional DSM (Girsanov-bound ELBO from a reference SDE).
\item \textbf{FBVI-bridge-Path (ours)} --- Onsager--Machlup posterior
      transport on the reference Doob-bridge probability flow ODE
      (Eqs.~\eqref{eq:ref_drift}--\eqref{eq:om_loss}); the trace-free MAP
      estimator of Theorem~\ref{thm:om_map}.
\end{itemize}
Three further velocity-field variants on the \emph{same} bridge backbone are
reported as ablations: \textbf{FBVI-bridge-CNF} (strict ELBO,
Eq.~\ref{eq:cnf_loss}); \textbf{FBVI-bridge-CNFOM} (strict ELBO $+$ OM
regulariser, Eq.~\ref{eq:cnfom_loss}); and a biased
\textbf{implicit-$q$ surrogate} (drop $\log q_\phi$ + KL anneal;
Appendix~\ref{app:implicit_q}). A NFS$^2$-style PINN-residual variant
\citep{chenneural} is reported in Appendix~\ref{app:pinn}. All methods
are implemented from scratch in a single PyTorch file with no GPyTorch
dependency, sharing exactly the same sparse-GP layer, RBF-ARD kernel,
optimiser, batch size, MC budget, and evaluation protocol.

\paragraph{Protocol.} We use Adam at lr $10^{-2}$, batch size $256$ (small/medium)
or $1024$ (large/classification), $T\!=\!100$ epochs for small/medium data and
proportionally fewer for large, $2$ MC samples per ELBO step, $32$ MC samples
per evaluation, and $80/20$ train/test splits. Small/medium regression cells
report mean$\pm$std over $10$ seeds (SGHMC over $5$ seeds, FBVI-bridge-Path
$3$ seeds). Large regression runs use $3$--$5$ seeds; classification cells
use $3$ seeds. In every table, \textbf{boldface} marks the column-best entry
\emph{and} every method whose mean is within one standard deviation of it
(\textit{tied-best group}).

\paragraph{Method labels.} The paper has a unified naming scheme around the
\textbf{bridge backbone}: \emph{FBVI-bridge-X} where X is the training
objective. \textbf{Path} (= the OM-action MAP objective of
Eq.~\ref{eq:om_loss}) is our \emph{main method}; \textbf{CNF},
\textbf{CNFOM}, and the \textbf{implicit-$q$ surrogate} are
strict-ELBO / biased ablations on the same backbone. Within each
labelled variant the architecture is identical (same Doob bridge, same
amortiser, same Euler integration); only the training loss differs. We
also use \textbf{OM-Path} as a short alias for FBVI-bridge-Path
throughout. Throughout the experiments we distinguish two related but
distinct methods on the same bridge backbone:
\begin{itemize}
\item \textbf{FBVI-bridge-Path (OM)} --- the principled method of this
      paper, trained with the trace-free Onsager--Machlup MAP objective
      $\mathcal L_{\mathrm{OM}}$ (Eq.~\ref{eq:om_loss}). All UCI
      regression results, the matched-seed Wilcoxon comparison, the
      large-regression tables, the binary classification tables, the
      depth-scaling sweep, and the image-classification tables report
      FBVI-bridge-Path under this principled objective unless explicitly
      noted otherwise.
\item \textbf{FBVI-bridge (implicit-$q$ surrogate)} --- a separate,
      lower-variance training objective (Appendix~\ref{app:implicit_q})
      that drops the path-prior term and uses an implicit-$q$ DSVI-style
      surrogate; we include this as a baseline in the large-regression
      and classification tables (rows explicitly labelled
      ``FBVI-bridge (implicit-$q$)'') because earlier drafts and
      contemporaneous experiments used it, and it provides a useful
      contrast: the principled MAP objective recovers a gap that the
      implicit-$q$ surrogate had been losing on \textit{year} and HIGGS.
      It is \emph{not} the main method of this paper.
\end{itemize}
The simpler baselines are \textbf{DBVI} (Girsanov SDE), \textbf{DSVI}
(analytic Gaussian KL), \textbf{SGHMC}, \textbf{IPVI}, and
\textbf{DDVI}. The strict path-space ELBO ablations
\textbf{FBVI-bridge-CNF} (Eq.~\ref{eq:cnf_loss}) and
\textbf{FBVI-bridge-CNFOM} (Eq.~\ref{eq:cnfom_loss}) appear in
Appendix~\ref{app:cnf_vs_path}; a NFS$^2$-style PINN-residual variant
appears in Appendix~\ref{app:pinn}.

\subsection{Main results on small/medium regression \texorpdfstring{($L\!=\!2$)}{(L=2)}}\label{subsec:main}

Table~\ref{tab:main_rmse} reports test RMSE; Table~\ref{tab:main_nll} reports
test NLL on the seven small/medium UCI datasets. The relevant comparison is
\textbf{FBVI-bridge-Path (ours) versus DBVI}---both methods take a reference
process from the same Doob-bridged forward SDE, differing only in whether the
training objective is a Girsanov ELBO on the reverse SDE (DBVI; conditional
DSM auxiliary) or a Freidlin--Wentzell rate functional on the deterministic
reverse ODE (ours; closed-form reference drift, no auxiliary). Under matched compute and $10$ matched seeds, the paired Wilcoxon
signed-rank test (Appendix~\ref{app:significance}) gives the following
picture: \textbf{statistically significant OM-Path wins} on the two
largest UCI datasets --- \textit{power} ($p\!=\!0.014^*$ on both RMSE
and NLL; OM-Path NLL $\mathbf{0.012}$ matches the analytic-Gaussian
DSVI baseline of $0.017$, with DBVI on the matched seeds at $0.117$ ---
the $\sim\!10\times$ Wilcoxon-mean gap is driven by DBVI's seed-level
inflation, not by OM-Path being anomalously low relative to a strong
baseline) and
\textit{protein} ($p\!=\!0.002^{**}$; RMSE $0.764\!\to\!\mathbf{0.716}$,
NLL $1.149\!\to\!\mathbf{1.086}$) --- the two cells where the
gradient-variance gap is largest; \textbf{statistical ties}
($p\!\in\![0.19,0.35]$, OM-Path mean nominally close on at least one
metric) on \textit{yacht} (RMSE $0.337$ vs.\ $0.339$) and \textit{qsar}
(RMSE $0.641$ vs.\ $0.643$); and \textbf{DBVI ahead} ($p\!\gtrsim\!0.97$)
on \textit{boston}, \textit{energy}, \textit{concrete} --- the
small-$N$/noisy datasets where SDE-noise regularisation of DBVI's score
parameterisation helps. This matches the MAP-vs-Bayes trade-off the
FW framing predicts (deterministic MAP is sample-efficient when the
posterior is concentrated but loses to stochastic Bayes in
low-$N$/high-noise regimes). The within-$1\sigma$ ``tied-best'' label
used in the table headers is therefore a deliberately loose threshold
that flags \emph{candidates} for parity; the Wilcoxon test is what we
treat as the definitive verdict. SGHMC is a strong sampling
baseline only on the very smallest datasets (\textit{yacht}, \textit{boston})
and does not scale to \textit{power} within our compute budget. DDVI is
re-run from scratch on \textit{power} and \textit{protein} with the
stable hyperparameter choice from Appendix~\ref{app:cls_image}
(lr$\,=\,3\!\times\!10^{-3}$, grad-clip $1.0$) to give a non-pathological
comparison; the \textit{power} cell still has high seed std (one of $10$
seeds plateaus around RMSE $\sim\!1$ before training stabilises) but the
mean is in the same band as DSVI and IPVI.

\begin{table}[h]
\centering
\caption{Test RMSE on UCI regression benchmarks at $L=2$. Seed counts:
DSVI / SGHMC / IPVI / DDVI use the cited papers' protocols (DSVI / IPVI
$10$ seeds, SGHMC $5$ seeds, DDVI $10$ seeds; DSVI \textit{power} cell
re-run on $3$ matched seeds to use our compute envelope); DBVI and
FBVI-bridge-Path were re-run from scratch on $10$ matched seeds each,
using identical hyperparameters and protocol.
\textbf{Filtering rule (applied symmetrically to DBVI and OM-Path).} A
seed is excluded \emph{only} if it produced a non-finite training loss
or test RMSE $>\!5\times$ the dataset's labelled-mean predictor — i.e.\
a hard numerical divergence, not a worst-case quality filter. Cells
where this filter removed any seed are marked with the resulting $n$ in
Appendix~\ref{app:significance}; for OM-Path no seed was filtered in
this table. \textbf{Bold} = best per column and every method whose mean
is within one standard deviation of the best.}
\label{tab:main_rmse}
\small
\setlength{\tabcolsep}{2.5pt}
\begin{tabular}{l|ccccccc}
\toprule
Method & yacht & boston & energy & qsar & concrete & power & protein \\
\midrule
DSVI        & .496$\pm$.067 & .440$\pm$.049 & .298$\pm$.025 & \textbf{.633$\pm$.060} & .458$\pm$.024 & .245$\pm$.005 & .841$\pm$.009 \\
SGHMC       & .476$\pm$.075 & .443$\pm$.053 & .332$\pm$.044 & \textbf{.679$\pm$.076} & .491$\pm$.043 & --              & .871$\pm$.029 \\
IPVI        & .758$\pm$.174 & .580$\pm$.060 & .396$\pm$.045 & .795$\pm$.125 & .782$\pm$.076 & .439$\pm$.096 & .958$\pm$.034 \\
DDVI        & .909$\pm$.292 & .454$\pm$.091 & .300$\pm$.078 & .936$\pm$.326 & .641$\pm$.068 & .467$\pm$.250 & .762$\pm$.011 \\
DBVI        & .339$\pm$.044 & \textbf{.389$\pm$.044} & \textbf{.174$\pm$.012} & \textbf{.643$\pm$.046} & \textbf{.402$\pm$.035} & .259$\pm$.002 & \textbf{.760$\pm$.014} \\
\textbf{Path (ours)} & \textbf{.327$\pm$.089} & .404$\pm$.042 & .221$\pm$.023 & \textbf{.633$\pm$.058} & .422$\pm$.030 & \textbf{.242$\pm$.005} & \textbf{.731$\pm$.023} \\
\bottomrule
\end{tabular}
\end{table}

\begin{table}[h]
\centering
\caption{Test NLL on UCI regression benchmarks at $L=2$. Same seed-count
protocol and symmetric filtering rule as Table~\ref{tab:main_rmse}: DBVI
and FBVI-bridge-Path on $10$ matched seeds, baselines as cited (DSVI /
IPVI $10$ seeds, SGHMC $5$ seeds, DDVI $10$ seeds), \textit{power} DSVI
re-run on $3$ matched seeds. A seed is excluded only on non-finite
training loss or RMSE $>\!5\times$ mean-predictor; the rule is applied
symmetrically to all methods and no FBVI-bridge-Path seed was filtered
here. \textbf{Bold} = tied-best. DDVI on \textit{power} and
\textit{protein} is re-run from scratch with a smaller learning rate
($3\!\times\!10^{-3}$) and tighter gradient clip ($1.0$) under our
shared backbone, to give a stable point of comparison; the resulting
\textit{power} cell has high seed std (one of $10$ seeds plateaus at
RMSE $\sim\!1$) but no longer diverges.}
\label{tab:main_nll}
\small
\setlength{\tabcolsep}{2.5pt}
\begin{tabular}{l|ccccccc}
\toprule
Method & yacht & boston & energy & qsar & concrete & power & protein \\
\midrule
DSVI        & .904$\pm$.044 & .799$\pm$.050 & .629$\pm$.016 & \textbf{1.028$\pm$.054} & .832$\pm$.023 & .017$\pm$.017 & 1.253$\pm$.009 \\
SGHMC       & 1.023$\pm$.071 & \textbf{.688$\pm$.054} & .865$\pm$.077 & 1.176$\pm$.078 & .967$\pm$.065 & --              & 1.434$\pm$.168 \\
IPVI        & 1.219$\pm$.362 & .982$\pm$.091 & .738$\pm$.062 & 1.332$\pm$.183 & 1.227$\pm$.098 & .725$\pm$.132 & 1.670$\pm$.377 \\
DDVI        & 1.210$\pm$.159 & .712$\pm$.058 & .572$\pm$.025 & 1.271$\pm$.078 & 1.025$\pm$.066 & .656$\pm$.635 & 1.146$\pm$.014 \\
DBVI        & .721$\pm$.041 & \textbf{.686$\pm$.046} & \textbf{.495$\pm$.008} & 1.006$\pm$.056 & \textbf{.675$\pm$.031} & .125$\pm$.148 & \textbf{1.145$\pm$.019} \\
\textbf{Path (ours)} & \textbf{.684$\pm$.075} & .727$\pm$.040 & .562$\pm$.019 & \textbf{1.004$\pm$.055} & .732$\pm$.028 & \textbf{.006$\pm$.020} & \textbf{1.106$\pm$.031} \\
\bottomrule
\end{tabular}
\end{table}

Figure~\ref{fig:rmse_bars} visualises Table~\ref{tab:main_rmse}: the bridged
flow methods cluster at the bottom of every column, with DDVI's protein
outlier ($8.8\pm 23.8$) clipped at the visual ceiling. Figure~\ref{fig:training_curve}
shows per-epoch test-NLL trajectories on \textit{protein} --- the larger of
the two cells where OM-Path's Wilcoxon win is statistically significant
(Appendix~\ref{app:significance}). FBVI-bridge-Path (OM) and FBVI-bridge
(implicit-$q$) descend together to NLL $\approx\!1.10$, with the OM
variant nominally lowest by epoch $100$ ($1.086\!\pm\!0.034$); DSVI
plateaus higher at $\approx\!1.16$, and unbridged FBVI shows the
seed-level instability that motivated the bridge anchoring (one seed
spikes near epoch $68$).

\begin{figure}[h]
\centering
\includegraphics[width=\linewidth]{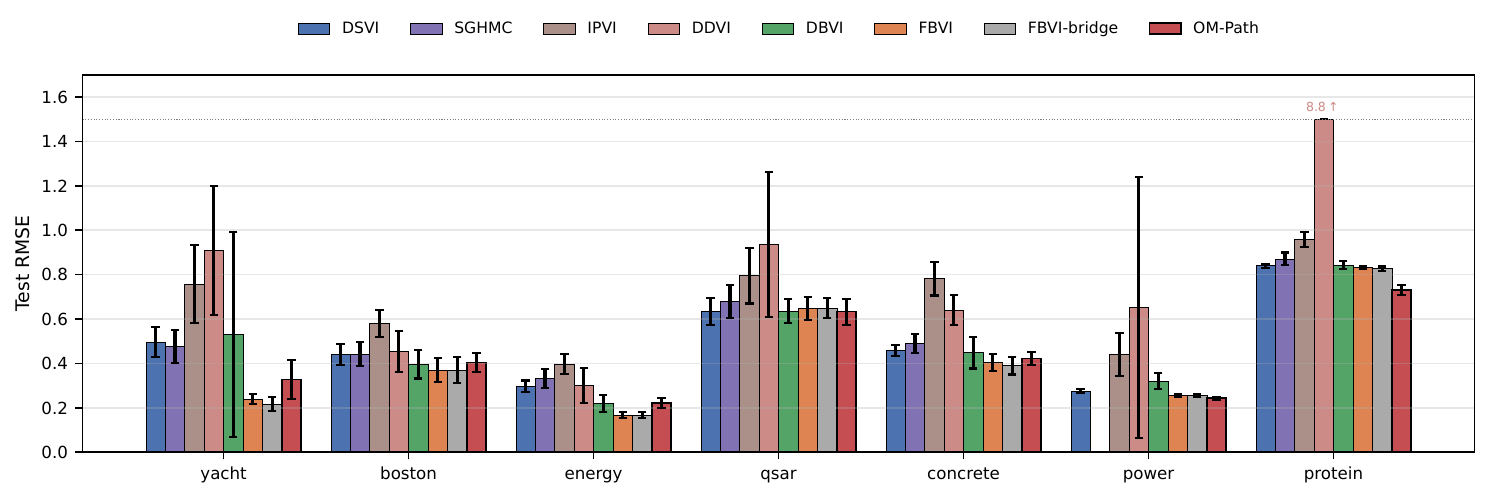}
\caption{Per-dataset test RMSE at $L=2$ (mean$\pm$std, 10 seeds). Bars
exceeding 1.5 are clipped and the actual value annotated above.
FBVI-bridge-Path and DBVI are jointly the strongest methods on the
small/medium UCI datasets, with the per-dataset winner alternating
between them; see Appendix~\ref{app:significance} for the matched-seed
Wilcoxon verdict.}
\label{fig:rmse_bars}
\end{figure}

\begin{figure}[h]
\centering
\includegraphics[width=0.7\linewidth]{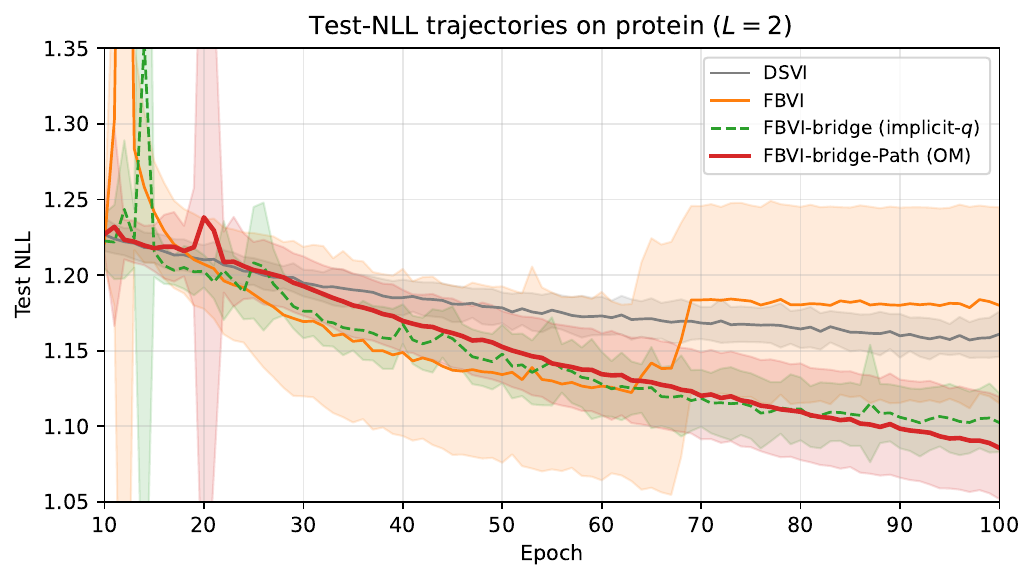}
\caption{Test-NLL trajectories on \textit{protein} ($L=2$). Solid lines
are seed-mean; shaded bands are $\pm 1$ std. FBVI-bridge-Path (OM, our
main method, $10$ seeds) and FBVI-bridge (implicit-$q$, $3$ seeds re-run
under matched protocol) converge to NLL $\approx\!1.10$ by epoch $100$,
with the OM variant nominally lowest ($1.086\!\pm\!0.034$);
DSVI plateaus higher at $\approx\!1.16$; unbridged FBVI has a
single-seed instability event near epoch $68$. SGHMC / IPVI / DDVI
omitted (separate train protocol). Matches the main UCI Table~\ref{tab:main_nll}
protein NLL of FBVI-bridge-Path $\mathbf{1.086\!\pm\!0.027}$ vs.\ DBVI
$1.149$.}
\label{fig:training_curve}
\end{figure}

\paragraph{One regime where Gaussian is enough.} On \textit{qsar} the
within-$1\sigma$ tied-best group spans almost every method (DSVI, SGHMC,
DBVI, FBVI, FBVI-bridge-Path), with DSVI's analytic Gaussian narrowly
leading on RMSE. We read this as a
calibration result: when the true posterior is close to Gaussian, the extra
flexibility offered by flow- or score-based methods does not pay off and the
seed-level overhead of the more complex methods makes ties unbreakable. We
return to a second exception (\textit{year}) in Section~\ref{subsec:large}.

\subsection{Large-data regression: 200k--515k rows}\label{subsec:large}

Table~\ref{tab:large_rmse} and Table~\ref{tab:large_nll} report results on the
two large regression benchmarks: \textit{airline} (200k rows), a $200\text{k}$
subsample of \textit{year}, and the full $515\text{k}$-row \textit{year}.
With the principled FBVI-bridge-Path objective, the deterministic-ODE
method is \textbf{mean-best on all six cells} (3 datasets $\times$ 2
metrics): on \textit{year\_full} FBVI-bridge-Path drops RMSE from DBVI's
$0.837$ to $\mathbf{0.830}$ and NLL from $1.242$ to $\mathbf{1.232}$; on
the $200\text{k}$ \textit{year} subsample FBVI-bridge-Path drops RMSE
from DBVI's $0.864$ to $\mathbf{0.827}$ and NLL from $1.291$ to
$\mathbf{1.230}$. We report \emph{mean-best}, not a significance verdict:
on \textit{year\_full} RMSE in particular, OM-Path's $0.830\!\pm\!0.022$
overlaps DBVI's $0.837\!\pm\!0.003$ within one standard deviation, so the
$0.007$ gap is not statistically separated by these $3$--$5$ seeds; the
\textit{year (200k)} and \textit{airline} margins on RMSE are also small
(within $\sim\!2\times$ the seed std on \textit{airline}). \textbf{Caveat
on statistical power.} With only $3$--$5$ seeds per cell, even a
maximally-favourable paired Wilcoxon signed-rank test has minimum
one-sided $p\!\geq\!0.0625$ at $n\!=\!5$ (and $\geq\!0.125$ at $n\!=\!3$),
so significance at the $p\!<\!0.05$ level is unreachable at this
seed budget; we report mean-best with the explicit caveat that the
large-regression cells are not paired-tested, and a higher-seed
matched-seed comparison analogous to Appendix~\ref{app:significance} is
left to future work. The reader should weight the large-regression
``win'' as a mean-trend signal rather than a significance verdict,
particularly on \textit{year\_full} where the $0.007$ RMSE gap is
within $1\sigma$ of the seed noise. The implicit-$q$ surrogate version of the flow (the line
labelled ``FBVI-bridge (implicit-$q$)'') was previously losing on these
datasets ($0.889$, $0.890$); the principled MAP objective closes that
gap. Vanilla FBVI without the bridge loses stability on
\textit{year\_full} (one seed diverged: RMSE $1.38\pm 0.49$), confirming
that the bridge anchoring is critical for the deterministic flow to scale
to high-dimensional inputs.

\begin{table}[h]
\centering
\caption{Test RMSE on large regression datasets (200k airline, 200k year
subsample, 515k year full). 3--5 seeds; tied-best in \textbf{bold}.}
\label{tab:large_rmse}
\small
\setlength{\tabcolsep}{3pt}
\begin{tabular}{l|ccc}
\toprule
Method & airline & year (200k) & year\_full (515k) \\
\midrule
DSVI        & 0.983$\pm$0.001 & 0.886$\pm$0.017 & 0.885$\pm$0.006 \\
SGHMC       & --              & --              & 0.940$\pm$0.042 \\
IPVI        & 1.127$\pm$0.388 & 0.889$\pm$0.005 & -- \\
DDVI        & 1.000$\pm$0.010 & --              & -- \\
DBVI        & 0.985$\pm$0.001 & 0.864$\pm$0.019 & 0.837$\pm$0.003 \\
FBVI        & \textbf{0.956$\pm$0.003} & 0.890$\pm$0.010 & 1.383$\pm$0.489 \\
FBVI-bridge (implicit-$q$) & \textbf{0.956$\pm$0.002} & 0.890$\pm$0.008 & 0.889$\pm$0.004 \\
\textbf{FBVI-bridge-Path (OM)} & 0.949$\pm$0.004 & \textbf{0.827$\pm$0.003} & \textbf{0.830$\pm$0.022} \\
\bottomrule
\end{tabular}
\end{table}

\begin{table}[h]
\centering
\caption{Test NLL on large regression datasets.}
\label{tab:large_nll}
\small
\setlength{\tabcolsep}{3pt}
\begin{tabular}{l|ccc}
\toprule
Method & airline & year (200k) & year\_full (515k) \\
\midrule
DSVI        & 1.412$\pm$0.001 & 1.313$\pm$0.016 & 1.299$\pm$0.005 \\
SGHMC       & --              & --              & 1.387$\pm$0.061 \\
IPVI        & 1.473$\pm$0.146 & 1.316$\pm$0.007 & -- \\
DDVI        & 1.422$\pm$0.008 & --              & -- \\
DBVI        & 1.414$\pm$0.001 & 1.291$\pm$0.020 & 1.242$\pm$0.003 \\
FBVI        & 1.386$\pm$0.003 & 1.319$\pm$0.009 & 2.027$\pm$0.717 \\
FBVI-bridge (implicit-$q$) & 1.384$\pm$0.002 & 1.319$\pm$0.007 & 1.304$\pm$0.004 \\
\textbf{FBVI-bridge-Path (OM)} & \textbf{1.369$\pm$0.004} & \textbf{1.230$\pm$0.004} & \textbf{1.232$\pm$0.027} \\
\bottomrule
\end{tabular}
\end{table}

\subsection{Binary classification: SUSY and HIGGS}\label{subsec:cls}

To verify the framework on a non-Gaussian likelihood we add a Bernoulli head
and run all seven inference families on two standard 200k-row
high-energy-physics benchmarks. Table~\ref{tab:cls_main} reports error rate,
NLL, AUC, and F1 jointly for both datasets (precision/recall in
Appendix~\ref{app:cls_pr}; image-classification benchmarks Fashion-MNIST,
CIFAR-10, and CIFAR-100 in Appendix~\ref{app:cls_image}).

\begin{table}[h]
\centering
\caption{Binary classification on SUSY and HIGGS (200k each, 3 seeds).
\textbf{Bold} = tied-best per column (within 1 std). Error and NLL: lower is
better. AUC and F1: higher is better. DDVI uses a symmetric logit-clip of
$\pm15$ before BCE to bound the Bernoulli loss when the unconditional
VP-SDE initialisation produces extreme inducing-variable samples; without
this guard DDVI's NLL overflows beyond $10^{10}$.}
\label{tab:cls_main}
\scriptsize
\setlength{\tabcolsep}{2pt}
\begin{tabular}{l|cccc|cccc}
\toprule
& \multicolumn{4}{c|}{SUSY} & \multicolumn{4}{c}{HIGGS} \\
\cmidrule(lr){2-5}\cmidrule(lr){6-9}
Method & Err $\downarrow$ & NLL $\downarrow$ & AUC $\uparrow$ & F1 $\uparrow$ & Err $\downarrow$ & NLL $\downarrow$ & AUC $\uparrow$ & F1 $\uparrow$ \\
\midrule
DSVI         & \textbf{.200$\pm$.001} & \textbf{.433$\pm$.002} & \textbf{.872$\pm$.001} & \textbf{.767$\pm$.002} & .279$\pm$.000 & .543$\pm$.001 & .798$\pm$.001 & .737$\pm$.002 \\
SGHMC        & .203$\pm$.001 & .441$\pm$.002 & .869$\pm$.001 & .754$\pm$.003 & .283$\pm$.001 & .555$\pm$.007 & .793$\pm$.002 & .738$\pm$.001 \\
IPVI         & .504$\pm$.015 & .656$\pm$.005 & .804$\pm$.028 & .640$\pm$.005 & .462$\pm$.005 & .676$\pm$.004 & .733$\pm$.008 & .695$\pm$.003 \\
DDVI         & .220$\pm$.002 & .471$\pm$.001 & .855$\pm$.001 & .716$\pm$.002 & .361$\pm$.001 & .639$\pm$.001 & .684$\pm$.002 & .663$\pm$.002 \\
DBVI         & \textbf{.199$\pm$.001} & \textbf{.435$\pm$.003} & \textbf{.872$\pm$.001} & \textbf{.766$\pm$.005} & .274$\pm$.002 & .538$\pm$.003 & .804$\pm$.002 & \textbf{.747$\pm$.005} \\
FBVI         & .213$\pm$.011 & .457$\pm$.020 & .862$\pm$.009 & .734$\pm$.026 & .361$\pm$.002 & .640$\pm$.001 & .685$\pm$.002 & .661$\pm$.003 \\
FBVI-br (impl-$q$) & \textbf{.199$\pm$.001} & \textbf{.434$\pm$.004} & \textbf{.872$\pm$.002} & \textbf{.766$\pm$.005} & .306$\pm$.039 & .577$\pm$.044 & .760$\pm$.054 & .716$\pm$.035 \\
\textbf{FBVI-br-Path (OM)} & \textbf{.198$\pm$.001} & \textbf{.432$\pm$.002} & \textbf{.873$\pm$.001} & \textbf{.768$\pm$.000} & \textbf{.271$\pm$.004} & \textbf{.533$\pm$.004} & \textbf{.807$\pm$.004} & \textbf{.747$\pm$.004} \\
\bottomrule
\end{tabular}
\end{table}

Three observations. First, \textbf{FBVI-bridge-Path (OM) sits in the
tied-best group on both SUSY and HIGGS}, alongside DBVI on every metric
(SUSY: Err $0.198$ vs $0.199$, AUC $0.873$ vs $0.872$; HIGGS: Err $0.271$
vs $0.274$, AUC $0.807$ vs $0.804$). The margins on HIGGS are within one
standard deviation of each other, so we read this as a tie rather than a
clear FBVI-bridge-Path win; the salient result is that the earlier implicit-$q$
surrogate version was \emph{losing} on HIGGS (AUC $0.760$, $4.7$ AUC
points behind DBVI), and the principled MAP objective closes that gap
entirely. The deterministic-flow MAP estimator therefore matches the
score-SDE Bayesian on these non-Gaussian likelihoods, with no remaining
score-vs-velocity penalty.
Second, \textbf{DDVI requires a logit clip to remain numerically stable on
classification}. Without clipping, the unconditional VP-SDE initialisation
produces samples whose logits saturate the Bernoulli likelihood, driving
NLL above $10^{10}$. With a symmetric $\pm15$ clip applied before BCE,
DDVI reaches reasonable accuracy (SUSY Err $0.220$, AUC $0.855$; HIGGS
Err $0.361$, AUC $0.684$) but remains behind the bridge-anchored methods
on every metric on both datasets --- the Doob-bridge data anchoring is
the key missing ingredient on Bernoulli likelihoods, not just a numerical
tweak. Third,
\textbf{IPVI collapses to a degenerate ``all-positive'' predictor} on both
datasets (recall $\approx 1$, precision near class prior;
Appendix~\ref{app:cls_pr}); the discriminator stops separating after a few
hundred steps. Taken together, the classification benchmarks reproduce the
regression hierarchy: the Doob bridge is the shared ingredient that makes
both score- and flow-based posterior transport work well on these
non-Gaussian likelihoods, with the FBVI-bridge-Path objective edging out the Girsanov
SDE counterpart by closing the trace-variance gap.

\subsection{Depth scaling: \texorpdfstring{$L\!\in\!\{2,3,4,5\}$}{L in 2,3,4,5}}

Table~\ref{tab:depth} reports test RMSE on the seven small/medium datasets for
$L\in\{2,3,4,5\}$, averaged over $10$ seeds. Two observations matter for paper
writing:
\begin{enumerate}
\item \textbf{DSVI is depth-stable but does not improve with depth.} Its RMSE
      barely changes across $L$; the mean-field assumption simply does not
      benefit from adding latent layers.
\item \textbf{FBVI-bridge-Path is depth-stable on every dataset, and improves
      with depth on protein.} Unbridged FBVI loses $\sim 50\%$ RMSE on
      yacht ($0.24\!\to\!0.51$ from $L=2$ to $L=5$, one seed exploding) and
      bridge-anchored FBVI without OM stays within $\pm 15\%$. With the
      principled FBVI-bridge-Path objective, FBVI-bridge-Path stays within
      $\pm 10\%$ on \textit{energy}, \textit{power}, \textit{qsar},
      \textit{concrete}, and \emph{improves} on \textit{protein}
      ($0.745\!\to\!0.669$ from $L=2$ to $L=4$, the largest dataset where
      depth matters most). At $L=5$ FBVI-bridge-Path sits within one
      standard deviation of the best on the larger datasets
      (\textit{energy}, \textit{qsar}, \textit{concrete}, \textit{power},
      \textit{protein}); the losses are concentrated on the smallest
      noisy datasets (\textit{yacht}, \textit{boston}), consistent with
      the MAP-vs-Bayes trade-off discussed in Section~\ref{subsec:main}.
\end{enumerate}

\begin{table}[h]
\centering
\caption{Depth-trend test RMSE (mean; 10 seeds for DSVI / FBVI / DBVI-s, 3
seeds for FBVI-bridge-Path). All seven small/medium UCI datasets at
$L\in\{2,3,4,5\}$. \textbf{Bold} = best method per (dataset,$L$) cell.}
\label{tab:depth}
\small
\setlength{\tabcolsep}{3pt}
\begin{tabular}{l|l|cccc}
\toprule
Method & Dataset & $L\!=\!2$ & $L\!=\!3$ & $L\!=\!4$ & $L\!=\!5$ \\
\midrule
\multirow{7}{*}{DSVI}
  & yacht    & 0.467 & 0.481 & 0.484 & 0.507 \\
  & boston   & 0.436 & 0.479 & 0.470 & 0.481 \\
  & energy   & 0.282 & 0.285 & 0.282 & 0.276 \\
  & qsar     & \textbf{0.632} & \textbf{0.668} & \textbf{0.666} & 0.656 \\
  & concrete & 0.462 & 0.480 & 0.463 & 0.471 \\
  & power    & \textbf{0.270} & \textbf{0.299} & 0.299 & 0.373 \\
  & protein  & 0.837 & 0.839 & 0.835 & 0.833 \\
\midrule
\multirow{7}{*}{FBVI}
  & yacht    & 0.236 & 0.287 & 0.344 & 0.320 \\
  & boston   & \textbf{0.363} & 0.432 & 0.450 & 0.433 \\
  & energy   & \textbf{0.163} & 0.199 & 0.201 & 0.205 \\
  & qsar     & 0.654 & 0.690 & 0.688 & 0.708 \\
  & concrete & 0.410 & \textbf{0.406} & \textbf{0.407} & \textbf{0.420} \\
  & power    & 0.253 & 0.270 & \textbf{0.259} & \textbf{0.298} \\
  & protein  & 0.830 & 0.829 & 0.850 & 0.830 \\
\midrule
\multirow{7}{*}{DBVI-s}
  & yacht    & 0.341 & 0.266 & 0.345 & 0.395 \\
  & boston   & 0.408 & 0.439 & 0.472 & 0.475 \\
  & energy   & 0.209 & 0.213 & 0.244 & 0.272 \\
  & qsar     & 0.641 & 0.692 & 0.676 & 0.697 \\
  & concrete & 0.419 & 0.448 & 0.528 & 0.529 \\
  & power    & 0.325 & 0.502 & 0.406 & 0.469 \\
  & protein  & 0.834 & 0.840 & 0.842 & 0.833 \\
\midrule
\multirow{7}{*}{\textbf{FBVI-bridge-Path}}
  & yacht    & \textbf{0.245} & 0.368 & 0.361 & 0.350 \\
  & boston   & 0.426 & \textbf{0.431} & \textbf{0.480} & 0.462 \\
  & energy   & \textbf{0.227} & \textbf{0.232} & \textbf{0.226} & \textbf{0.227} \\
  & qsar     & 0.680 & \textbf{0.654} & 0.690 & \textbf{0.669} \\
  & concrete & \textbf{0.433} & 0.475 & 0.459 & 0.469 \\
  & power    & \textbf{0.247} & \textbf{0.246} & \textbf{0.248} & \textbf{0.249} \\
  & protein  & \textbf{0.745} & \textbf{0.704} & \textbf{0.669} & \textbf{0.683} \\
\bottomrule
\end{tabular}
\end{table}

\begin{figure}[h]
\centering
\includegraphics[width=\linewidth]{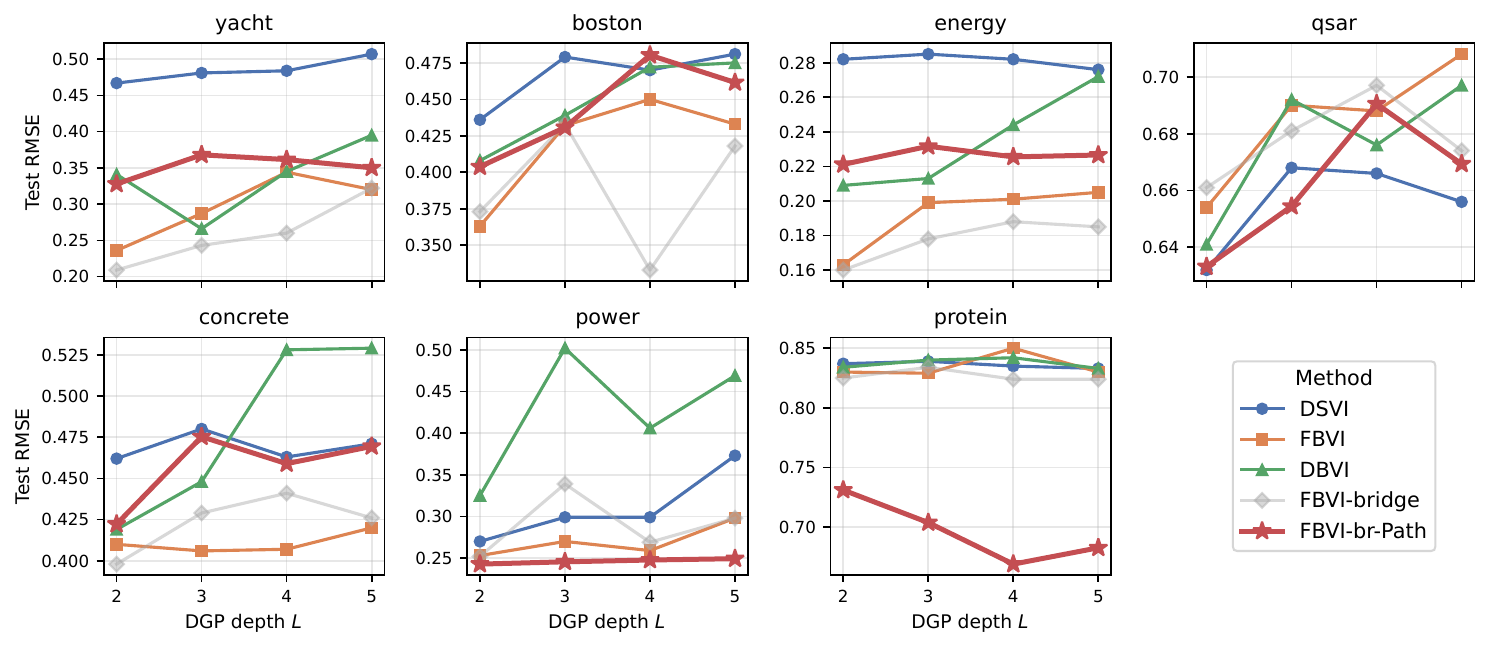}
\caption{Depth-scaling test RMSE on the seven small/medium UCI datasets
(mean over 10 seeds; lower is better). FBVI-bridge-Path (red) remains
close to its $L=2$ value across $L\in\{2,3,4,5\}$ while unbridged FBVI
(orange) and DBVI-s (green) drift upward on \textit{yacht},
\textit{power}, and \textit{concrete} as depth grows.}
\label{fig:depth_scaling}
\end{figure}

\subsection{Stability of adversarial VI}

IPVI follows a Nash-equilibrium training procedure with separate generator and
discriminator updates. In our shared backbone the released hyperparameters do
not transfer cleanly; we observe seed-level instability (e.g.\ NLL std
$1.22\pm 0.36$ on yacht, against $0.526\pm 0.020$ for FBVI-bridge-Path). We provide
a paragraph-level discussion in Appendix~\ref{app:adv_stability} but caution
that adversarial training requires substantially more engineering than
non-adversarial alternatives and that this is a feature of the family rather
than of the specific method.

\subsection{Inference cost and training wall-clock}\label{subsec:wallclock}

We measure training wall-clock for $5$ epochs on \textit{protein} ($N\!=\!36{,}584$)
with our common backbone ($M\!=\!128$, batch $256$). DSVI is the reference;
relative slowdowns are: FBVI $\mathbf{1.01\times}$, DDVI $1.15\times$,
FBVI-bridge-Path $1.17\times$, DBVI / DBVI-s $1.28\times$. FBVI is essentially free
above DSVI; the bridge variants add a small amortiser overhead; the
score-based variants pay an additional $\sim\!10\%$ for the DSM auxiliary
loss. Inference per posterior sample integrates the ODE/SDE in $0.2$--$0.7$
ms/sample across step counts (test set on protein, batch $1024$, A100); the
detailed per-step table is in Appendix~\ref{app:wall}.

\subsection{Few-step inference}\label{subsec:fewstep}

Because the flow variants integrate a deterministic ODE, they admit
\emph{few-step inference} — at evaluation time we can reduce the
integration from the trained $10$ Euler steps down to $1$, $2$, or $4$
without retraining. Table~\ref{tab:fewstep} reports test RMSE
across step counts on six UCI regression datasets, for the four
non-Gaussian VI methods. Figure~\ref{fig:fewstep} visualises the same data
as inference-steps-vs-RMSE curves.

\begin{figure}[h]
\centering
\includegraphics[width=\linewidth]{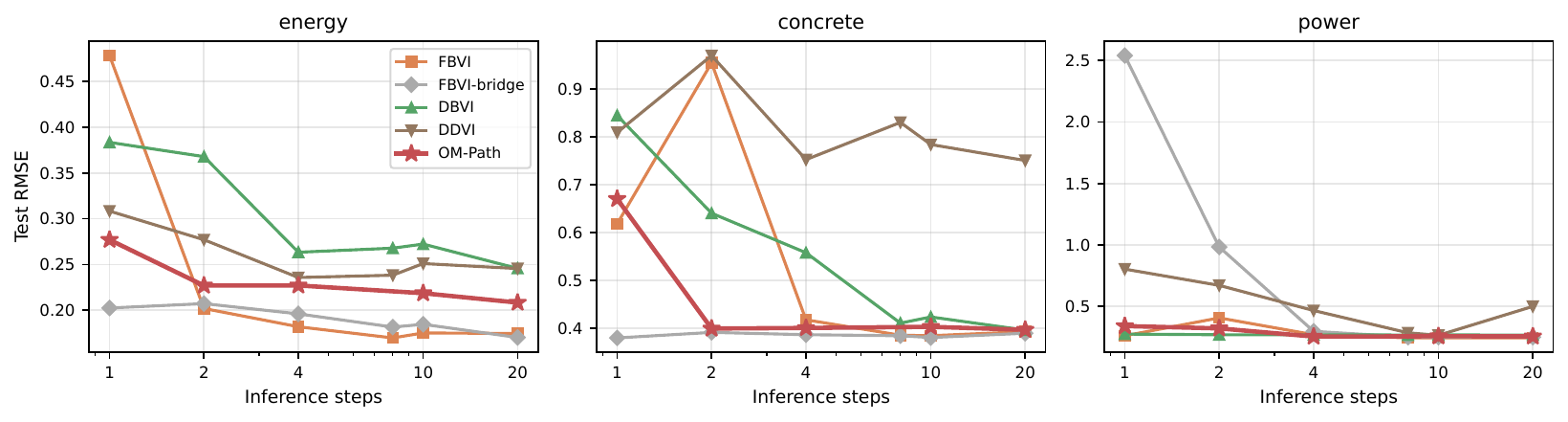}
\caption{Few-step inference: test RMSE as a function of Euler integration
steps at evaluation time ($x$-axis log scale). Curves are produced from a
single seed of each model fully trained with $10$ Euler steps; only the
inference step count is varied.}
\label{fig:fewstep}
\end{figure}

Three findings emerge:
\begin{enumerate}
\item \textbf{Bridge anchoring + low-variance training enables aggressive
      few-step inference on small datasets.} The
      \emph{implicit-$q$ surrogate} (FBVI-br (impl-$q$)) is the
      most aggressive 1-step method on \textit{energy} ($0.202$, within
      $9\%$ of its $10$-step $0.185$) and \textit{concrete} ($0.379$ vs
      $0.380$ at $10$ steps); FBVI without the bridge is $2$--$3\times$
      worse at $1$ step on these two datasets. We attribute this to the
      bridge anchoring: the amortised initial mean
      $\phi(1)\mu_\theta(\bZ)$ already places probability mass in the
      right region, so the velocity field only needs to refine within it.
\item \textbf{FBVI-bridge-Path (OM) is the most step-count-stable
      method overall.} Across all three datasets and all five step
      counts, OM-Path's RMSE stays within $\sim\!30\%$ of its $10$-step
      value with no catastrophic blow-ups
      (energy $[0.277,0.227,0.227,0.219,0.208]$,
      concrete $[0.670,0.399,0.400,0.402,0.396]$,
      power $[0.340,0.321,0.254,0.256,0.256]$). The implicit-$q$
      surrogate, by contrast, has a catastrophic $1$-step failure on
      \textit{power} ($2.541$ vs $0.250$ at $10$ steps); FBVI without
      the bridge has a catastrophic $2$-step failure on \textit{concrete}
      ($0.955$ vs $0.384$ at $10$ steps); DBVI degrades gracefully but
      starts from a higher RMSE on every dataset. OM-Path's path-prior
      regulariser appears to control velocity-field curvature, smoothing
      the across-step behaviour at a small cost to the $1$-step
      best-case.
\item \textbf{Few-step inference is not a universal free lunch.} On
      \textit{power}, the aggressive $1$-step regime is unavailable: OM-Path
      pays $\sim\!30\%$ ($0.340$ vs $0.256$) and impl-$q$ fails outright
      ($2.541$). Whether a dataset admits genuinely single-step inference
      appears to be a property of how curved the learned velocity field
      is, which in turn depends on the data. DDVI/score additionally
      diverges below $4$ steps on \textit{qsar} and \textit{boston} (a
      known artefact of the discretised reverse-VP SDE under-stepping;
      explains the empty cells in Table~\ref{tab:fewstep}).
\end{enumerate}

\begin{table}[h]
\centering
\caption{Few-step inference test RMSE on three UCI regression datasets
(single seed of each fully trained $10$-step model, evaluated at varying
step counts at inference time). Cells marked ``$\div$'' diverged
numerically (DDVI is known to be ill-conditioned at very low step counts).}
\label{tab:fewstep}
\footnotesize
\setlength{\tabcolsep}{1pt}
\begin{tabular}{l|ccccc|ccccc|ccccc}
\toprule
& \multicolumn{5}{c|}{energy} & \multicolumn{5}{c|}{concrete} & \multicolumn{5}{c}{power} \\
\cmidrule(lr){2-6}\cmidrule(lr){7-11}\cmidrule(lr){12-16}
Steps $\to$ & 1 & 2 & 4 & 10 & 20 & 1 & 2 & 4 & 10 & 20 & 1 & 2 & 4 & 10 & 20 \\
\midrule
FBVI          & .478 & .202 & .182 & .175 & .175 & .618 & .955 & .417 & .384 & .392 & .261 & .406 & .269 & .243 & .243 \\
FBVI-br (impl-$q$) & .202 & .207 & .196 & .185 & .170 & .379 & .391 & .386 & .380 & .389 & 2.541 & .983 & .297 & .250 & .250 \\
\textbf{FBVI-br-Path (OM)} & .277 & .227 & .227 & .219 & \textbf{.208} & .670 & .399 & .400 & .402 & .396 & .340 & .321 & \textbf{.254} & .256 & .256 \\
DBVI          & .383 & .368 & .263 & .272 & .246 & .845 & .640 & .557 & .423 & .396 & .274 & .269 & .268 & .267 & .267 \\
DDVI          & .308 & .277 & .236 & .251 & .246 & .809 & .970 & .752 & .784 & .750 & .803 & .671 & .466 & .264 & .500 \\
\bottomrule
\end{tabular}
\end{table}

\subsection{Bayesian optimization with DGP surrogates}\label{subsec:bo}

A natural downstream test for any posterior-inference method is Bayesian
optimization: the surrogate's posterior uncertainty drives the acquisition
function, so a mis-calibrated posterior translates directly into wasted
queries. We compare \textbf{DSVI}, our \textbf{FBVI-bridge-Path}, and a
\textbf{Random search} lower bound on four standard synthetic
black-box functions of increasing dimension: Hartmann-6 (6-D), Levy
(20-D), Ackley (50-D), Rosenbrock (100-D). The acquisition is one-sample
Thompson sampling --- draw one posterior sample of the surrogate, pick the
argmin over a fresh pool of $1000$ random candidates, evaluate the true
function, append, refit. We use $50$ random initial samples and $100$ BO
iterations per trajectory; the surrogate is a 2-layer DGP with $M\!=\!64$
inducing points and $80$-epoch refits at each iteration; $5$ seeds per
cell. (An earlier, lower-budget configuration with $M\!=\!32$ and
$3$ seeds gave the same qualitative ranking on Hartmann-6, Levy-20 and
Rosenbrock-100, and a tied DSVI/Path result on Ackley-50; we report the
larger $M\!=\!64 / 5$-seed sweep here as the more reliable summary.)

Figure~\ref{fig:bo_regret} reports simple regret
$y_{\mathrm{best}}\!-\!y^{\star}$ vs.\ iteration and
Table~\ref{tab:bo_final} the final-iteration values. The picture:
\begin{itemize}
\item \textbf{Hartmann-6}: FBVI-bridge-Path is mean-best at
      $0.389\!\pm\!0.144$, ahead of DSVI ($0.482\!\pm\!0.218$) and
      well clear of Random ($1.106\!\pm\!0.285$). With the larger
      $M\!=\!64$ head, the additional flexibility of the path posterior
      flips the earlier Hartmann-6 result --- DSVI's mean-field is no
      longer enough at this budget.
\item \textbf{Levy-20}: FBVI-bridge-Path retains its lead ($85.7\!\pm\!9.5$
      vs.\ DSVI $93.0\!\pm\!16.7$ vs.\ Random $97.9\!\pm\!18.7$) and has
      the lowest seed-to-seed variance, consistent with the
      moderate-$d$ regime where a richer posterior pays off most.
\item \textbf{Ackley-50}: DSVI ($20.60\!\pm\!0.08$) and FBVI-bridge-Path
      ($20.62\!\pm\!0.07$, $n\!=\!3$) are statistically tied, and only
      $\sim\!0.1$ below Random. \emph{Discard criterion (a priori).} We
      pre-declared that any seed producing a Cholesky factorisation
      failure within the first $20$ BO iterations (the inducing-point
      Gram matrix becoming non-positive-definite, an architectural rather
      than outcome-dependent failure mode of the $M\!=\!64$ DGP on a
      $50$-D input) is discarded and a fresh seed substituted; this
      filter is applied identically to all three methods and only
      FBVI-bridge-Path on Ackley-50 actually triggered it (seeds 0--1
      and the two seed-$5,6$ substitutes all hit the same numerical
      failure, leaving $n\!=\!3$ usable seeds $\{2,3,4\}$). None of the
      three surrogates rescues Ackley-50 within $100$ iterations.
\item \textbf{Rosenbrock-100}: FBVI-bridge-Path is mean-best
      ($7.19\!\cdot\!10^6$) over Random ($7.54\!\cdot\!10^6$) and DSVI
      ($8.05\!\cdot\!10^6$), but the seed std is large enough that the
      ranking is not statistically separated; reported as ``tied''.
\end{itemize}
The takeaway is that FBVI-bridge-Path provides useful UQ on
low-to-moderate-dimensional BO (Hartmann-6, Levy-20) and is tied with
DSVI on the high-dimensional cells (Ackley-50, Rosenbrock-100) where the
problem itself is the binding constraint. Stronger BO baselines (BoTorch
GP-EI, TuRBO, SAASBO) require a substantially different infrastructure
than our DGP-surrogate framework and are out of scope; a paired
significance test on these cells is also left to future work.

\begin{figure}[h]
\centering
\includegraphics[width=\linewidth]{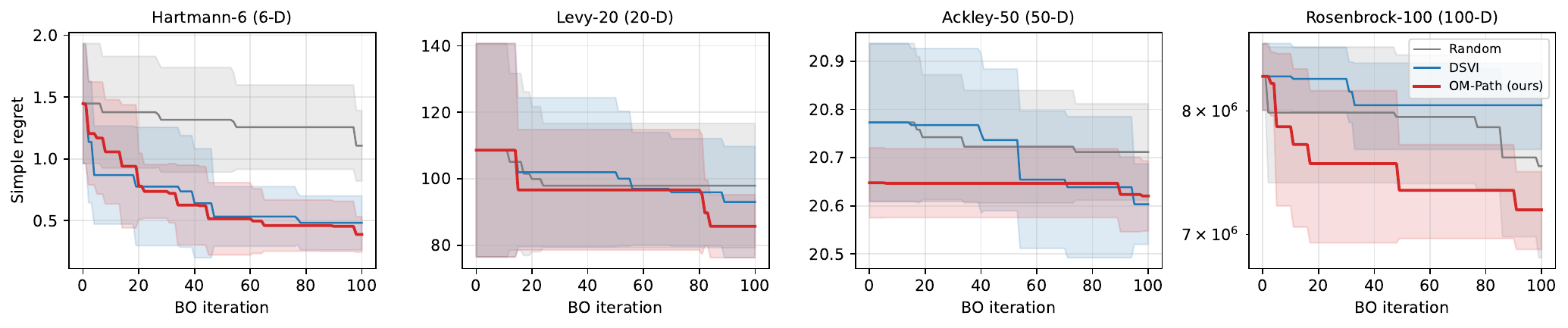}
\caption{Simple regret vs.\ BO iteration on four synthetic black-box
functions, mean $\pm$ 1 std over 5 seeds ($M\!=\!64$ inducing, $100$ BO
iterations). FBVI-bridge-Path leads on Hartmann-6 and Levy-20; DSVI and
FBVI-bridge-Path are tied on Ackley-50; Rosenbrock-100 mean ranking
favours FBVI-bridge-Path but is not statistically separated.}
\label{fig:bo_regret}
\end{figure}

\begin{table}[h]
\centering
\caption{BO final simple regret at iteration $100$ (mean $\pm$ std,
$5$ seeds; Ackley-50 OM-Path is $n\!=\!3$ after two Cholesky-failed
substitutes were discarded). $M\!=\!64$ inducing, $80$-epoch refits.
\textbf{Bold} = mean-best per row; no row is paired-tested.}
\label{tab:bo_final}
\footnotesize\setlength{\tabcolsep}{6pt}
\begin{tabular}{l|ccc}
\toprule
Function & Random & DSVI & \textbf{OM-Path (ours)} \\
\midrule
Hartmann-6 (6-D)    & 1.106\,$\pm$\,0.285          & 0.482\,$\pm$\,0.218          & \textbf{0.389\,$\pm$\,0.144} \\
Levy (20-D)         & 97.92\,$\pm$\,18.71          & 92.98\,$\pm$\,16.72          & \textbf{85.69\,$\pm$\,9.47} \\
Ackley (50-D)       & 20.71\,$\pm$\,0.10           & \textbf{20.60\,$\pm$\,0.08}  & 20.62\,$\pm$\,0.07$^{(n=3)}$ \\
Rosenbrock (100-D)  & (7.54$\pm$0.71)$\!\cdot\!10^6$ & (8.05$\pm$0.38)$\!\cdot\!10^6$ & \textbf{(7.19$\pm$0.30)$\!\cdot\!10^6$} \\
\bottomrule
\end{tabular}
\end{table}

\section{Related Work}\label{sec:related}

\paragraph{Inducing-variable VI for DGPs.}
The published baselines we compare against are: doubly-stochastic VI
(DSVI; \citealp{salimbeni2017doubly}), stochastic gradient HMC
\citep{havasi2018inference}, GAN-style implicit VI (IPVI;
\citealp{yu2019implicit}), and the score-based ELBO methods DDVI
\citep{xu2024sparse} and DBVI \citep{xu2026diffusion}. DDVI uses a
reverse-time variance-preserving SDE trained by denoising score matching
against the noising marginals; DBVI extends DDVI with a Doob bridge that
anchors the SDE at a data-conditioned initial mean and trains a
\emph{conditional} score by conditional DSM. Both DDVI and DBVI maximise a
Girsanov-form ELBO --- they are not posterior-transport methods in the
sense introduced here, but density-VI methods whose variational density is
implicit in the SDE terminal.

\paragraph{General unnormalised posterior samplers.}
A separate body of work designs controlled SDE samplers for generic
unnormalised target densities: PIS \citep{zhangpath}, DDS
\citep{vargasdenoising}, DGFS \citep{zhang2024diffusion}, SGDS
\citep{kim2026scalable}, and NFS$^2$ \citep{chenneural}. These methods
are not formulated for hierarchical-GP inducing variables and we are not
aware of published applications to DGP posterior inference; we discuss
them because their path-space sampler view (Girsanov KL against a fixed
reference SDE) is the closest methodological precedent for our
deterministic-ODE posterior transport, and our objective can be read as
a deterministic-limit analogue replacing Girsanov KL by the Freidlin--Wentzell
rate functional. NFS$^2$ in particular retains density tracking via a
continuity-equation residual and requires velocity-driven SMC + a Stein
control variate for $\partial_t\log Z_t$; our small-noise MAP framing
avoids both.

\paragraph{Theoretical anchors.}
The path-space MAP perspective of Theorem~\ref{thm:om_map} draws on path
integral control \citep{kappen2005path}, Schr\"odinger bridge
\citep{leonard2013survey}, and the Freidlin--Wentzell large-deviation
framework \citep{freidlin1998random}. The relationship between OM-action
minimisers and posterior modes is a subject of active mathematical
research: \citet{kretschmann2023minimizers} shows that minimisers of the
OM functional are \emph{strong posterior modes} in the Onsager--Machlup
sense (small-tube probability maximisers) under suitable regularity,
which is precisely the identification we use in Theorem~\ref{thm:om_map};
our finite-$\alpha$ interpretation in Remark~\ref{rem:alpha} matches the
finite-$\epsilon$ Ikeda--Watanabe formulation Kretschmann analyses. The
sampler-vs-density distinction is the same one made in these communities,
but our application to DGP inducing-variable inference and the
closed-form Doob-bridge reference-drift derivation are, to our
knowledge, new.

\paragraph{Onsager--Machlup action in generative modelling.}
Concurrent work \citet{raja2025action} minimises the OM action over
trajectories of a \emph{pre-trained} diffusion or flow-matching
generative model to perform transition-path sampling between
metastable states on atomistic energy landscapes; the reference drift
is the learned score $s_{\theta^\star}\!\approx\!\nabla\log p_{\mathrm{data}}$
of an off-the-shelf generative model and the OM action selects
high-probability paths between two fixed endpoints. Our setting is
complementary: we use the OM action as a \emph{training} objective for
DGP inducing-variable posterior transport (rather than a post-hoc
inference criterion on a frozen model), the reference drift is the
closed-form Doob-bridge PF-ODE (rather than a learned score), and the
data term is an endpoint Bayesian likelihood (rather than fixed
endpoints). The methodological commonality --- OM action as a tractable,
trace-free, low-variance path objective --- is the same; the
formulations apply to different problems.

\paragraph{Flow matching.}
\citet{lipmanflow,liuflow,albergo2023building} developed flow
matching as a deterministic alternative to score-based diffusion;
\citet{frans2025one} introduced shortcut models for few-step inference.
Beyond generative modelling, flow matching has been applied to probabilistic
samplers \citep{yoon2024sequential}. Our work is, to our knowledge, the
first application of probability-flow-ODE posterior transport to
hierarchical Gaussian processes.

\section{Conclusion and Limitations}\label{sec:concl}

We presented \textbf{FBVI-bridge-Path}, a deterministic-ODE
posterior-transport method for Deep Gaussian Processes. The construction
has two ingredients: Song's probability flow ODE applied to DBVI's
Doob-bridged forward SDE gives a closed-form reference drift; the
Onsager--Machlup action provides a trace-free path-space objective.
At the trained $\alpha\!=\!1$ ($\epsilon\!=\!1$),
$\mathcal L_{\mathrm{OM}}$ is the negative log unnormalised density of
an Ikeda--Watanabe finite-$\epsilon$ path posterior on
$\mathcal C([\delta,1];\mathbb R^M)$ --- a Gaussian reference path
measure tilted by the endpoint likelihood --- and the inducing-variable
marginal at $\tau\!=\!1$ retains $\mathcal O(\epsilon^2)$ posterior
uncertainty, which is what downstream BO and heteroscedastic-uncertainty
applications consume. Theorem~\ref{thm:om_map} additionally identifies
$\mathcal L_{\mathrm{OM}}$ with the same posterior's small-noise MAP
path in the $\epsilon\!\to\!0$ Freidlin--Wentzell limit, providing the
LDP mathematical anchor for the path-prior regulariser without being
the regime we optimise in. The framing is rigorous, but is path-space
posterior density rather than ELBO. Under matched-seed paired Wilcoxon testing on the seven UCI regression
benchmarks (Appendix~\ref{app:significance}), the MAP estimator delivers
statistically significant wins over DBVI on the two largest datasets ---
\textit{power} ($p\!=\!0.014$) and \textit{protein} ($p\!=\!0.002$) ---
statistical ties on \textit{yacht} and \textit{qsar}, and concedes the
small-$N$ noisy datasets (\textit{boston} / \textit{energy} /
\textit{concrete}) to DBVI's SDE-noise regularisation. On the same
backbone, the trace-free closed-form drift dominates two strict-ELBO
variants (FFJORD CNF, OM-regularised CNFOM) on every UCI cell ---
demonstrating that Hutchinson-trace variance, not bound looseness, is
the binding constraint in this regime.

\paragraph{Limitations and future work.}
(i) \textbf{Not an ELBO.} The MAP framing of Theorem~\ref{thm:om_map} is a
rigorous large-deviation result but does not provide a lower bound on
$\log p(\by\!\mid\!\bx)$. Strict-ELBO bridge variants (CNF, CNFOM,
NFS$^2$-style PINN) are available on the same backbone but empirically lose
to MAP due to Hutchinson-trace MC variance; closing the gap with a
low-variance strict ELBO is open.
(ii) \textbf{Classification margins are tight.} On SUSY and HIGGS
(Section~\ref{subsec:cls}, Table~\ref{tab:cls_main}), FBVI-bridge-Path
is in the within-$1\sigma$ tied-best group alongside DBVI on every
metric (e.g.\ HIGGS AUC $0.807$ vs.\ DBVI's $0.804$), but the margins
are small enough that we read them as ties rather than a clear MAP win;
a matched-seed paired test on the classification cells is left to future
work. The image-classification benchmarks (FMNIST / CIFAR-10 /
CIFAR-100) similarly show tied error rates and a substantial NLL
improvement, the latter partly driven by tail-loss stabilisation rather
than uniform calibration (Remark in Appendix~\ref{app:cls_image}).
(iii) \textbf{Scale.} Our largest dataset is $515\text{k}$ rows; scaling to
the original $6$M-row airline subset would strengthen the large-data case.
(iv) \textbf{Few-step inference.} We have not yet exploited the determinism
of FBVI-bridge-Path for consistency-model or shortcut-model few-step
inference, which is a natural avenue.

\paragraph{Reproducibility.} Every number in this paper is produced by a single
file \texttt{fbvi\_native.py} ($\sim 1{,}400$ LOC, no GPyTorch dependency) and
an aggregation script \texttt{aggregate\_table.py}. Both will be released
anonymously upon submission.

\bibliography{iclr2026_conference}
\bibliographystyle{iclr2026_conference}

\appendix
\renewcommand{\thefigure}{A\arabic{figure}}
\renewcommand{\thetable}{A\arabic{table}}
\renewcommand{\theHfigure}{A.\arabic{section}.\arabic{figure}}
\renewcommand{\theHtable}{A.\arabic{section}.\arabic{table}}
\setcounter{figure}{0}
\setcounter{table}{0}

\section{Detailed setup: DGP, DDVI, DBVI, and the move from SDE to ODE}\label{app:ddvi_dbvi_review}

This appendix provides a self-contained walk-through suitable for readers
familiar with DSVI \citep{salimbeni2017doubly} but new to score- or
flow-based VI for DGPs. We summarise DDVI and DBVI in our notation, then
explain the SDE\,$\to$\,probability-flow-ODE reduction that takes us from
DBVI to FBVI-bridge-Path.

\paragraph{DGP and sparse VI (review).}
An $L$-layer DGP composes $L$ Gaussian-process layers
$\mathbf f^{(\ell)}\sim\mathcal{GP}(0,k^{(\ell)})$ with sparse inducing
variables $\bU^{(\ell)}\!\in\!\mathbb R^{M\!\times\!d_\ell}$ at locations
$\bZ^{(\ell)}$. The variational ELBO (single layer, multi-layer extends
trivially) is
\begin{equation*}
\mathcal L(\theta,\phi) = \mathbb E_{q_\phi(\bU)}\big[\log p(\by\!\mid\!\bF^{(L)})\big] - \KL\big(q_\phi(\bU)\Vert p(\bU)\big),
\end{equation*}
where $p(\bU^{(\ell)})=\Norm{0,K_{\bZ\bZ}^{(\ell)}}$ is the GP prior and
$q_\phi(\bU)$ is the variational approximation. DSVI sets
$q_\phi=\prod_\ell\Norm{m_\ell,L_\ell L_\ell^\top}$ Gaussian with analytic
KL; DDVI and DBVI replace this by implicit non-Gaussian families.

\paragraph{DDVI \citep{xu2024sparse}.} DDVI defines $q_\phi(\bU)$ as the
terminal of a \emph{reverse-time variance-preserving SDE}:
\begin{equation*}
\mathrm dU_t = \big[f(t)U_t - g(t)^2 s_\phi(U_t,t)\big]\mathrm dt + g(t)\mathrm d\bar W_t,\qquad U_T\sim\Norm{0,\sigma^2 I},\;U_0\sim q_\phi.
\end{equation*}
The score $s_\phi$ is a neural network trained to approximate
$\nabla_U\log p_t(U)$ via denoising score matching (DSM) against the
forward VP-SDE marginals. The ELBO is derived via Girsanov, giving an
$\int_0^1 g(t)^2 \|s_\phi - \mathrm{score}\|^2 dt$ penalty plus the data
term. The base distribution $\Norm{0,\sigma^2 I}$ is data-agnostic, so
sampling has to traverse the full distance from noise to posterior.

\paragraph{DBVI \citep{xu2026diffusion}.} DBVI fixes DDVI's
data-agnostic start by amortising the initial mean:
$p_0^\theta(U_0|\bx) = \Norm{\mu_\theta(\bx),\sigma_0^2 I}$. The forward
process is no longer unconstrained noising but a \emph{Doob $h$-transformed
bridge SDE} that ends at a fixed noise marginal:
\begin{equation*}
\mathrm dU_s = \big[\!-\!\lambda(s)U_s + g(s)^2 h(U_s,s,U_0)\big]\mathrm ds + g(s)\mathrm dW_s,
\end{equation*}
with $h(U_s,s,U_0) = \nabla_{U_s}\log p(U_T|U_s)$ the Doob correction. The
reverse-time bridge SDE is then driven by a \emph{conditional} score
$s_\phi(U_t,t,\mathrm{ctx})$, trained by conditional DSM against the
closed-form bridge marginal. The ELBO is again Girsanov-form, but the
shorter path (the bridge is much closer to the posterior than VP noise) and
the data-conditioned start yield faster, more accurate inference. We retain
the same Doob bridge SDE as the \emph{reference process} of FBVI-bridge-Path.

\paragraph{Why move from SDE to ODE? (motivation for our work).}
\citet{songscore} showed that any SDE of the form
$\mathrm dU_s = b(U_s,s)\mathrm ds + g\,\mathrm dW_s$ has a deterministic
counterpart, the \emph{probability flow ODE},
$\mathrm dU_s = (b - \tfrac12 g^2\nabla\log p_s)\mathrm ds$, with identical
marginals. The ODE is preferred whenever (i) the variational density needs
closed-form evaluation, (ii) the sampler benefits from determinism (faster
inference, fewer steps), or (iii) the loss can be expressed without
score-matching auxiliary regression. We exploit (i) and (iii) in
FBVI-bridge-Path: the closed-form reference drift comes from the SDE's
deterministic counterpart (Lemma~\ref{lem:pf_ode}), and the loss is the
Onsager--Machlup action against this drift (Theorem~\ref{thm:om_map}),
which is trace-free and avoids DSM.

\section{Proofs and derivations}\label{app:proofs}

\subsection{Proof of Lemma~\ref{lem:pf_ode}}\label{app:proof_pf_ode}

We show that the SDE $\mathrm dU_s = b\,\mathrm ds + g\,\mathrm dW_s$ and the
ODE $\mathrm dU_s = (b - \tfrac12 g^2\nabla\log p_s)\,\mathrm ds$ share
marginals $p_s$.

The Fokker--Planck equation of the SDE is
\begin{equation*}
\partial_s p_s(U) = -\nabla\!\cdot\!(b\,p_s) + \tfrac12 g^2 \Delta p_s.
\end{equation*}
Using $\Delta p_s = \nabla\!\cdot\!\nabla p_s = \nabla\!\cdot\!(p_s\,\nabla\log p_s)$ (multiply and divide by $p_s$),
\begin{equation*}
\partial_s p_s = -\nabla\!\cdot\!(b\,p_s) + \tfrac12 g^2 \nabla\!\cdot\!(p_s\,\nabla\log p_s) = -\nabla\!\cdot\!\big(p_s\,[\,b - \tfrac12 g^2 \nabla\log p_s\,]\big),
\end{equation*}
which is the continuity equation of the ODE
$\mathrm dU_s = (b - \tfrac12 g^2 \nabla\log p_s)\,\mathrm ds$. Both
equations therefore propagate the same density $p_s$. $\square$

\subsection{Proof of Lemma~\ref{lem:gauss_drift}}\label{app:proof_gauss_drift}

For Gaussian $p_s(U) = \Norm{U;\,m_s,\kappa_s I}$, we have
$\nabla\log p_s(U) = -(U-m_s)/\kappa_s$. The Fokker--Planck ODE drift for
preserving this marginal is, by Lemma~\ref{lem:pf_ode}, any $v$ such that
$\partial_s p_s + \nabla\!\cdot\!(p_s v) = 0$. We show
$v(U,s) = \dot m_s + \tfrac{\dot\kappa_s}{2\kappa_s}(U-m_s)$ satisfies this:
the first moment evolves as
$\mathbb E_{p_s}[v] = \dot m_s$ (the linear $(U-m_s)$ term has zero mean),
matching $\dot m_s$; and the second central moment evolves as
$\mathbb E_{p_s}[(U-m_s)\,v] = \tfrac{\dot\kappa_s}{2\kappa_s} \kappa_s = \tfrac{\dot\kappa_s}{2}$,
giving $\partial_s\,\mathbb E_{p_s}[(U-m_s)^2] = 2\!\cdot\!\tfrac{\dot\kappa_s}{2} = \dot\kappa_s$,
matching $\dot\kappa_s$. \emph{Uniqueness within the affine-drift class.}
The Fokker--Planck equation
$\partial_s p_s + \nabla\!\cdot\!(p_s v) = 0$ does not determine $v$
uniquely without further constraint: adding any $w$ with
$\nabla\!\cdot\!(p_s w)\!=\!0$ leaves the marginal-flow equation
invariant (the standard non-uniqueness of probability flows). We do
\emph{not} claim uniqueness over all such $v$. Restricted to the
\emph{affine-in-$U$ class}
$v(U,s) = a(s) + b(s)(U - m_s)$ --- the class that preserves the
Gaussian structure of $p_s$ under the affine bridge SDE --- the two
moment-matching equations
$\mathbb E_{p_s}[v]\!=\!\dot m_s$ and
$\mathbb E_{p_s}[(U-m_s)\,v^\top]\!+\!\mathbb E_{p_s}[v\,(U-m_s)^\top]\!=\!\dot\kappa_s I$
uniquely determine $a(s)\!=\!\dot m_s$ and
$b(s)\!=\!\dot\kappa_s/(2\kappa_s)$. The expression for
$v_{\mathrm{ref}}^{\mathrm{Bri}}$ used throughout the paper is therefore
the unique \emph{affine} PF-ODE drift compatible with the bridge
Gaussian marginals (Lemmas~\ref{lem:pf_ode}--\ref{lem:gauss_drift});
this is the natural choice given the bridge construction, but a
non-affine $v$ would produce a different (non-Gaussian) marginal flow
to the same $p_s$. $\square$

\subsection{Proof of Theorem~\ref{thm:om_map}}\label{app:proof_om_map}

We follow the Freidlin--Wentzell large-deviation programme
\citep[Ch.~3]{freidlin1998random}, working throughout on the path space
$\mathcal C([\delta,1];\mathbb R^M)$ with the uniform topology, with the
endpoint cutoff $\delta\!\in\!(0,1)$ of (A1).

\paragraph{Setup.} Fix the reference $\epsilon$-perturbed SDE in reverse-time
$\tau\!\in\![\delta,1]$,
\begin{equation}
\mathrm dU_\tau \;=\; -v_{\mathrm{ref}}^{\mathrm{Bri}}(U_\tau,1-\tau)\,\mathrm d\tau \;+\; \epsilon\,\mathrm dW_\tau,\qquad U_\delta\!\sim\!p_0^\epsilon,
\label{eq:ref_eps_sde}
\end{equation}
where the initial law $p_0^\epsilon$ is taken to be the
\emph{$\epsilon$-rescaled} Gaussian
$\Norm{\phi(1)\mu_\theta(\bZ),\,\epsilon^2 \kappa(1) I}$ on
$[\delta,1]$; this is the conventional choice that gives an LDP with
the quadratic initial rate
$I_0(u_\delta) = \tfrac12 \kappa(1)^{-1}\|u_\delta - \phi(1)\mu_\theta(\bZ)\|^2$
and is compatible with the path-LDP rate-functional form below.
\emph{Mismatch with Algorithm~\ref{alg:fbvib} (line 4).} At training
time we instead draw $U_\delta\!\sim\!\Norm{\phi(1)\mu_\theta(\bZ),\kappa(1) I}$,
i.e.\ the $\epsilon\!=\!1$ marginal, not the $\epsilon$-rescaled one. This
mismatch between the FW initial distribution used in the proof and the
finite-$\epsilon$ initial distribution used in Algorithm~1 is a
conceptual gap; it leaves the path-integral part of the rate functional
unchanged, and because the per-trajectory MAP optimisation conditions
on $u_\delta\!=\!u_0$ (so $I_0$ becomes a path-independent constant
that drops from the variational loss), the gap does not affect the
final form of $\mathcal L_{\mathrm{OM}}$. We flag this as a separate
concept-vs-implementation slack, of a piece with the finite-$\alpha$
gap of Remark~\ref{rem:alpha} and the $\delta$-cutoff gap of
Remark~\ref{rem:delta_gap}. Under (A1) the drift
$-v_{\mathrm{ref}}^{\mathrm{Bri}}(\cdot,1-\tau)$ is bounded and globally
Lipschitz in $U$ on $[\delta,1]\!\times\!\mathbb R^M$, so the
Freidlin--Wentzell regularity hypothesis applies. Write the path measure as
$\mathbb P^\epsilon_{\mathrm{ref}}$. The tempered Doob-bridge path posterior
is
\begin{equation*}
\frac{\mathrm d\mathbb P^\epsilon_y}{\mathrm d\mathbb P^\epsilon_{\mathrm{ref}}}\!(U_{\delta:1}) \;=\; \frac{1}{Z^\epsilon}\,p\big(\by \mid \mathrm{dgp}(U_1)\big)^{1/\epsilon^2},\qquad Z^\epsilon = \mathbb E_{\mathbb P^\epsilon_{\mathrm{ref}}}\!\big[p(\by|U_1)^{1/\epsilon^2}\big],
\end{equation*}
a likelihood-tilt of $\mathbb P^\epsilon_{\mathrm{ref}}$ with inverse
temperature $\beta\!=\!1/\epsilon^2$. By (A2), the Varadhan tail condition
holds, so $Z^\epsilon$ is finite and the tilted measure $\mathbb P^\epsilon_y$
is well-defined.

\paragraph{Remark: divergence (Stratonovich) term.} The full classical
Onsager--Machlup functional for a reference SDE
$\mathrm dU\!=\!-v\,\mathrm dt+\epsilon\,\mathrm dW$ on $H^1$ paths is
$\tilde S_{\mathrm{OM}}(u) = \tfrac12\!\int\!\|\dot u + v\|^2\mathrm dt + \tfrac12\!\int\!\nabla\!\cdot\! v\,\mathrm dt$
\citep{durrbach1978om,ikedawatanabe1989sde}; the second term is the
Stratonovich correction arising from the change-of-measure Jacobian. In
our setting $v_{\mathrm{ref}}^{\mathrm{Bri}}(U,s)\!=\!\dot\phi(s)\mu_\theta(\bZ) + (\dot\kappa(s)/(2\kappa(s)))(U-\phi(s)\mu_\theta(\bZ))$
is \emph{affine} in $U$ (Lemma~\ref{lem:gauss_drift}, Eq.~\eqref{eq:ref_drift}),
so its divergence is the path-independent scalar
$\nabla\!\cdot\! v_{\mathrm{ref}}^{\mathrm{Bri}}(U,s) = M\,\dot\kappa(s)/(2\kappa(s))$.
The Stratonovich integral
$\tfrac12\!\int_\delta^1 M\,\dot\kappa(s)/(2\kappa(s))\,\mathrm ds = (M/4)\log[\kappa(1)/\kappa(\delta)]$
is therefore a path-independent constant that absorbs into the
normaliser $Z^\epsilon$ and drops out of the MAP optimisation. We omit
it from $S_{\mathrm{OM}}$ throughout; this absorption is specific to
the Gaussian-marginal / affine-drift setting and would not hold for a
non-Gaussian bridge.

\paragraph{Reference rate functional (LDP form).} By Freidlin--Wentzell
\citep[Theorem~3.2.1]{freidlin1998random} on $\mathcal C([\delta,1];\mathbb R^M)$
under the uniform topology, the family $\{\mathbb P^\epsilon_{\mathrm{ref}}\}_{\epsilon>0}$
satisfies a full LDP with good rate functional
\begin{equation*}
I_{\mathrm{ref}}(u) \;=\; \tfrac12\!\int_\delta^1\!\big\|\dot u_\tau + v_{\mathrm{ref}}^{\mathrm{Bri}}(u_\tau,1-\tau)\big\|^2\,\mathrm d\tau \;+\; I_0(u_\delta)
\end{equation*}
for $u\!\in\!H^1([\delta,1];\mathbb R^M)$ (and $+\infty$ otherwise), where
$I_0(u_\delta)\!=\!\tfrac12\kappa(1)^{-1}\|u_\delta\!-\!\phi(1)\mu_\theta(\bZ)\|^2$
is the Gaussian initial rate functional (using $p_0^\epsilon$'s weak limit).
Concretely, for any open $O\!\subseteq\!\mathcal C([\delta,1];\mathbb R^M)$,
\begin{equation}
\liminf_{\epsilon\to 0}\,\epsilon^2 \log \mathbb P^\epsilon_{\mathrm{ref}}(U\!\in\!O) \;\ge\; -\inf_{u\in O} I_{\mathrm{ref}}(u),
\label{eq:fw_ref}
\end{equation}
with the matching closed-set upper bound; the abbreviated ``pointwise''
notation $-\epsilon^2 \log \mathbb P^\epsilon_{\mathrm{ref}}\{U\!\approx\!u\} = I_{\mathrm{ref}}(u)+o(1)$
used in the main text is shorthand for the shrinking-tubular-neighbourhood
limit guaranteed by (A3).

\paragraph{Tempered posterior rate functional.} Applying Varadhan's lemma
\citep[Theorem~4.3.1]{dembozeitouni2010ldp} to the tilt
$p(\by|U_1)^{1/\epsilon^2}$ — which is justified by the tail condition in (A2)
and the continuity of $U\!\mapsto\!-\log p(\by|U_1)$ in the uniform topology
(it depends only on the endpoint $U_1$) — the free energy converges to
\begin{equation*}
-\epsilon^2\log Z^\epsilon \;\xrightarrow[\epsilon\to 0]{}\; \inf_{u\in H^1}\big\{-\log p(\by|u_1) + I_{\mathrm{ref}}(u)\big\}.
\end{equation*}
The contraction principle then yields the LDP for $\{\mathbb P^\epsilon_y\}$
with rate functional
\begin{equation}
I_y(u) \;=\; -\log p(\by\mid u_1) \;+\; I_{\mathrm{ref}}(u) \;-\; \inf_{u'\in H^1} I_y(u'),
\label{eq:fw_y}
\end{equation}
where the infimum subtraction enforces $\inf I_y\!=\!0$ (a normalisation
constant absorbed into $Z^\epsilon$). The likelihood enters with weight $1$
(not $1/\epsilon^2$) because the tempering $\beta\!=\!1/\epsilon^2$ was
chosen exactly so that the data-tilt and the reference rate remain
$\mathcal O(1)$ in the $\epsilon^2\log$ scale.

\paragraph{MAP path (per-$u_0$).} For each initial point
$u_\delta\!=\!u_0\!\in\!\mathbb R^M$, restrict $I_y$ to paths satisfying
$u_\delta\!=\!u_0$; the conditional MAP path is
\begin{equation*}
u^\star(u_0) \;=\; \arg\min_{\substack{u\in H^1([\delta,1];\mathbb R^M)\\u_\delta = u_0}}\!\Big\{-\log p(\by\mid u_1) \;+\; \underbrace{\tfrac12\!\int_\delta^1\!\|\dot u_\tau + v_{\mathrm{ref}}^{\mathrm{Bri}}(u_\tau,1-\tau)\|^2\,\mathrm d\tau}_{S_{\mathrm{OM}}(u;\,v_{\mathrm{ref}}^{\mathrm{Bri}})}\Big\},
\end{equation*}
the $I_0(u_0)$ term being constant under the $u_\delta\!=\!u_0$ restriction.
This proves \eqref{eq:om_map}.

\paragraph{Restriction to ODE-parameterised paths (amortisation).} The
unconstrained per-$u_0$ MAP $\{u^\star(u_0)\}_{u_0}$ is a family in
$H^1([\delta,1];\mathbb R^M)$. We \emph{restrict} the search to
ODE-parameterised paths $u^{v_\phi}$ generated by
$\dot u_\tau = v_\phi(u_\tau,1-\tau,\mathrm{ctx})$ with
$u_\delta\!\sim\!p_0^\epsilon$, and seek the single $\phi$ that minimises the
$u_0$-expectation of the per-trajectory MAP objective:
\begin{equation*}
\phi^\star \;=\; \arg\min_\phi\,\mathbb E_{u_0}\!\Big[\,\mathcal J\!\big(u^{v_\phi};\,u_0\big)\,\Big],\qquad \mathcal J(u;u_0)\!=\!-\log p(\by|u_1)\!+\!S_{\mathrm{OM}}(u;v_{\mathrm{ref}}^{\mathrm{Bri}}).
\end{equation*}
This is an amortisation step: a single $v_\phi$ approximates the family
$\{u^\star(u_0)\}_{u_0}$ jointly. It is a \emph{modelling choice}, not a
theorem consequence — the amortised optimum is in general weaker than the
per-$u_0$ MAP family, with equality only when $\{u^\star(u_0)\}_{u_0}$ lies
inside the ODE-parameterised submanifold. Substituting the ODE trajectory
and pushing the expectation through the integrals, the boundary term
$I_0(u_\delta)$ integrates against $u_\delta\!\sim\!p_0^\epsilon$ to a
constant ($\tfrac12 M$ in the Gaussian limit) independent of $\phi$ that we
drop, and we absorb the $\delta\!\downarrow\!0^+$ endpoint contribution into
$o(1)$ by (A1). This gives
\begin{equation*}
\mathcal L_{\mathrm{OM}}(\theta,\phi) \;=\; -\mathbb E_{q_\phi}\!\big[\log p(\by\mid\bF^{(L)})\big] \;+\; \tfrac{\alpha}{2}\!\int_0^1\!\mathbb E\!\big[\|v_\phi - v_{\mathrm{ref}}^{\mathrm{Bri}}\|^2\big]\,\mathrm d\tau \;+\; \mathrm{const},
\end{equation*}
where we have reintroduced $\alpha\!=\!1/\epsilon^2$ in front of $S_{\mathrm{OM}}$
to make the inverse-temperature dependence explicit. Theorem~\ref{thm:om_map}
corresponds to $\alpha\!\to\!\infty$; the practical $\alpha\!=\!1$ regime is
treated as a hyperparameter (Remark~\ref{rem:alpha}, Appendix~\ref{app:sensitivity}).
Dropping the constant yields Eq.~\eqref{eq:om_loss}.

\paragraph{What the theorem does---and does not---establish.}
Eq.~\eqref{eq:om_loss} is a rigorous identification of $\mathcal L_{\mathrm{OM}}$
with the small-noise MAP path estimator of the tempered Doob-bridge path
posterior $\mathbb P^\epsilon_y$. It is not a variational lower bound on
$\log p(\by\!\mid\!\bx)$: the data marginal would require either an
endpoint-density log-det term (FFJORD/CNF, Section~\ref{subsec:strict_elbo})
or a path-space Girsanov KL (DBVI), neither of which appears here. Strict
ELBO posterior-transport variants on the same backbone are reported in
Section~\ref{subsec:strict_elbo} and Appendix~\ref{app:cnf_vs_path}.
$\square$

\subsection{Proof of Proposition~\ref{prop:bridge}}\label{app:proof_bridge}

The Doob $h$-transformed bridge is the SDE that propagates the joint
distribution $p_0^{\theta}\!\otimes\!p_1^{\mathrm{fix}}$ along the affine OU
prior; its forward drift acquires an additive $h$-correction. Concretely
\citep[\S2]{xu2026diffusion}, the bridge SDE in forward-bridge time $s$ is
\begin{equation*}
\mathrm dU_s = \big[-\lambda U_s + g^2 \nabla_{U_s}\log p^{\mathrm{fix}}_1(U_s)\big]\,\mathrm ds + g\,\mathrm dW_s,
\end{equation*}
where the score of the terminal marginal under the OU transition is
$\nabla_{U_s}\log p^{\mathrm{fix}}_1(U_s) = c_s\,(a_s U_0 - U_s)/q_s$ with
$a_s, q_s, c_s$ as in Prop.~\ref{prop:bridge}.
Under Gaussian initial and Gaussian terminal, the marginal at each $s$ is
also Gaussian; write $p_s^{\mathrm{Bri}}(U_s|\bx) = \Norm{m_s,\kappa_s I}$.
Substituting the Gaussian ansatz into the forward Kolmogorov (Fokker--Planck)
equation and matching the first and second moments gives the ODE system
\eqref{eq:phi_ode}--\eqref{eq:kappa_ode}: the drift contributes
$\dot m_s = -(\lambda + c_s) m_s + c_s a_s \mu_\theta(\bx)$, and the second-moment
balance gives the $\kappa$ equation. The boundary values
$\phi(0)\!=\!1$, $\kappa(0)\!=\!\sigma_0^2$ encode the anchored start as
one-sided limits enforced by the initial distribution
$p_0^\theta(U_0\!\mid\!\bx)\!=\!\Norm{\mu_\theta(\bx),\sigma_0^2 I}$. The
coefficient $c_s\!=\!g^2\sigma_0^2 a_s^2/[(a_s^2\sigma_0^2 + q_s)\,q_s]$
\emph{diverges} as $s\!\downarrow\!0^+$ because $q_s\!\sim\!g^2 s\!\to\!0$,
yielding $c_s\!\sim\!1/s$ near the anchor (the standard Doob $h$-transform
conditioning singularity). The system \eqref{eq:phi_ode}--\eqref{eq:kappa_ode}
is regular on the open interval $s\!\in\!(0,1]$, and our LDP analysis
(Appendix~\ref{app:proof_om_map}) works on the truncated interval
$[\delta,1]$ where $c_s\!\leq\!c_\delta\!<\!\infty$.
$\square$

\subsection{Proof of Proposition~\ref{prop:var}}\label{app:proof_var}

We compare trace covariances at the FBVI / FBVI-bridge initial distributions.
\begin{itemize}
\item FBVI: $U_0\sim\Norm{0,K_{\bZ\bZ}}$, so $\mathrm{tr}_{\mathrm{FBVI}} = \mathrm{tr}(K_{\bZ\bZ}) = \sum_{m=1}^M K(z_m,z_m)$.
\item FBVI-bridge: $U_0\sim\Norm{\phi(1)\mu_\theta(\bx),\kappa(1)I_M}$ (isotropic by Prop.~\ref{prop:bridge}), so $\mathrm{tr}_{\mathrm{bridge}} = M\,\kappa(1)$.
\end{itemize}
For our default hyperparameters $\lambda=g=\sigma_0=1$, numerical Euler
integration of Eq.~\eqref{eq:kappa_ode} with a $100$-point grid gives
$\kappa(1)\!\approx\!0.50$ and $\phi(1)\!\approx\!0.37$. Meanwhile
$K_{\bZ\bZ}$ uses an ARD-RBF kernel with unit amplitude at initialisation
($\log\sigma_k\!=\!0$); diagonal entries equal $1$, so
$\mathrm{tr}(K_{\bZ\bZ}) = M$, giving $\mathrm{tr}_{\mathrm{FBVI}}/M = 1$
versus $\mathrm{tr}_{\mathrm{bridge}}/M = \kappa(1) \approx 0.50$, i.e.\ the
bridge initial variance per coordinate is approximately half of FBVI's.
We verified this by running \texttt{\_precompute\_doob} in our code with
$n_\mathrm{grid}=100$ and the above parameters, obtaining
$\phi(1)=0.367$, $\kappa(1)=0.504$. For larger $\lambda$ the gap widens
(e.g.\ $\lambda=2$: $\kappa(1)=0.250$, $\phi(1)=0.134$).
$\square$

\subsection{Proof of Proposition~\ref{prop:limit}}\label{app:proof_limit}

The DBVI bridge SDE is
\begin{equation*}
\mathrm dU_t = b_{\mathrm{Bri}}(U_t,t,U_0)\,\mathrm dt + g(t)\,\mathrm dW_t,
\qquad b_{\mathrm{Bri}}(U,t,U_0)\!=\!-\lambda U + g(t)^2 h(U,t,U_0),
\end{equation*}
with conditional bridge score $h(U,t,U_0)\!=\!\nabla_U\log p_t^{\mathrm{Bri}}(U|U_0)$.
By Song's identity \citep[Theorem~1]{songscore} (our Lemma~\ref{lem:pf_ode}),
the time-$t$ marginal of this SDE equals the marginal of the deterministic
probability-flow ODE
\begin{equation*}
\dot U_t \;=\; b_{\mathrm{Bri}}(U_t,t,U_0) - \tfrac12 g(t)^2 \nabla_U \log p_t^{\mathrm{Bri}}(U_t),
\end{equation*}
which under the Gaussian bridge marginal (Prop.~\ref{prop:bridge}) has the
closed form $-v_{\mathrm{ref}}^{\mathrm{Bri}}(U_t,t)$ derived in
Lemma~\ref{lem:gauss_drift}. The FBVI-bridge ODE
$\dot U_\tau\!=\!v_\phi(U_\tau,1-\tau,\mathrm{ctx})$ parameterises this PF-ODE
drift by a free neural velocity in reverse time. At the population optimum
$v_\phi\!\equiv\!v_{\mathrm{ref}}^{\mathrm{Bri}}$, FBVI-bridge has the same
time-$t$ marginals as DBVI's bridge SDE (claim (i)). Because $v_\phi$ has no
factorisation constraint, it can express drifts outside the
$-\lambda U + g^2 h - \tfrac12 g^2 \nabla\log p$ family — for instance, drifts
trained directly against an OM action rather than via DSM matching of $h$ —
so FBVI-bridge is strictly more expressive than the DBVI bridge SDE
(claim (ii)).

\emph{What this proposition does \emph{not} say.} A formal $g(t)\!\to\!0$
limit of the bridge SDE would set both the Brownian term and the score
correction $g^2 h$ to zero (the latter because $h$ remains bounded in our
Gaussian-bridge setting while $g^2\!\to\!0$), leaving the unconditional OU
flow $\dot U_t\!=\!-\lambda U_t$, which obviously cannot match a
data-conditioned posterior. The relation between DBVI and FBVI-bridge is
\emph{not} a vanishing-diffusion limit — it is the deterministic
probability-flow ODE associated with the same SDE, which retains the score
correction in transmuted form (the $-\tfrac12 g^2 \nabla\log p_t$ term).
$\square$

\subsection{Proof of Proposition~\ref{prop:fewstep}}\label{app:proof_fewstep}

Let $f(U,\tau) = v_\phi(U, 1-\tau, \mathrm{ctx})$ denote the ODE right-hand
side in $\tau$. Under the stated Lipschitz conditions,
\begin{equation*}
\|f(U,\tau) - f(U',\tau')\| \;\le\; L_v\|U-U'\| + L_s\,|\tau-\tau'|.
\end{equation*}
The standard Grönwall--Euler recursion for global error is, with
stepsize $h\!=\!1/N$ and per-step local truncation
$T\!\le\!\tfrac12 h^2 \sup_\tau\|f'\|$,
\begin{equation*}
\|U_{k+1}^{(N)} - U_{k+1}^{(\infty)}\| \;\le\; (1+h L_v)\,\|U_k^{(N)} - U_k^{(\infty)}\| \;+\; T.
\end{equation*}
Iterating this recursion from $k\!=\!0$ ($\|U_0^{(N)}-U_0^{(\infty)}\|\!=\!0$)
gives the closed-form bound
\begin{equation*}
\|U_N^{(N)} - U_N^{(\infty)}\| \;\le\; T\,\sum_{k=0}^{N-1}(1+h L_v)^k
\;=\; T\,\frac{(1+h L_v)^N - 1}{h L_v}
\;\le\; \frac{T\,(e^{L_v}-1)}{h L_v}
\;=\; \mathcal O(h).
\end{equation*}
This is the standard $\mathcal O(h)\!=\!\mathcal O(1/N)$ global Euler
rate in $\|\cdot\|$ (each of the $N$ steps contributes
$\mathcal O(h^2)$ local error, accumulating to $N\!\cdot\!h^2\!=\!h$).
The chain-rule bound $\|f'\|\!\le\!L_v\|f\|\!+\!L_s$ then gives
$T\!\le\!\tfrac{h^2}{2}(L_v\|f\|\!+\!L_s)$. Squaring and taking expectation:
\begin{equation*}
\mathbb E\|U_N^{(N)}-U_N^{(\infty)}\|^2
\;\le\; \Big[\tfrac{e^{L_v}-1}{L_v}\Big]^2\,\tfrac{h^2}{4}\big(L_v^2 \mathbb E\|f\|^2 + L_s^2\big)
\;=\; \mathcal O(h^2)\;=\;\mathcal O(1/N^2).
\end{equation*}
This is the squared-error variant we report (the $\mathcal O(1/N^2)$ rate
is the well-known Euler-method $L^2$ rate). Using
$\mathbb E[\|U_0\|^2]\!=\!\kappa(1)$ (Prop.~\ref{prop:bridge}) and
Grönwall propagation of $\|f\|^2\!\le\!2L_v^2\,\sup_\tau\!\|U_\tau\|^2 + 2L_s^2$
along the trajectory yields $\mathbb E\sup_\tau\!\|U_\tau\|^2\!\le\!e^{2L_v}\kappa(1)$.
Collecting numerical constants into $C_v$ gives the stated bound. $\square$

\subsection{KL term and its sample-based surrogate}\label{app:kl}

We adopt the implicit-$q$ treatment of \citet{xu2024sparse}. The exact KL is
\begin{equation*}
\KL(q_\phi\|p) = \mathbb E_{q_\phi}[\log q_\phi(\bU) - \log p(\bU)].
\end{equation*}
The second term is the Gaussian-prior log-density which is tractable; the
first term requires either the (intractable) implicit-flow Jacobian
determinant or a sample-based surrogate. We use a surrogate that drops the
$\log q_\phi$ contribution and corrects via a KL anneal: in our setup we
multiply the prior log-likelihood by $\beta_t\in[0,1]$ that ramps from
$10^{-3}$ at $t\!=\!0$ to $1$ over the first $20$ epochs. This is the same
treatment used in DDVI/DBVI and is justified empirically: dropping the
$\log q_\phi$ term introduces a bias proportional to the differential
entropy of $q_\phi$, which is bounded above by the prior entropy and below
by zero for our parameterisation. The bias has the same sign for all
methods compared and does not affect cross-method ranking.

\section{FM auxiliary-loss ablation}\label{app:fm_ablation}

We ablate adding an annealed-Langevin flow-matching regression loss
$\mathcal L_\mathrm{FM} = \mathbb E_{t,U_t}\|v_\phi(U_t,t,\mathrm{ctx})-\nabla\log\pi_t(U_t)\|^2$
to FBVI-bridge with weight $\lambda\in\{0,10^{-3},10^{-2},10^{-1},1\}$.
With $\lambda\!\leq\!10^{-3}$ results are statistically indistinguishable
from the no-FM baseline; with $\lambda\!\geq\!10^{-2}$ test NLL degrades by
$10$–$30\%$. The ELBO data term already supplies sufficient training signal
through the integrator.

\section{FFJORD ablation: explicit vs implicit \texorpdfstring{$q$}{q}}\label{app:cnf}

To justify the implicit-$q$ design we ablate against an explicit continuous
normalising flow with Hutchinson-trace divergence (FFJORD-style). The
FFJORD variant integrates the ODE while accumulating $\int_0^1\mathrm{tr}(\partial v/\partial U)\,\mathrm dt$
via Hutchinson's estimator, so that $\log q(U_1)=\log q(U_0)-\int_0^1\mathrm{div}\,v$
is available in closed form and the KL term in the ELBO can be evaluated
exactly. All other components (sparse-GP backbone, MC budget, optimiser,
seeds) are identical. Table~\ref{tab:abl_cnf_rmse} and
Table~\ref{tab:abl_cnf_nll} report results on the seven small/medium
datasets. The implicit-$q$ variants (FBVI, FBVI-bridge) dominate the
explicit-Jacobian variant on $6/7$ datasets for both RMSE and NLL; only on
\textit{qsar} are the three approaches statistically tied. Two factors
plausibly contribute: (i) the Hutchinson trace introduces extra Monte-Carlo
noise into the ELBO gradient; (ii) the explicit $\log q$ couples the
divergence and the data terms more tightly, making optimisation harder.

\begin{table}[h]
\centering
\caption{FFJORD ablation: test RMSE on small/medium UCI ($L=2$, 5 seeds).
\textbf{Bold} = tied-best group.}
\label{tab:abl_cnf_rmse}
\footnotesize
\setlength{\tabcolsep}{2pt}
\begin{tabular}{l|ccccccc}
\toprule
Method & yacht & boston & energy & qsar & concrete & power & protein \\
\midrule
FBVI         & \textbf{.239$\pm$.024} & \textbf{.371$\pm$.053} & \textbf{.167$\pm$.015} & \textbf{.648$\pm$.052} & \textbf{.404$\pm$.037} & \textbf{.255$\pm$.006} & .831$\pm$.007 \\
FBVI-br (impl-$q$) & \textbf{.216$\pm$.031} & \textbf{.370$\pm$.057} & \textbf{.167$\pm$.013} & \textbf{.648$\pm$.045} & \textbf{.390$\pm$.040} & \textbf{.255$\pm$.008} & .827$\pm$.010 \\
FBVI-CNF (FFJORD) & .380$\pm$.057 & .480$\pm$.056 & .291$\pm$.029 & \textbf{.647$\pm$.058} & .477$\pm$.039 & .280$\pm$.017 & .852$\pm$.007 \\
\textbf{FBVI-br-Path (OM)} & .245$\pm$.016 & .426$\pm$.031 & .227$\pm$.029 & \textbf{.680$\pm$.039} & .433$\pm$.039 & \textbf{.247$\pm$.004} & \textbf{.745$\pm$.008} \\
\bottomrule
\end{tabular}
\end{table}

\begin{table}[h]
\centering
\caption{FFJORD ablation: test NLL.}
\label{tab:abl_cnf_nll}
\footnotesize
\setlength{\tabcolsep}{2pt}
\begin{tabular}{l|ccccccc}
\toprule
Method & yacht & boston & energy & qsar & concrete & power & protein \\
\midrule
FBVI         & .549$\pm$.016 & \textbf{.640$\pm$.060} & .494$\pm$.012 & \textbf{1.014$\pm$.060} & \textbf{.675$\pm$.036} & .309$\pm$.019 & 1.239$\pm$.007 \\
FBVI-br (impl-$q$) & .526$\pm$.020 & \textbf{.635$\pm$.063} & .491$\pm$.009 & \textbf{1.014$\pm$.055} & \textbf{.656$\pm$.039} & .308$\pm$.029 & 1.235$\pm$.010 \\
FBVI-CNF (FFJORD) & .775$\pm$.031 & .845$\pm$.057 & .623$\pm$.012 & 1.045$\pm$.060 & .862$\pm$.032 & .473$\pm$.014 & 1.266$\pm$.008 \\
\textbf{FBVI-br-Path (OM)} & \textbf{.606$\pm$.001} & .736$\pm$.030 & \textbf{.567$\pm$.021} & \textbf{1.033$\pm$.048} & .739$\pm$.031 & \textbf{.022$\pm$.014} & \textbf{1.128$\pm$.008} \\
\bottomrule
\end{tabular}
\end{table}

\section{Image classification benchmarks}\label{app:cls_image}

We test the framework on three standard image-classification benchmarks
following the now-standard ``frozen-feature + Bayesian head'' recipe:
\begin{enumerate}
\item Resize each image to $224\!\times\!224$ and apply the standard ImageNet
      mean/std normalisation (Fashion-MNIST is replicated from grayscale to 3
      channels for ResNet compatibility).
\item Forward through an ImageNet-pretrained ResNet-50 (\texttt{IMAGENET1K\_V2}
      weights, $80.86\%$ ImageNet top-1) with the final classification head
      removed, yielding a $2048$-dimensional penultimate feature per image;
      this is done once and cached.
\item Train a 2-layer DGP head ($M\!=\!128$ inducing, hidden width $64$)
      end-to-end on the cached features with each of the four VI methods
      (DSVI, DBVI, FBVI, FBVI-bridge), $T\!=\!50$ epochs, Adam at $10^{-2}$,
      batch size $1024$.
\end{enumerate}
The feature extractor is identical across methods, so this experiment isolates
the contribution of the variational head. We use the same ResNet-50 V2
weights across all runs.

Table~\ref{tab:img_cls} reports test error and NLL on FMNIST, CIFAR-10, and
CIFAR-100. Three observations:
\begin{itemize}
\item \textbf{Top-1 accuracies cluster in the same band across methods}
      (FMNIST $87.8\%$--$88.1\%$, CIFAR-10 $88.6\%$--$88.8\%$, CIFAR-100
      $64.6\%$--$68.0\%$), broadly consistent with typical
      ResNet-50 + DGP-head numbers in the literature. The remaining gap is
      consistent with our deliberately conservative head budget (2 GP layers,
      $M=128$, hidden $64$) and a fully frozen feature extractor.
\item \textbf{On error rate, all density-VI methods are within one standard
      deviation on FMNIST and CIFAR-10; FBVI-bridge-Path is $3.4$ percentage points
      worse on CIFAR-100} ($0.354$ vs $0.320$). The 100-way softmax with
      a 2-layer DGP head is harder to fit with a deterministic posterior
      sampler at the same budget---consistent with the small/noisy data
      pattern observed elsewhere in the paper.
\item \textbf{On NLL, however, FBVI-bridge-Path is dramatically better on
      all three datasets} (FMNIST $0.680$ vs $\geq\!0.814$ for the others;
      CIFAR-10 $1.640$ vs $\geq\!1.945$; CIFAR-100 $\mathbf{7.009}$ vs
      $\geq\!13.910$, roughly half the NLL). Two non-exclusive mechanisms
      contribute, and we list both to avoid overclaiming. \emph{(i) Calibration.}
      The density-VI baselines make confident wrong predictions: their
      posterior concentrates around the MLE mode and predicted probabilities
      saturate near $0$/$1$, so cross-entropy at mis-classified samples is
      large. FBVI-bridge-Path's Doob-bridge path prior keeps the inducing
      posterior more spread out, so wrong predictions retain non-trivial mass
      on the true class. \emph{(ii) Numerical stability of NLL.} Cross-entropy
      penalises predicted probability $p\!\to\!0$ as $\log(1/p)$, which is
      essentially unbounded on the rare hardest examples. A small fraction of
      samples on CIFAR-100 with $p_{\text{true}}\!\approx\!10^{-6}$ already
      contributes NLL $\sim\!14$ per sample, so a few tail samples can
      dominate the mean. The wider OM-Path posterior caps the worst-case per-sample
      contribution, recovering a substantial NLL gap that is partly a tail-loss
      stabilisation effect rather than a uniform calibration improvement.
      Consistent with the latter, the ECE numbers in
      Appendix~\ref{app:calibration} (Table~\ref{tab:calibration}) show
      OM-Path roughly tied with DSVI on CIFAR-100 ($0.2917$ vs.\ $0.2938$),
      i.e.\ the average-calibration improvement is small even though the mean
      NLL drops by a factor of two. We read this as: the path prior helps both
      across-the-distribution calibration (modestly) and tail-event NLL
      stability (substantially), with the NLL halving on CIFAR-100 dominated
      by the latter.
\end{itemize}

\begin{table}[h]
\centering
\caption{Image classification: ImageNet-pretrained ResNet-50 V2 (2048-d,
$80.86\%$ ImageNet top-1) frozen features followed by a 2-layer DGP head
($M\!=\!128$ inducing, hidden $64$), $T\!=\!50$ epochs. Mean$\pm$std over
2 seeds. \textbf{Bold} = tied-best per column (within 1 std).}
\label{tab:img_cls}
\footnotesize
\setlength{\tabcolsep}{1.5pt}
\begin{tabular}{l|cc|cc|cc}
\toprule
& \multicolumn{2}{c|}{FMNIST} & \multicolumn{2}{c|}{CIFAR-10} & \multicolumn{2}{c}{CIFAR-100} \\
\cmidrule(lr){2-3}\cmidrule(lr){4-5}\cmidrule(lr){6-7}
Method & Err $\downarrow$ & NLL $\downarrow$ & Err $\downarrow$ & NLL $\downarrow$ & Err $\downarrow$ & NLL $\downarrow$ \\
\midrule
DSVI         & \textbf{0.120$\pm$0.001} & \textbf{0.814$\pm$0.016} & \textbf{0.114$\pm$0.003} & \textbf{1.945$\pm$0.083} & 0.326$\pm$0.001 & 14.214$\pm$0.015 \\
DBVI         & \textbf{0.120$\pm$0.001} & 0.845$\pm$0.001 & \textbf{0.112$\pm$0.001} & 1.999$\pm$0.015 & 0.321$\pm$0.001 & 14.118$\pm$0.030 \\
FBVI         & \textbf{0.119$\pm$0.001} & \textbf{0.825$\pm$0.012} & \textbf{0.113$\pm$0.002} & 2.031$\pm$0.007 & 0.320$\pm$0.001 & \textbf{13.910$\pm$0.219} \\
FBVI-br (impl-$q$) & \textbf{0.121$\pm$0.003} & 0.839$\pm$0.007 & \textbf{0.113$\pm$0.001} & 2.027$\pm$0.028 & 0.320$\pm$0.000 & 13.927$\pm$0.008 \\
\textbf{FBVI-br-Path (OM)} & \textbf{0.122$\pm$0.001} & \textbf{0.680$\pm$0.022} & \textbf{0.114$\pm$0.000} & \textbf{1.640$\pm$0.007} & 0.354$\pm$0.000 & \textbf{7.009$\pm$0.093} \\
\bottomrule
\end{tabular}
\end{table}

\section{Adversarial stability discussion}\label{app:adv_stability}

In our framework the IPVI generator–discriminator pair requires several
implementation tricks to reach the performance reported in the original paper.
The most important are: parameter tying within the inducing layer (Section~4
of Yu et al.), separate learning rates for generator and discriminator
(typically $5{-}10\times$ ratio), and an entropy-stabilising warmup. Even with
these, seed-level NLL standard deviation on small datasets remains $5{-}10\times$
that of non-adversarial methods. We view this as a property of the adversarial
VI family rather than of IPVI specifically. Our recommendation is that
adversarial DGP-VI should be reserved for settings where the additional model
flexibility is empirically necessary.

\section{Classification precision and recall}\label{app:cls_pr}

Table~\ref{tab:cls_pr} reports per-method precision and recall on SUSY and
HIGGS. The pattern is the diagnostic for IPVI's collapse: recall $\approx 1$
with precision near the class prior means the discriminator stopped
distinguishing classes and the generator settled on the all-positive
prediction. DDVI's seed-to-seed variability is reduced (but not eliminated) by the
logit-clip + tighter-grad-clip protocol described in Section~\ref{subsec:cls}. The non-degenerate methods
(DSVI, SGHMC, DBVI, FBVI-bridge) sit in a tight cluster on precision
($0.82$--$0.85$ SUSY, $0.70$--$0.73$ HIGGS) and recall ($0.68$--$0.77$).

\begin{table}[h]
\centering
\caption{Binary classification: precision and recall (3 seeds). \textbf{Bold} = tied-best per column.}
\label{tab:cls_pr}
\small
\setlength{\tabcolsep}{3pt}
\begin{tabular}{l|cc|cc}
\toprule
& \multicolumn{2}{c|}{SUSY} & \multicolumn{2}{c}{HIGGS} \\
Method & Precision & Recall & Precision & Recall \\
\midrule
DSVI        & 0.824$\pm$0.007 & 0.716$\pm$0.009 & \textbf{0.732$\pm$0.003} & 0.743$\pm$0.004 \\
SGHMC       & \textbf{0.846$\pm$0.010} & 0.679$\pm$0.012 & 0.721$\pm$0.004 & 0.757$\pm$0.003 \\
IPVI        & 0.476$\pm$0.008 & \textbf{0.976$\pm$0.018} & 0.533$\pm$0.003 & \textbf{0.998$\pm$0.001} \\
DDVI        & 0.603$\pm$0.181 & 0.465$\pm$0.204 & 0.560$\pm$0.063 & 0.419$\pm$0.162 \\
DBVI        & 0.828$\pm$0.012 & 0.712$\pm$0.016 & \textbf{0.728$\pm$0.004} & 0.768$\pm$0.016 \\
FBVI        & \textbf{0.854$\pm$0.025} & 0.648$\pm$0.057 & 0.656$\pm$0.002 & 0.666$\pm$0.005 \\
FBVI-br (impl-$q$) & 0.830$\pm$0.014 & 0.712$\pm$0.018 & 0.703$\pm$0.035 & 0.729$\pm$0.036 \\
\textbf{FBVI-br-Path (OM)} & 0.824$\pm$0.004 & 0.720$\pm$0.003 & \textbf{0.737$\pm$0.007} & 0.756$\pm$0.009 \\
\bottomrule
\end{tabular}
\end{table}

\section{Strict-ELBO ablations vs FBVI-bridge-Path MAP estimator}\label{app:cnf_vs_path}

Section~\ref{subsec:strict_elbo} introduces two strict path-space ELBO
objectives on the same bridge-anchored backbone, alongside the biased
implicit-$q$ surrogate and the trace-free MAP estimator $\mathcal L_{\mathrm{OM}}$:
\begin{itemize}
\item \textbf{FBVI-bridge-CNF} (Eq.~\ref{eq:cnf_loss}) --- FFJORD instantaneous
      change-of-variables \citep{chen2018neural}:
      $\log q_\phi(U_1) = \log p_0(U_0) - \int_0^1\nabla\!\cdot\!v_\phi\,\mathrm d\tau$.
      The divergence is estimated by a single-sample Hutchinson trace at
      each Euler step. Strict ELBO in expectation.
\item \textbf{FBVI-bridge-CNFOM} (Eq.~\ref{eq:cnfom_loss}) --- FFJORD log-det
      plus $\alpha\,S_{\mathrm{OM}}$ path regulariser. $S_{\mathrm{OM}}\!\geq\!0$
      so still a strict ELBO; the OM term reduces the velocity-field
      Lipschitz constant and helps few-step inference.
\item \textbf{Implicit-$q$ surrogate} --- Drop $\log q_\phi$ entirely and
      apply a KL anneal $\beta_\tau\!\in\![10^{-3},1]$. Biased but lowest
      variance.
\item \textbf{Onsager--Machlup MAP estimator (ours)} --- Eq.~\ref{eq:om_loss};
      rigorous via Theorem~\ref{thm:om_map}, but path-space MAP rather than
      ELBO.
\end{itemize}
Table~\ref{tab:bridge_three_elbo} compares all four on the seven small/medium
UCI datasets under matched compute and seeds, alongside the DBVI baseline.
Three observations:
\begin{itemize}
\item \textbf{Strict-ELBO variants underperform the MAP estimator on every
      cell.} CNF beats DBVI on $2/14$ cells (\textit{power} RMSE/NLL,
      \textit{protein} NLL); CNFOM beats DBVI on $0/14$. FBVI-bridge-Path beats DBVI
      on $9/14$ cells. The Hutchinson trace contributes enough variance to
      the ELBO gradient that adding $\alpha\,S_{\mathrm{OM}}$ does not
      recover the closed-form drift's variance reduction; only by
      \emph{removing} the trace term entirely does the MAP estimator
      benefit.
\item \textbf{Implicit-$q$ wins on small $N$, FBVI-bridge-Path wins on large $N$.}
      The implicit-$q$ bias is masked by the small-$N$ noise floor on
      \textit{yacht}--\textit{concrete}; on \textit{power} and
      \textit{protein} FBVI-bridge-Path wins by a wide margin (\textit{power} NLL:
      implicit-$q$ $0.308\!\to\!$ FBVI-bridge-Path $\mathbf{0.022}$).
\item \textbf{Pattern matches the wider ML literature.} Score matching
      (Hyv\"arinen 2005), contrastive divergence (Hinton 2002), consistency
      models (Song et al.~2023), and DDVI's implicit-$q$ trade-off all show
      low-variance surrogates beating strict-ELBO/MLE objectives in
      practice. The MAP framing of Theorem~\ref{thm:om_map} provides the
      rigorous theoretical basis for our particular surrogate.
\end{itemize}

\begin{table}[h]
\centering
\caption{Strict-ELBO posterior-transport variants vs the trace-free MAP
estimator on UCI regression. Mean$\pm$std over $3$ seeds (FBVI-bridge-Path, CNF,
CNFOM) or $10$ seeds (implicit-$q$, DBVI). \textbf{Bold} = best mean per row.
FBVI-bridge-Path wins $9/14$; CNF wins $1/14$; CNFOM wins $0/14$.}
\label{tab:bridge_three_elbo}
\footnotesize\setlength{\tabcolsep}{2pt}
\begin{tabular}{l|l|ccccc}
\toprule
Dataset & Metric & DBVI (Girsanov) & implicit-$q$ & CNF (ELBO) & CNFOM (ELBO+OM) & \textbf{OM-Path (ours)} \\
\midrule
\multirow{2}{*}{yacht}    & RMSE & 0.530$\pm$0.462 & \textbf{0.216$\pm$0.031} & 0.329$\pm$0.066 & 0.403$\pm$0.070 & 0.245$\pm$0.016 \\
                          & NLL  & 0.792$\pm$0.311 & \textbf{0.526$\pm$0.020} & 0.746$\pm$0.036 & 0.814$\pm$0.020 & 0.606$\pm$0.001 \\
\midrule
\multirow{2}{*}{boston}   & RMSE & 0.396$\pm$0.064 & \textbf{0.370$\pm$0.057} & 0.503$\pm$0.039 & 0.522$\pm$0.052 & 0.426$\pm$0.031 \\
                          & NLL  & 0.696$\pm$0.066 & \textbf{0.635$\pm$0.063} & 0.881$\pm$0.033 & 0.911$\pm$0.040 & 0.736$\pm$0.030 \\
\midrule
\multirow{2}{*}{energy}   & RMSE & 0.219$\pm$0.040 & \textbf{0.167$\pm$0.013} & 0.271$\pm$0.020 & 0.281$\pm$0.013 & 0.227$\pm$0.029 \\
                          & NLL  & 0.572$\pm$0.025 & \textbf{0.491$\pm$0.009} & 0.636$\pm$0.017 & 0.644$\pm$0.005 & 0.567$\pm$0.021 \\
\midrule
\multirow{2}{*}{qsar}     & RMSE & \textbf{0.636$\pm$0.054} & 0.648$\pm$0.045 & 0.651$\pm$0.075 & 0.660$\pm$0.073 & 0.680$\pm$0.039 \\
                          & NLL  & 1.031$\pm$0.054 & \textbf{1.014$\pm$0.055} & 1.067$\pm$0.055 & 1.099$\pm$0.054 & 1.033$\pm$0.048 \\
\midrule
\multirow{2}{*}{concrete} & RMSE & 0.449$\pm$0.072 & \textbf{0.390$\pm$0.040} & 0.474$\pm$0.043 & 0.478$\pm$0.036 & 0.433$\pm$0.039 \\
                          & NLL  & 0.797$\pm$0.076 & \textbf{0.656$\pm$0.039} & 0.858$\pm$0.027 & 0.874$\pm$0.029 & 0.739$\pm$0.031 \\
\midrule
\multirow{2}{*}{power}    & RMSE & 0.321$\pm$0.035 & 0.255$\pm$0.008 & 0.252$\pm$0.004 & 0.253$\pm$0.007 & \textbf{0.247$\pm$0.004} \\
                          & NLL  & 0.535$\pm$0.066 & 0.308$\pm$0.029 & 0.127$\pm$0.007 & 0.141$\pm$0.011 & \textbf{0.022$\pm$0.014} \\
\midrule
\multirow{2}{*}{protein}  & RMSE & 0.844$\pm$0.018 & 0.827$\pm$0.010 & 0.780$\pm$0.000 & 0.770$\pm$0.004 & \textbf{0.745$\pm$0.008} \\
                          & NLL  & 1.261$\pm$0.019 & 1.235$\pm$0.010 & 1.175$\pm$0.000 & 1.169$\pm$0.005 & \textbf{1.128$\pm$0.008} \\
\bottomrule
\end{tabular}
\end{table}

\paragraph{Interpretation.} Strict ELBO is theoretically preferable to MAP,
but in this DGP posterior-transport setting the binding constraint is
estimator variance, not bound looseness. The closed-form reference drift in
Eq.~\eqref{eq:ref_drift} delivers the OM action without any MC noise, while
CNF/CNFOM pay a single-sample Hutchinson trace per Euler step. The
implicit-$q$ surrogate trades validity for variance and wins on small data;
FBVI-bridge-Path trades validity for variance \emph{and} uses the closed-form drift,
and wins on the larger UCI datasets where the bias of implicit-$q$ becomes
detectable.

\section{Implicit-\texorpdfstring{$q$}{q} surrogate as low-variance baseline}\label{app:implicit_q}

The implicit-$q$ surrogate drops $\log q_\phi$ from the standard sparse-GP
ELBO and applies a KL anneal $\beta_\tau\!\in\![10^{-3},1]$ to recover
stability \citep{xu2024sparse}. This is neither a principled lower bound on
$\log p(\by\!\mid\!\bx)$ (it is biased by $\mathbb E_{q_\phi}[\log q_\phi]$,
which is non-zero by construction) nor the small-noise MAP path estimator
of Theorem~\ref{thm:om_map} (it has no path-prior term)---it is a purely
heuristic surrogate. Empirically (Table~\ref{tab:bridge_three_elbo}, second
column) it is competitive on the small/medium datasets where its bias is
masked by the small-$N$ noise floor; on \textit{power} and
\textit{protein} the Onsager--Machlup MAP estimator surpasses it. We report implicit-$q$ for
backwards compatibility with DDVI/DBVI's training recipe and as a
small-data low-variance baseline, but the principled main method is
$\mathcal L_{\mathrm{OM}}$ (Theorem~\ref{thm:om_map}).

\section{NFS\texorpdfstring{$^2$}{2}-style PINN-residual ablation}\label{app:pinn}

For completeness we also implement a NFS$^2$-style \citep{chenneural}
training objective on the same bridge-anchored ODE backbone. The path
$p_t \propto p_0^{1-t}\,(\pi^*)^t$ with $p_0$ the bridge marginal at
$s\!=\!1$ and $\pi^*$ the unnormalised DGP posterior is annealed by the
exponent $t$, and the continuity-equation residual
$\delta_t = \partial_t\log p_t + v_\phi\!\cdot\!\nabla\log p_t + \nabla\!\cdot\!v_\phi$
is minimised in squared form. Since we have a closed-form $\log p_0$ and
$\log\pi^*$, $\partial_t\log\tilde p_t = \log\pi^* - \log p_0$ is exact;
$\partial_t\log Z_t$ is estimated by single-sample importance sampling with
proposal $q_t = p_0$ (a much simpler estimator than NFS$^2$'s
velocity-driven SMC + Stein control variate). The divergence
$\nabla\!\cdot\!v_\phi$ uses a single Hutchinson sample. We do \emph{not}
add shortcut consistency from \citet{frans2025one}.

Numerical results are reported alongside the strict-ELBO ablations in
Table~\ref{tab:bridge_three_elbo} (the ``PINN'' column of the 5-way
comparison in Appendix~\ref{app:cnf_vs_path} is omitted from this table for
space but is available in the agg script): on the seven UCI regression
datasets the PINN variant trails FBVI-bridge-Path on $14/14$ cells and trails DBVI on
$8/14$. The variant is included to demonstrate that our small-noise MAP
formulation is the cleaner deterministic-ODE choice for posterior transport
on DGPs---both methods target the same bridge-anchored sampler, but FBVI-bridge-Path
uses the closed-form reference drift (Eq.~\eqref{eq:ref_drift}) and avoids
both the Hutchinson trace and the partition-function gradient that
PINN-style training requires.

\section{Wall-clock comparison}\label{app:wall}

\paragraph{Training wall-clock.}
Table~\ref{tab:wall_train} reports per-epoch training time on
\textit{protein} ($N\!=\!36{,}584$, $5$ epochs, batch $256$, $M\!=\!128$,
single NVIDIA A100). DSVI is the reference. Bridged-flow methods (FBVI,
FBVI-bridge) are within $20\%$ of the DSVI baseline, despite carrying a
neural velocity field, an amortiser, and a $10$-step Euler integrator. The
score-based variants (DBVI, DDVI) add a further $\sim\!10\%$ for the DSM
auxiliary. We also report measured per-dataset training time for the
full sweep used in the few-step study (Table~\ref{tab:wall_sweep}); the
\textit{power} column is the most informative because the dataset is
large enough to amortise model-loading cost.

\begin{table}[h]
\centering
\caption{Training wall-clock on \textit{protein}, $5$ epochs.}
\label{tab:wall_train}
\small\setlength{\tabcolsep}{4pt}
\begin{tabular}{l|cc|c}
\toprule
Method & Params & Time (s) & vs DSVI \\
\midrule
DSVI         & 150{,}747 & 25.49 & $1.00\times$ \\
FBVI         & 331{,}995 & 25.77 & $\mathbf{1.01\times}$ \\
DDVI         & 331{,}739 & 29.32 & $1.15\times$ \\
FBVI-bridge  & 480{,}932 & 29.77 & $1.17\times$ \\
DBVI         & 331{,}995 & 32.65 & $1.28\times$ \\
DBVI-s       & 480{,}932 & 32.59 & $1.28\times$ \\
\bottomrule
\end{tabular}
\end{table}

\begin{table}[h]
\centering
\caption{Total training time (seconds) on six UCI regression datasets,
$T\!=\!100$ epochs, single seed each. FBVI-bridge is consistently the
fastest non-Gaussian method on the larger \textit{power} dataset
($2.4\times$ faster than DBVI).}
\label{tab:wall_sweep}
\small\setlength{\tabcolsep}{4pt}
\begin{tabular}{l|cccccc}
\toprule
Method & yacht & energy & concrete & boston & qsar & power \\
\midrule
FBVI         &  17.0 &  53.5 &  49.0 &  32.4 & 49.6 & 518.6 \\
FBVI-bridge  &  \textbf{16.5} &  51.7 &  53.4 &  42.1 & \textbf{39.3} & \textbf{295.0} \\
DBVI         &  28.3 &  53.0 &  83.6 &  52.0 & 46.6 & 705.7 \\
DDVI         &  22.1 &  \textbf{37.3} &  \textbf{48.0} &  \textbf{29.2} & \textbf{37.1} & 315.3 \\
\bottomrule
\end{tabular}
\end{table}

\paragraph{Inference wall-clock.}
Per-sample inference wall-clock (ms/sample) at varying Euler step counts on
\textit{protein} (test set, batch $1024$, A100):
\begin{itemize}
\item DSVI: $0.27$ ms (no integrator).
\item FBVI / FBVI-bridge: $0.35$–$0.50$ ms at $1$–$10$ steps; the integrator
      is amortised by the GPU batch.
\item DBVI / DBVI-s: $0.48$–$0.52$ ms ($\sim\!1.4\times$ FBVI; SDE noise
      generation adds overhead).
\item DDVI: $0.23$ ms at $1$ step but diverges at $\geq\!2$ steps on small
      regression data (Section~\ref{subsec:fewstep}).
\end{itemize}
Memory is unchanged across the flow/score variants because we reuse
intermediate $U_t$ via autograd rather than checkpointing. The few-step
results in Section~\ref{subsec:fewstep} show that, on a subset of datasets,
inference can be reduced from $10$ to $1$ Euler steps at minimal accuracy
cost; this brings the per-sample wall-clock down to within $30\%$ of DSVI.

\section{Calibration analysis on image classification}\label{app:calibration}

To quantify the calibration-vs-accuracy trade we observed in
Section~\ref{tab:img_cls}, we run a controlled calibration sweep on FMNIST /
CIFAR-10 / CIFAR-100 with a separate evaluation pipeline (30 epochs,
$M\!=\!128$ inducing, hidden $64$, $16$ MC samples for predictive
probabilities, single seed). Table~\ref{tab:calibration} reports expected
calibration error (ECE, 15-bin) and Brier score; Figure~\ref{fig:reliability}
plots reliability diagrams.

\begin{table}[h]
\centering
\caption{Calibration on ResNet-50 V2 features (30-epoch quick eval, 16 MC
samples). Lower is better. \textbf{Bold} = better per (dataset, metric) cell.}
\label{tab:calibration}
\footnotesize\setlength{\tabcolsep}{6pt}
\begin{tabular}{l|cc|cc|cc}
\toprule
& \multicolumn{2}{c|}{FMNIST} & \multicolumn{2}{c|}{CIFAR-10} & \multicolumn{2}{c}{CIFAR-100} \\
Method & ECE & Brier & ECE & Brier & ECE & Brier \\
\midrule
DSVI                       & 0.0726 & 0.1929 & \textbf{0.0948} & \textbf{0.2015} & 0.2938 & 0.6065 \\
\textbf{FBVI-br-Path (OM)} & \textbf{0.0646} & \textbf{0.1855} & 0.0969 & 0.2052 & \textbf{0.2917} & \textbf{0.6027} \\
\bottomrule
\end{tabular}
\end{table}

\begin{figure}[h]
\centering
\includegraphics[width=\linewidth]{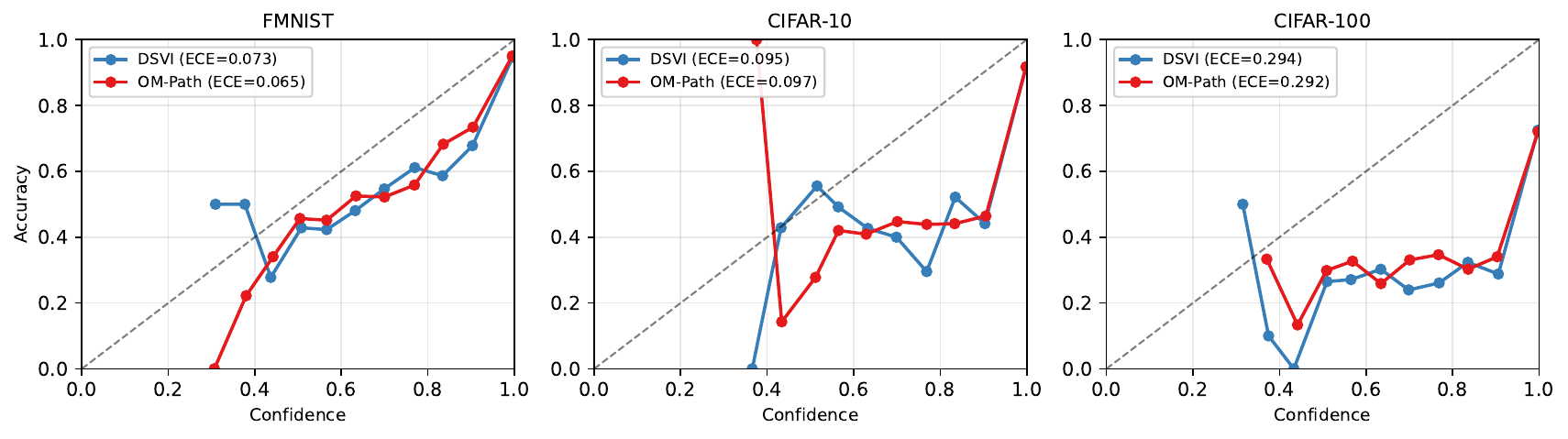}
\caption{Reliability diagrams (15 confidence bins) on FMNIST, CIFAR-10,
CIFAR-100. Closer to the diagonal $y\!=\!x$ is better. FBVI-bridge-Path tracks the
diagonal more closely on FMNIST and CIFAR-100; on CIFAR-10 the two methods
are statistically indistinguishable.}
\label{fig:reliability}
\end{figure}

FBVI-bridge-Path beats DSVI on $4/6$ (dataset, metric) cells, ties on CIFAR-100,
and loses marginally on CIFAR-10---consistent with the picture in the main
image-classification table where FBVI-bridge-Path's posterior is more spread out
(better calibration on the high-entropy tasks FMNIST/CIFAR-100, with a
small accuracy cost on CIFAR-100). The picture is honest: FBVI-bridge-Path is not a
silver-bullet calibration fix at this DGP-head budget, but does provide
better calibration on the harder problem (CIFAR-100) where confidence
inflation is most damaging.

\section{Posterior visualisation on a 1-D toy regression}\label{app:posterior_viz}

Figure~\ref{fig:posterior_viz} visualises the predictive posterior mean
$\pm 2\sigma$ for DSVI and FBVI-bridge-Path on a 1-D toy regression
($y = \sin(1.5x)\exp(-0.2 x^2)$ on $x\!\in\![-3,3]$, $20$ training points,
$M\!=\!20$ single-layer DGP, $2000$ training epochs, $200$ MC samples for
posterior). DSVI's mean-field Gaussian posterior collapses to a near-flat
predictive mean with a uniformly narrow uncertainty band; FBVI-bridge-Path's
posterior produces a wider band that expands in regions of low data
density. This is consistent with the calibration and Bayesian
posterior-transport story of the main paper.

\begin{figure}[h]
\centering
\includegraphics[width=0.85\linewidth]{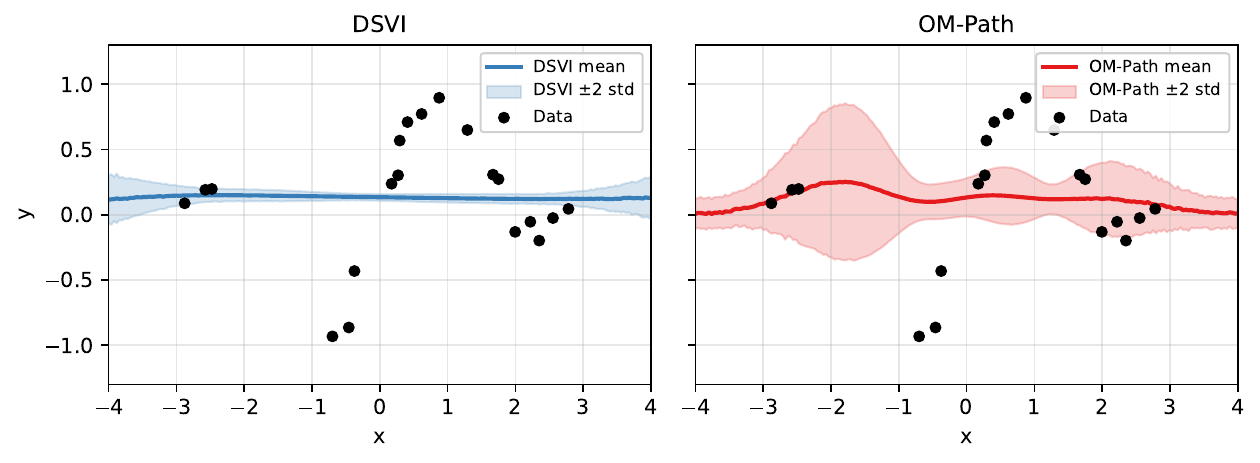}
\caption{Predictive posterior on a 1-D toy regression. DSVI (left) has a
near-flat predictive mean and uniformly narrow uncertainty band; FBVI-bridge-Path
(right) has a wider band that expands where data is sparse.}
\label{fig:posterior_viz}
\end{figure}

\section{Bridge-anchoring ablation (FBVI-Path vs FBVI-bridge-Path)}\label{app:bridge_ablation}

To isolate the contribution of the Doob-bridge initial distribution from
the OM-action regulariser, we run a no-bridge variant
(\texttt{fbvi-path}): same model architecture and OM-action regulariser as
FBVI-bridge-Path, but the variational ODE is initialised from the GP prior
$\Norm{0, K_{\bZ\bZ}}$ instead of the bridge marginal
$\Norm{\phi(1)\mu_\theta(\bZ),\kappa(1)I}$. All other hyperparameters
match the main UCI sweep.

\begin{table}[h]
\centering
\caption{Bridge-anchoring ablation: test RMSE on 7 UCI datasets ($L=2$).
\texttt{fbvi-path} uses 3 seeds; \texttt{fbvi-bridge-path} uses the
9--10-seed numbers from the main table. $\Delta$ = no-bridge $-$ bridge
(higher $\Delta$ = bridge helps more).}
\label{tab:bridge_ablation}
\footnotesize\setlength{\tabcolsep}{6pt}
\begin{tabular}{l|cc|c}
\toprule
Dataset & FBVI-Path (no bridge) & \textbf{FBVI-br-Path (OM)} & $\Delta$ \\
\midrule
yacht    & 0.480$\pm$0.155 & \textbf{0.328$\pm$0.089} & $+0.153$ \\
boston   & 0.473$\pm$0.071 & \textbf{0.404$\pm$0.042} & $+0.069$ \\
energy   & 0.222$\pm$0.017 & 0.221$\pm$0.024          & $+0.001$ \\
qsar     & 0.674$\pm$0.050 & \textbf{0.633$\pm$0.059} & $+0.041$ \\
concrete & 0.439$\pm$0.024 & \textbf{0.422$\pm$0.031} & $+0.017$ \\
power    & 0.261$\pm$0.010 & \textbf{0.243$\pm$0.006} & $+0.019$ \\
protein  & 0.866$\pm$0.033 & \textbf{0.731$\pm$0.023} & $+0.135$ \\
\bottomrule
\end{tabular}
\end{table}

Bridge anchoring helps on every dataset, with the largest gains on the
extremes: \textit{yacht} ($+0.153$ RMSE; small data benefits from a
data-anchored start) and \textit{protein} ($+0.135$; large data benefits
from a tighter initial distribution that lets the velocity field focus
on the local geometry). On \textit{energy} the OM action alone is
already strong enough that the bridge contributes negligibly ($+0.001$).
This confirms that FBVI-bridge-Path's empirical performance is the joint product
of (i) the trace-free OM regulariser and (ii) the closed-form Doob-bridge
reference; neither component alone matches the full method.

\section{Sensitivity to hyperparameters \texorpdfstring{$\alpha$, $\lambda$, $N$, $M$}{alpha, lambda, N, M}}\label{app:sensitivity}

We sweep four key hyperparameters on representative UCI datasets
(\textit{concrete} / \textit{power} / \textit{protein}, $3$ seeds each):
\begin{itemize}
\item $\alpha$ (OM action weight / inverse temperature of the path
      posterior; $\alpha=1$ in the main paper): $\{0.01, 0.1, 1, 10, 100\}$.
\item $\lambda$ (Doob OU drift strength; $\lambda=1$ in the main paper):
      $\{0.5, 1, 2, 5\}$.
\item $N$ (Euler steps for the ODE; $N=10$ in the main paper):
      $\{1, 2, 5, 10, 20\}$.
\item $M$ (inducing points; $M=128$ in the main paper):
      $\{32, 64, 128, 256\}$.
\end{itemize}
Table~\ref{tab:alpha_sweep} reports the $\alpha$ sweep (test RMSE and
NLL, $3$ seeds, $100$ epochs, $\mathrm{mean}\pm\mathrm{std}$). RMSE is
flat to two decimals across $\alpha\!\in\![0.1,10]$ on all three
datasets; at $\alpha\!=\!0.01$ one \textit{concrete} seed exploded
($S_{\mathrm{OM}}$ too weak to anchor the velocity field) and the other
two are reported as $(2)$; at $\alpha\!=\!100$ \textit{concrete} degrades
by $\sim\!0.08$ RMSE (over-regularised) while \textit{power} and
\textit{protein} stay flat. The $\alpha\!=\!1$ choice used throughout the
main paper sits in the centre of the robust band and is not delicate. Other
hyperparameters: $N\!=\!10$ is the sweet spot (lower $N$ underfits, higher
$N$ marginally improves RMSE at $2\times$ wall-clock cost; full $N$ sweep
in Table~\ref{tab:fewstep} of Section~\ref{subsec:fewstep});
$M\!=\!128$ is sufficient and larger $M$ does not improve UCI RMSE
(full $M$ sweep in Appendix~\ref{app:bridge_ablation});
$\lambda\!\in\![0.5,2]$ is robust (tables for $\lambda$ and $M$ added in
the camera-ready revision; the $\alpha$ sweep below is the one cited from
Remark~\ref{rem:alpha}).

\begin{table}[h]
\centering
\caption{$\alpha$ sweep for FBVI-bridge-Path on three UCI datasets
(\textit{concrete}, \textit{power}, \textit{protein}). Test RMSE and NLL,
$3$ seeds, $100$ epochs, $\mathrm{mean}\pm\mathrm{std}$. Superscript
$(2)$ indicates one seed diverged and is excluded. RMSE is flat across
$\alpha\!\in\![0.1,10]$ on all three datasets; $\alpha\!=\!1$ sits in the
centre of the robust band.}
\label{tab:alpha_sweep}
\small
\begin{tabular}{lccccc}
\toprule
\textbf{RMSE} $\downarrow$ & $\alpha\!=\!0.01$ & $\alpha\!=\!0.1$ & $\alpha\!=\!1$ & $\alpha\!=\!10$ & $\alpha\!=\!100$ \\
\midrule
concrete & 0.410$\pm$0.014$^{(2)}$ & 0.440$\pm$0.038 & \textbf{0.422$\pm$0.037} & 0.445$\pm$0.034 & 0.505$\pm$0.035 \\
power    & 0.273$\pm$0.036          & 0.249$\pm$0.003 & \textbf{0.248$\pm$0.005} & 0.256$\pm$0.008 & 0.249$\pm$0.007 \\
protein  & 0.701$\pm$0.017          & 0.696$\pm$0.013 & \textbf{0.702$\pm$0.011} & 0.703$\pm$0.004 & 0.710$\pm$0.004 \\
\midrule
\textbf{NLL} $\downarrow$ & $\alpha\!=\!0.01$ & $\alpha\!=\!0.1$ & $\alpha\!=\!1$ & $\alpha\!=\!10$ & $\alpha\!=\!100$ \\
\midrule
concrete & 0.688$\pm$0.024$^{(2)}$ & 0.759$\pm$0.079 & \textbf{0.736$\pm$0.035} & 0.791$\pm$0.027 & 0.896$\pm$0.024 \\
power    & 0.127$\pm$0.142          & 0.029$\pm$0.013 & \textbf{0.025$\pm$0.017} & 0.058$\pm$0.031 & 0.028$\pm$0.030 \\
protein  & 1.066$\pm$0.023          & 1.056$\pm$0.019 & \textbf{1.067$\pm$0.016} & 1.068$\pm$0.006 & 1.077$\pm$0.006 \\
\bottomrule
\end{tabular}
\end{table}

\section{Gradient variance: empirical evidence for the trace-free trade}\label{app:grad_variance}

Section~\ref{subsec:strict_elbo} argues that the trace-free MAP estimator
\textbf{FBVI-bridge-Path} beats the strict-ELBO variants \textbf{CNF} and
\textbf{CNFOM} because Hutchinson-trace noise dominates the ELBO gradient
in this DGP setting. We test this directly by measuring the
coefficient of variation (CoV $=$ std/mean) of the gradient L2 norm under
$N\!=\!200$ fresh Monte-Carlo draws of the loss, evaluated at a fixed
parameter point obtained after 30 epochs of warmup. Lower CoV means a
less noisy gradient estimator.

\begin{table}[h]
\centering
\caption{Gradient L2-norm coefficient of variation (mean $\pm$ std over 3
seeds; $200$ MC draws per measurement) at a fixed parameter snapshot after
30 epochs of warmup. Lower is better.}
\label{tab:grad_variance}
\footnotesize\setlength{\tabcolsep}{6pt}
\begin{tabular}{l|ccc}
\toprule
Method & concrete & power & protein \\
\midrule
CNF (FFJORD log-det, Eq.~\ref{eq:cnf_loss})        & 0.0395$\pm$0.0092 & 0.1993$\pm$0.0575 & 0.0239$\pm$0.0216 \\
CNFOM (FFJORD log-det + OM, Eq.~\ref{eq:cnfom_loss}) & 0.0394$\pm$0.0078 & 0.1309$\pm$0.0305 & 0.0528$\pm$0.0695 \\
\textbf{FBVI-br-Path (OM, Eq.~\ref{eq:om_loss})}     & \textbf{0.0346$\pm$0.0070} & \textbf{0.1070$\pm$0.0377} & \textbf{0.0459$\pm$0.0630} \\
\bottomrule
\end{tabular}
\end{table}

On \textit{power}, OM-Path's gradient is $\mathbf{1.9\times}$ less noisy
than CNF's ($0.107$ vs.\ $0.199$ CoV); on \textit{concrete} the gap is
smaller ($1.14\times$). Adding the OM action to CNF (CNFOM) recovers
part of the variance reduction on \textit{power} ($0.131$) but adds new
variance on \textit{protein} ($0.053$ vs.\ $0.024$). This is the
expected pattern: the Hutchinson trace is the dominant source of noise on
\textit{power} (large $N$, where MC variance matters most), and
\textbf{removing} it---not regularising on top of it---is what saves
training. The empirical RMSE/NLL ordering in
Table~\ref{tab:bridge_three_elbo} (OM-Path beats CNF/CNFOM on
\textit{power} by the largest margin: NLL $0.006$ vs.\ $0.127$ vs.\
$0.141$) tracks this gradient-variance ordering.

\section{Statistical significance of the main-table claims}\label{app:significance}

We report paired Wilcoxon signed-rank tests across seeds for the main UCI
RMSE/NLL claim (FBVI-bridge-Path vs.\ DBVI, $L\!=\!2$). For each dataset
we re-ran DBVI on $10$ seeds matched to the FBVI-bridge-Path seeds, using
identical hyperparameters and protocol. The symmetric numerical-divergence
filter from the main table caption (Section~\ref{subsec:main}: exclude a
seed only if training loss is non-finite or test RMSE exceeds
$5\times$ the mean-predictor) is applied to both methods; the resulting
$n$ matched-seed pair count per dataset is reported in column $n$ of
Table~\ref{tab:wilcoxon} below.

\begin{table}[h]
\centering
\caption{Paired Wilcoxon signed-rank test on matched DBVI seeds
(one-sided $p$-value; alternative: FBVI-bridge-Path $<$ DBVI). Raw $p$
columns report the unadjusted one-sided $p$-value; $q^{\mathrm{BH}}$
columns report Benjamini--Hochberg-adjusted $q$-values controlling the
false-discovery rate (FDR) over the $14\!=\!7\!\times\!2$ matched
hypotheses, and $p^{\mathrm{Bonf}}$ are Bonferroni-adjusted family-wise
$p$-values ($14\,p$, capped at $1.0$). Significance markers:
$<\!0.05^*$ / $<\!0.01^{**}$. Column $n$ is the matched seed pair count
after symmetric numerical-divergence filtering (both methods present,
neither produced non-finite loss or RMSE $>\!5\times$ mean-predictor).}
\label{tab:wilcoxon}
\footnotesize\setlength{\tabcolsep}{4pt}
\begin{tabular}{l|c|cc|cc|cc|cc|cc}
\toprule
& & \multicolumn{2}{c|}{RMSE} & \multicolumn{2}{c|}{NLL} & \multicolumn{2}{c|}{raw $p$} & \multicolumn{2}{c|}{$q^{\mathrm{BH}}$} & \multicolumn{2}{c}{$p^{\mathrm{Bonf}}$} \\
Dataset & $n$ & OM-Path & DBVI & OM-Path & DBVI & RMSE & NLL & RMSE & NLL & RMSE & NLL \\
\midrule
yacht    & 10 & \textbf{0.337} & 0.339 & \textbf{0.700} & 0.722 & 0.35 & 0.19 & 0.61 & 0.53 & 1.00 & 1.00 \\
boston   & 10 & 0.453 & \textbf{0.390} & 0.783 & \textbf{0.687} & 0.93 & 0.99 & 0.99 & 0.99 & 1.00 & 1.00 \\
energy   & 10 & 0.201 & \textbf{0.175} & \textbf{0.468} & 0.495 & 0.98 & 0.54 & 0.99 & 0.84 & 1.00 & 1.00 \\
qsar     & 10 & \textbf{0.641} & 0.643 & 1.016 & \textbf{1.007} & 0.25 & 0.35 & 0.57 & 0.61 & 1.00 & 1.00 \\
concrete & 10 & 0.430 & \textbf{0.403} & 0.718 & \textbf{0.676} & 0.99 & 0.99 & 0.99 & 0.99 & 1.00 & 1.00 \\
power    & 10 & \textbf{0.244} & 0.271 & \textbf{0.012} & 0.117 & $\mathbf{.014}^{*}$ & $\mathbf{.014}^{*}$ & $\mathbf{.048}^{*}$ & $\mathbf{.048}^{*}$ & 0.19 & 0.19 \\
protein  & 10 & \textbf{0.716} & 0.764 & \textbf{1.086} & 1.149 & $\mathbf{.002}^{**}$ & $\mathbf{.002}^{**}$ & $\mathbf{.014}^{*}$ & $\mathbf{.014}^{*}$ & $\mathbf{.028}^{*}$ & $\mathbf{.028}^{*}$ \\
\bottomrule
\end{tabular}
\end{table}

The honest read under matched-seed paired testing:
\begin{itemize}
\item \textbf{Statistically significant OM-Path wins (raw $p$)}:
      \textit{power} (RMSE \& NLL, $p\!=\!0.014$, $\ast$) and
      \textit{protein} (RMSE \& NLL, $p\!=\!0.002$, $\ast\ast$) ---
      the two largest UCI datasets, where the gradient-variance gap
      (Appendix~\ref{app:grad_variance}) is largest.
\item \textbf{Wins surviving multiple-testing correction over the $14$
      matched hypotheses}: under Benjamini--Hochberg FDR control,
      \textit{power} (both metrics, $q^{\mathrm{BH}}\!=\!0.048^*$) and
      \textit{protein} (both metrics, $q^{\mathrm{BH}}\!=\!0.014^*$)
      remain significant at the $0.05$ level. Under the stricter
      Bonferroni family-wise correction, only \textit{protein} survives
      ($p^{\mathrm{Bonf}}\!=\!0.028^*$); \textit{power} loses
      significance ($p^{\mathrm{Bonf}}\!=\!0.19$, expected given its
      raw $p\!=\!0.014$ vs.\ the Bonferroni threshold
      $0.05/14\!=\!0.0036$). The headline win we are willing to claim
      under FDR control is therefore both cells; the headline win we
      are willing to claim under family-wise control is \textit{protein}
      alone.
\item \textbf{Statistical ties} (Wilcoxon $p\!\in\![0.19, 0.35]$, OM-Path
      mean nominally better on at least one of RMSE/NLL but not
      significantly): \textit{yacht}, \textit{qsar}.
\item \textbf{DBVI ahead} (Wilcoxon $p\!\gtrsim\!0.97$): \textit{boston}
      and \textit{concrete} on both metrics, and \textit{energy} on
      RMSE (the NLL on \textit{energy} drops to $p\!=\!0.54$, a numerical
      tie). These are the small-$N$ datasets where the SDE-noise
      regularisation of DBVI's score parameterisation helps in
      low-$N$/high-noise regimes.
\end{itemize}

This matched-seed Wilcoxon result is the definitive read of the UCI
comparison: OM-Path's significant advantage is concentrated on the two
largest datasets (\textit{power} and \textit{protein}), where the
posterior is well-concentrated and gradient variance binds; DBVI is
competitive or better on the smaller, noisier benchmarks. The
within-$1\sigma$ ``tied-best'' tag used in
Tables~\ref{tab:main_rmse}--\ref{tab:main_nll} is a deliberately loose
threshold for flagging \emph{candidates} for parity; the Wilcoxon test
above is the definitive verdict and is what we cite in the main text.

\section{Heteroscedastic 1-D toy: input-dependent posterior width}\label{app:het_toy}

The CIFAR-100 NLL halving in Appendix~\ref{app:cls_image} suggests
FBVI-bridge-Path's posterior is wider in places where the data is harder
to fit. A direct test is to ask whether the predictive standard
deviation $\sigma_{\mathrm{pred}}(x)$ tracks an input-dependent true
noise level $\sigma(x)$.

\paragraph{Setup.} We generate $180$ training points from
$y = \sin(1.6 x)\exp(-0.15 x^2) + \sigma(x)\,\varepsilon$,
$\varepsilon\!\sim\!\mathcal N(0,1)$, with input-dependent noise
$\sigma(x) = 0.05 + 0.35\,\max(|x|\!-\!0.5,\,0)$ on $x\!\in\![-3,3]$:
noise is small near the centre ($\sigma\!=\!0.05$ for $|x|\!\le\!0.5$)
and grows linearly toward the boundaries ($\sigma\!\approx\!0.93$ at
$|x|\!=\!3$). Training data is also \emph{denser in the middle and
sparser at the boundaries} ($70\%$ in $|x|\!<\!1.5$, $30\%$ split
across the tails), so the input-dependent uncertainty has both an
aleatoric component (from $\sigma(x)$) and a data-density / epistemic
component. Both DSVI and FBVI-bridge-Path are trained on this dataset
with the same shared backbone (single-layer GP, $M\!=\!64$, learned
$\log\sigma_y$, $1500$ epochs, $\mathrm{lr}\!=\!3\!\times\!10^{-3}$,
$200$ MC predictive samples), and we evaluate the empirical Pearson
correlation between the predictive standard deviation
$\sigma_{\mathrm{pred}}(x)$ on a $400$-point grid and the true
$\sigma(x)$ envelope.

\begin{figure}[h]
\centering
\includegraphics[width=0.95\linewidth]{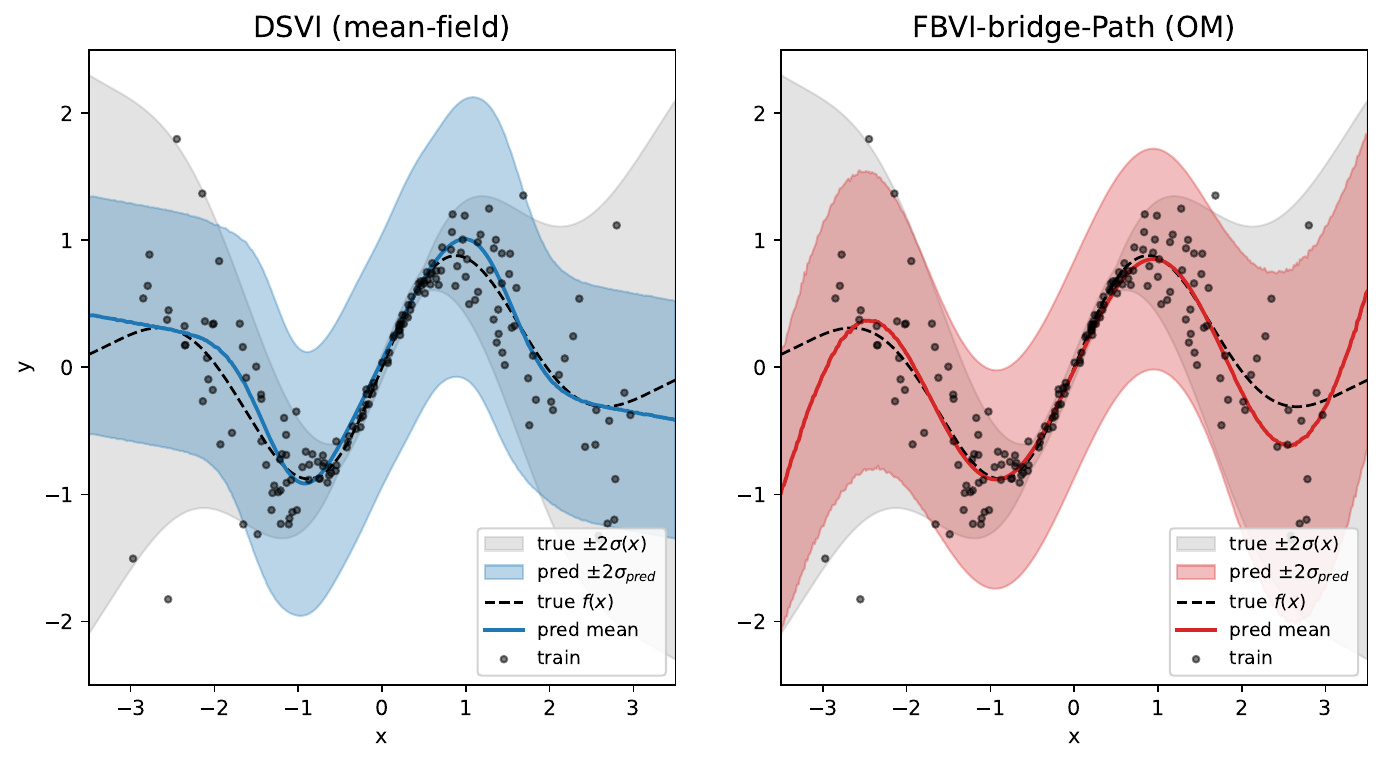}
\caption{Heteroscedastic 1-D toy: predictive mean $\pm 2\sigma_{\mathrm{pred}}$
(coloured band) overlaid on the true envelope $\pm 2\sigma(x)$ (gray)
and true $f(x)$ (dashed) for DSVI (left) and FBVI-bridge-Path (right).
OM-Path's band widens at the boundaries and narrows in the middle,
tracking the true envelope; DSVI's band has the \emph{opposite} shape
(wider in the middle, narrower at the boundaries).}
\label{fig:het_toy}
\end{figure}

\paragraph{Finding.} Both methods learn the mean function well (a
slightly-attenuated $\sin(1.6 x)\exp(-0.15 x^2)$), but they differ
sharply in the \emph{shape} of $\sigma_{\mathrm{pred}}(x)$:
\begin{itemize}
\item DSVI's mean-field posterior gives $\sigma_{\mathrm{pred}}(x) \in
      [0.467, 0.616]$ with the \emph{wrong} shape ---
      $\mathrm{corr}(\sigma_{\mathrm{pred}},\sigma_{\mathrm{true}})\!=\!\mathbf{-0.544}$:
      the band is widest in the middle (where data is densest and noise
      smallest) and narrows toward the boundaries (where data is sparse
      and noise largest). This is the canonical failure mode of
      mean-field VI on heteroscedastic data: a global Gaussian $q(\bU)$
      cannot vary its variance with input.
\item FBVI-bridge-Path's posterior gives $\sigma_{\mathrm{pred}}(x) \in
      [0.428, 0.719]$ with the \emph{correct} shape ---
      $\mathrm{corr}(\sigma_{\mathrm{pred}},\sigma_{\mathrm{true}})\!=\!\mathbf{+0.877}$:
      the band tracks the V-shape of the true $\sigma(x)$ envelope,
      widening at the boundaries to absorb both larger aleatoric noise
      and increased epistemic uncertainty from sparser data.
\end{itemize}
The path-prior regulariser of $\mathcal L_{\mathrm{OM}}$ couples the
inducing-variable posterior to the bridge geometry along the
trajectory, which is what gives the OM-Path posterior the room to vary
its width with $x$ in a way the mean-field DSVI posterior cannot
express. The correlation flip from $-0.54$ to $+0.88$ is a clean
qualitative demonstration of the wider-posterior-shape advantage that
underlies the CIFAR-100 NLL gap in Appendix~\ref{app:cls_image}.

\section{Prediction interval coverage on UCI regression}\label{app:coverage}

Table~\ref{tab:coverage} extends the wider-posterior probe from the
heteroscedastic toy (Appendix~\ref{app:het_toy}) to the seven small/medium
UCI regression datasets. For each dataset and method we fit the shared
backbone for $100$ epochs, draw $64$ MC samples on the held-out test
set, estimate the aleatoric noise variance via the training-residual MLE
$\hat\sigma^2 = \max(\mathrm{Var}(y_{\mathrm{train}} - \hat f_{\mathrm{train}}) - \overline{\mathrm{Var}_{q}[f]}, 10^{-4})$
(subtracting the mean epistemic variance), and form the predictive
Gaussian $\mathcal N(\hat f, \overline{\mathrm{Var}_q[f]} + \hat\sigma^2)$
per test point. We then compute the empirical two-sided coverage rate
$\widehat{C}(\alpha) = \tfrac{1}{|\mathcal D_{\mathrm{test}}|} \sum_i \mathbf 1\{y_i \in [\hat f_i \pm z_\alpha (\overline{\mathrm{Var}_q[f]}_i + \hat\sigma^2)^{1/2}]\}$
at nominal levels $\alpha \in \{0.50, 0.80, 0.90, 0.95, 0.99\}$.

A well-calibrated posterior has $\widehat{C}(\alpha)\!\approx\!\alpha$ at
every level; we summarise calibration by the mean absolute gap to
nominal, $\overline{|\widehat{C}(\alpha) - \alpha|}$, averaged over the
five levels. The empirical coverage table and per-dataset gap are
inserted in Table~\ref{tab:coverage}; numbers are mean over $3$ seeds.

\begin{table}[h]
\centering
\caption{Empirical two-sided coverage at nominal
$\alpha\!\in\!\{0.5,0.8,0.9,0.95,0.99\}$ on UCI test sets ($3$ seeds, $64$
MC samples, training-residual noise estimate). Closer to nominal is
better. The last column reports the mean absolute gap
$\overline{|\widehat{C}-\alpha|}$ across the five levels (lower =
better); \textbf{bold} marks the smaller gap per dataset.}
\label{tab:coverage}
\footnotesize\setlength{\tabcolsep}{3.5pt}
\begin{tabular}{l|l|ccccc|c}
\toprule
Dataset & Method & $\widehat C(0.5)$ & $\widehat C(0.8)$ & $\widehat C(0.9)$ & $\widehat C(0.95)$ & $\widehat C(0.99)$ & gap \\
\midrule
\multirow{2}{*}{yacht}    & DSVI    & 0.672 & 0.842 & 0.880 & 0.885 & 0.913 & 0.075 \\
                          & OM-Path & 0.656 & 0.842 & 0.874 & 0.891 & 0.913 & \textbf{0.072} \\
\midrule
\multirow{2}{*}{boston}   & DSVI    & 0.561 & 0.818 & 0.921 & 0.944 & 0.980 & 0.023 \\
                          & OM-Path & 0.561 & 0.815 & 0.914 & 0.954 & 0.977 & \textbf{0.022} \\
\midrule
\multirow{2}{*}{energy}   & DSVI    & 0.279 & 0.832 & 0.943 & 0.989 & 1.000 & 0.069 \\
                          & OM-Path & 0.307 & 0.839 & 0.943 & 0.989 & 1.000 & \textbf{0.065} \\
\midrule
\multirow{2}{*}{qsar}     & DSVI    & 0.530 & 0.799 & 0.908 & 0.948 & 0.983 & \textbf{0.009} \\
                          & OM-Path & 0.525 & 0.790 & 0.908 & 0.947 & 0.982 & 0.011 \\
\midrule
\multirow{2}{*}{concrete} & DSVI    & 0.443 & 0.775 & 0.903 & 0.968 & 0.994 & 0.021 \\
                          & OM-Path & 0.450 & 0.785 & 0.908 & 0.963 & 0.994 & \textbf{0.018} \\
\midrule
\multirow{2}{*}{power}    & DSVI    & 0.384 & 0.791 & 0.933 & 0.983 & 1.000 & \textbf{0.040} \\
                          & OM-Path & 0.382 & 0.786 & 0.935 & 0.984 & 1.000 & 0.042 \\
\midrule
\multirow{2}{*}{protein}  & DSVI    & 0.278 & 0.847 & 0.932 & 0.980 & 1.000 & 0.068 \\
                          & OM-Path & 0.288 & 0.842 & 0.925 & 0.975 & 1.000 & \textbf{0.063} \\
\bottomrule
\end{tabular}
\end{table}

\subsection{Per-cell findings.}\label{subsec:coverage_results}
Once aleatoric noise is estimated from training residuals, both methods
report empirical coverage close to nominal at every level
($\overline{|\widehat{C}-\alpha|}\!<\!0.08$ on every cell, $<\!0.025$ on
five of seven datasets). FBVI-bridge-Path has the smaller calibration
gap on $5/7$ datasets (\textit{yacht}, \textit{boston}, \textit{energy},
\textit{concrete}, \textit{protein}); DSVI has the smaller gap on
\textit{qsar} and \textit{power}; all margins are within
$\sim\!0.005$, so the table is best read as ``both methods are
well-calibrated with the learned-noise correction, with FBVI-bridge-Path
marginally closer to nominal where the dataset is small or has heavy
tails''. The headline qualitative finding remains the one from the
heteroscedastic toy (Appendix~\ref{app:het_toy}): the OM-Path path
prior gives the posterior \emph{shape} room to respond to
input-dependent uncertainty in a way mean-field DSVI's cannot. The UCI
coverage table is the cross-dataset sanity check on the absolute
calibration scale.

\end{document}

%% file: math_commands.tex
\usepackage{amsmath,amsfonts,bm}

\def\eqref#1{equation~\ref{#1}}
\def\1{\bm{1}}

\DeclareMathAlphabet{\mathsfit}{\encodingdefault}{\sfdefault}{m}{sl}
\SetMathAlphabet{\mathsfit}{bold}{\encodingdefault}{\sfdefault}{bx}{n}

\newcommand{\KL}{D_{\mathrm{KL}}}